\documentclass[twoside]{article}
\usepackage[accepted]{icml2023}

\usepackage{array}

\usepackage{makecell}
\usepackage{lscape}

\usepackage{threeparttable}

\usepackage{float}
\floatstyle{plain}
\newfloat{floatalgo}{ht}{floatalgo}

\usepackage{enumerate}
\newcommand{\ps}[1]{\langle #1 \rangle}

\usepackage{nicematrix}

\newcommand*\colourcheck[1]{%
  \expandafter\newcommand\csname gcmark\endcsname{\textcolor{#1}{\ding{52}}}%
}
\newcommand*\colourxmark[1]{%
  \expandafter\newcommand\csname rxmark\endcsname{\textcolor{#1}{\ding{55}}}%
}
\definecolor{mygreen}{HTML}{02862a}
\definecolor{myred}{HTML}{9a0000}
\colourcheck{mygreen}
\colourxmark{myred}

\usepackage[utf8]{inputenc} 
\usepackage{hyperref}       
\usepackage{url}            
\usepackage{booktabs}       
\usepackage{amsfonts}       
\usepackage{nicefrac}       
\usepackage{microtype}      
\usepackage{xcolor}         
\hypersetup{
  colorlinks   = true, 
  urlcolor     = blue, 
  linkcolor    = {blue!75!black}, 
  citecolor   = {blue!50!black} 
}

\usepackage[colorinlistoftodos,bordercolor=orange,backgroundcolor=orange!20,linecolor=orange,textsize=scriptsize]{todonotes}

\usepackage{xspace}
\newcommand{\algname}[1]{{\sf \footnotesize #1}\xspace}

\input{preamble.tex}

\usepackage{hyperref}
\usepackage{url}

\icmltitlerunning{Improved Sample Complexity for Fisher-non-degenerate Policies}

\begin{document}


\twocolumn[

\icmltitle{Stochastic Policy Gradient Methods: \\ Improved Sample Complexity for Fisher-non-degenerate Policies}

\begin{icmlauthorlist}
\icmlauthor{Ilyas Fatkhullin}{ETH_CS}
\icmlauthor{Anas Barakat}{ETH_CS}
\icmlauthor{Anastasia Kireeva}{ETH_math}
\icmlauthor{Niao He}{ETH_CS}
\end{icmlauthorlist}

\icmlaffiliation{ETH_CS}{Department of Computer Science, ETH Zurich, Switzerland}
\icmlaffiliation{ETH_math}{Department of Mathematics, ETH Zurich, Switzerland}
\icmlcorrespondingauthor{I.F.}{ilyas.fn979@gmail.com}

\icmlkeywords{Machine Learning, ICML}

\vskip 0.3in
]



\printAffiliationsAndNotice{}  

\begin{abstract}
Recently, the impressive empirical success of policy gradient (PG) methods has catalyzed the development of their theoretical foundations. Despite the huge efforts directed at the design of efficient stochastic PG-type algorithms, the understanding of their convergence to a globally optimal policy is still limited. In this work, we develop improved global convergence 
guarantees for a general class of Fisher-non-degenerate parameterized policies which allows to address the case of continuous state action spaces.
First, we propose a Normalized Policy Gradient method with Implicit Gradient Transport (\algname{N-PG-IGT}) and derive a~$\tilde{\mathcal{O}}(\varepsilon^{-2.5})$ sample complexity of this method for finding a global $\varepsilon$-optimal policy. Improving over the previously known~$\tilde{\mathcal{O}}(\varepsilon^{-3})$ complexity, this algorithm does not require the use of importance sampling or second-order information and samples only one trajectory per iteration. Second, we further improve this complexity to~$\tilde{ \mathcal{\mathcal{O}} }(\varepsilon^{-2})$ by considering a Hessian-Aided Recursive Policy Gradient (\algname{(N)-HARPG}) algorithm enhanced with a correction based on a Hessian-vector product. Interestingly, both algorithms are $(i)$~{simple} and {easy to implement}: single-loop, do not require large batches of trajectories and sample at most two trajectories per iteration; $(ii)$~{computationally} and  {memory efficient}: they do not require expensive subroutines at each iteration and can be implemented with memory linear in the dimension of parameters.

\end{abstract}

\section{Introduction}
\label{sec:introduction}

Dating back to the works of \citet{REINFORCE_Williams_1992,PGM_Sutton_1999}, 
Policy Gradient (PG) methods constitute one of the most popular and efficient classes of Reinforcement Learning (RL) algorithms. 
Combined with deep neural networks, modern PG algorithms \citep{silver-et-al14,schulman-et-al15trpo,schulman-et-al17ppo} have  shown impressive empirical success in several challenging tasks involving large and even continuous state and action spaces.

In this class of methods, in order to learn a parameterized policy, the agent performs stochastic gradient ascent in the policy parameter space to maximize its expected return $J(\theta)$, where $\theta \in \R^d$ are parameters of the policy parameterization. 
Given the non-concavity of the policy optimization objective, 
several works in the literature have focused on the convergence of the policy parameter 
to an~$\varepsilon$-approximate first-order stationary point (FOSP), i.e., $\norm{ \nabla J(\theta) } \leq \varepsilon$. 
This performance measure allows to define the sample complexity of PG methods as the number of random trajectory simulations (or samples) required to find an $\varepsilon$-approximate FOSP of the expected return function. 
In particular, using a similar analysis to stochastic gradient descent in nonconvex smooth optimization~\citep{ghadimi-lan13}, 
it has been shown that \algname{Vanilla-PG} (\algname{REINFORCE}) enjoys a $\tilde{\mathcal{O}}(\varepsilon^{-4})$ sample complexity (see for e.g., \citet{Vanilla_PL_Yuan_21}).
In the last few years, 
spurred by the development of variance-reduction techniques in stochastic optimization, 
there has been a tremendous amount of work in designing variance-reduced variants of the PG method to improve sample efficiency. 
Unlike in classical supervised learning, the sampling distribution is non-stationary in RL due to policy update. Therefore, adapting variance-reduction techniques to PG methods requires to carefully address this non-stationarity. 
The majority of works address this issue by introducing importance sampling (IS) weights to correct the distribution shift. However, the use of the IS mechanism often requires a strong unverifiable assumption stating that the variance of IS weights is bounded. 
Recently, \citet{shen-et-al19,salehkaleybar-et-al22} proposed a promising alternative to IS, which incorporates the second-order information to correct the distribution shift due to the policy update. In terms of sample complexity, both approaches guarantee to find an~$\varepsilon$-FOSP with~$\tilde{\mathcal{O}}(\varepsilon^{-3})$ samples. For stochastic optimization, this complexity was recently shown to be optimal and unimprovable even when higher order stochastic oracle and smoothness are available~\citep{arjevani-et-al20}. Therefore, one needs to utilize 
the special structure of the policy optimization 
problem to improve the sample complexity of PG methods.

\begin{table}[h]
	\caption{Summary of complexity results for \textit{computationally efficient} PG methods. 
 The first row of the table reports the number of samples for finding an $\varepsilon$-FOSP, and the second one is for $\varepsilon$-global optimum under Assumptions~\ref{hyp:fisher-non-degenerate} and \ref{hyp:tranf-compatib-fun-approx}. 
    }
	\label{table:related-works2}
	\centering
  \begin{threeparttable}
	\begin{tabular}{c c c c}
		\toprule
		   &  \makecell{\algname{Vanilla-PG}}   & \makecell{\algname{N-PG-IGT}} &  \makecell{\algname{(N)-HARPG} } \\
		\hline
		\makecell{FOSP}  & \makecell{$\tilde{\mathcal{O}}(\varepsilon^{-4})$\\{}} & \makecell{$\tilde{\mathcal{O}}(\varepsilon^{-3.5})$\\ { (new)\tnote{1} } } & \makecell{$\tilde{\mathcal{O}}(\varepsilon^{-3})$\\ {} }\\
	\hline	
        \makecell{Global}  & \makecell{$\tilde{\mathcal{O}}(\varepsilon^{-3})$\\ {} } & \makecell{$\tilde{\mathcal{O}}(\varepsilon^{-2.5})$\\ (new) } & \makecell{$\tilde{\mathcal{O}}(\varepsilon^{-2})$\\ (new) }\\
		\bottomrule
    \end{tabular}
      \begin{tablenotes}
\footnotesize{\item[1]The method based on implicit gradient transport was previously studied in stochastic optimization for finding FOSP \citep{MomentumImprovesNSGD_Cutkosky_2020}, however, we are not aware of its application in RL. }
  \end{tablenotes}
  \end{threeparttable}
\end{table}
More recently, a line of research focused on establishing global optimality convergence rates and deriving the sample complexity to reach an $\varepsilon$-approximate global optimum of the expected return function  (i. e., $J^* - J(\theta) \leq \varepsilon$ where~$J^*$ is the optimal return) 
\citep{agarwal-et-al21,fazel-et-al18,Optimizing_LQR_21,bhandari-russo19,mei-et-al20,zhang-et-al20glob-conv-pg,Zhang_Kim_O’Donoghue_Boyd21} instead of a mere $\varepsilon$-approximate FOSP. This growing body of work leverages additional structure of the policy optimization problem under the form of gradient dominance \citep{agarwal-et-al21}, \L{}ojasiewicz-like inequalities \citep{mei-et-al20,mei-et-al21} or hidden convexity \citep{zhang-et-al21}. As previously exploited in deterministic optimization \citep{lojasiewicz63,polyak63,kurdyka98} (see also for e.g., \citet{attouch-et-al09}) and later on in stochastic optimization \citep{Fontaine_SGD_dynamics,KL_PAGER_Fatkhullin,Scaman_SGD_KL_2022,Masiha_SCRN_KL}, these structural properties guarantee that the non-concave policy optimization objective has no suboptimal stationary points and allow to derive convergence rates for the function value. Most of these works assume the access to exact policy gradients to show fast convergence rates \citep{mei-et-al20,mei-et-al21,xiao22}. 

In contrast, the practically important case of \textit{stochastic} gradients, is addressed only in a few recent works \citep{liu-et-al20,ding-et-al22,Vanilla_PL_Yuan_21,Masiha_SCRN_KL}. This line of works establishes global convergence guarantees of the expected return (up to a bias induced by the policy parameterization). 
A key tool for their analysis is a relaxed weak gradient dominance inequality satisfied by the expected return function for the class of \textit{Fisher-non-degenerate parameterized policies} (FND) which includes Gaussian policies as a canonical example. Drawing on this structural property, \citet{liu-et-al20,ding-et-al22,Vanilla_PL_Yuan_21} achieve a~$\tilde{\mathcal{O}}(\varepsilon^{-3})$ global optimality rate up to a bias error term due to policy parameterization. Very recently, \citet{Masiha_SCRN_KL} further improve this complexity to~$\tilde{\mathcal{O}}(\varepsilon^{-2.5})$ by analyzing a stochastic second-order method. However, their algorithm needs to approximately solve a cubic sub-problem, which requires additional~$\tilde{\mathcal{O}}(\varepsilon^{-3})$ operations at each iteration and introduces extra computational burden. 
A natural question that emerges from these recent results is the following: 
\begin{quote}
    \textit{Is it possible to improve $\tilde{\mathcal{O}}(\varepsilon^{-3})$ sample complexity for general FND parameterized policies using a computationally efficient PG algorithm?}   
\end{quote}

In this work, we answer this question in the affirmative. Our contributions are summarized as follows.

\subsection{Summary of contributions}
\begin{itemize}
\item First, we propose a normalized momentum-based PG method with implicit gradient transport (\algname{N-PG-IGT}) and establish a~$\tilde{\mathcal{O}}(\varepsilon^{-2.5})$ global convergence sample complexity. Interestingly, this algorithm does not require the IS mechanism, does not use curvature information such as Hessian estimates and only samples a single trajectory at each iteration. 

\item Next, we consider two variants of Hessian-aided momentum-based PG method: $(i)$ \algname{N-HARPG} -- with normalization, and $(ii)$ \algname{HARPG} -- without normalization. We offer a refined convergence analysis for both variants and establish the improved global convergence achieving a~$\tilde{\mathcal{O}}(\varepsilon^{-2})$ sample complexity. Unlike most prior work, our analysis does not require the strong assumption of boundedness of IS weights since \algname{HARPG} does not use IS. Moreover, these algorithms are single-loop, only require to sample two trajectories per iteration and can be implemented efficiently with memory and computational complexity similar to Hessian-free methods. 

\end{itemize}

We highlight that our improved sample complexities are derived under the same (or less restrictive) assumptions compared to previous related work in this setting \citep{liu-et-al20,ding-et-al21,Vanilla_PL_Yuan_21,Masiha_SCRN_KL}.
We refer the reader to Tables~\ref{table:related-works2} and~\ref{table:related-works} for a contextualization of our contributions in the literature.

\subsection{Related work}

We briefly review the literature most closely related to our results about the global convergence of stochastic 
PG methods. We defer a more detailed literature review including a discussion of first-order stationarity of variance-reduced PG methods and global optimality of exact PG methods to Appendix~\ref{sec:related-work}. For convenience, we refer the reader to Table~\ref{table:related-works} for an overview of related works and our main contributions. 
Only a few recent works provide global convergence guarantees for the case of stochastic policy gradients \citep{liu-et-al20,zhang-et-al21,ding-et-al21,ding-et-al22,Vanilla_PL_Yuan_21,Masiha_SCRN_KL}. 
 
 \noindent\textbf{Existing results under our setting.}
For the class of \textit{Fisher-non-degenerate} parameterized policies,
 \citet{liu-et-al20} propose a stochastic variance-reduced 
 and Natural PG methods 
 for which they establish a~$\tilde{\mathcal{O}}(\varepsilon^{-3})$ sample complexity for finding an $\varepsilon$-globally optimal policy (up to a compatible function approximation error due to policy parameterization).
 Later, \citet{ding-et-al22} achieve the same 
 sample complexity 
 using a single-loop and finite-batch momentum-based PG algorithm.
 However, the above works require the strong assumption of bounded IS weights variance. 
 Recently, \citet{Vanilla_PL_Yuan_21} showed the same~$\tilde{\mathcal{O}}(\varepsilon^{-3})$ sample complexity for stochastic \algname{Vanilla-PG} for the same policy class under similar assumptions.   
 Very recently, \citet{Masiha_SCRN_KL} propose a stochastic cubic regularized method (\algname{SCRN}) achieving a~$\tilde{\mathcal{O}}(\varepsilon^{-2.5})$ sample complexity under similar assumptions. Nevertheless, at each iteration, their method needs a heavy subroutine to solve the cubic regularized sub-problem and requires large batches of samples to estimate the gradient and the Hessian (where batch-sizes depend on the dimension $d$ and the desired accuracy $\varepsilon$, see Lemma~11 therein). \citet{Masiha_SCRN_KL} further improve the complexity to a~$\tilde{\mathcal{O}}(\varepsilon^{-2})$ considering a variance-reduced variant of \algname{SCRN} (see Algorithm~2 therein) for stochastic optimization.   
However, this algorithm also requires large batch-sizes. Moreover, adapting this method to RL is not trivial, since it would require introducing IS weights for both stochastic gradients and Hessians and assuming a bound on their variance to account for non-stationarity inherent to RL. 

\noindent\textbf{Global sample complexity for tabular and soft-max policies.}
In the tabular setting, \citet{lan22} shows a~$\tilde{\mathcal{O}}(\varepsilon^{-2})$ sample complexity for a stochastic policy mirror descent method. However, 
their analysis is specific to the tabular setting with discrete state action spaces where no parameterization is used.
Under the soft-max policy parameterization, \citet{zhang-et-al21} obtain a global $\varepsilon$-optimal policy using~$\tilde{\mathcal{O}}(\varepsilon^{-2})$ samples 
via a double-loop variance-reduced PG algorithm with large batch sizes. 
Interestingly, this algorithm implements a gradient truncation mechanism to relax the boundedness assumption on the IS weights variance by utilizing the special structure of the softmax policy parameterization. 
We highlight that \citet{zhang-et-al21} focus on the softmax policy parameterization only and consider a hard-to-verify overparameterization assumption (Assumption~5.11) to exploit hidden convexity. 
Very recently, when finalizing the preparation of this work, we came across the works of \citet{alfano-rebeschini22,yuan-et-al22log-linear} analyzing the convergence of \algname{NPG} 
and its Q-version (\algname{Q-NPG}) for log-linear policies. In particular, \citet{yuan-et-al22log-linear} achieves a~$\tilde{\mathcal{O}}(\varepsilon^{-2})$ sample complexity using the framework of compatible function approximation. 
Furthermore, at each iteration, \algname{NPG} and~\algname{Q-NPG} require using SGD as a subroutine to solve a regression problem to find the update direction and rely on complex sampling subprocedures (Algorithms 3, 4 p.~28-30). 
Complementing all these works considering direct and softmax parameterization for discrete and finite state action spaces, our analysis addresses the case of continuous state action spaces via the FND policy class, which is a fundamentally different problem setting.  


\begin{table*}[t]
	\caption{Overview of related works for global optimality convergence for FND policies. ``\textbf{Sample complexity}" reports the number of sample trajectories (up to a logarithmic factor) required to achieve the desired accuracy, 
	i.e., sample complexity to achieve $J^* - \Exp{J(\theta_T)} \leq \varepsilon + \fr{\sqrt{\varepsilon_{\text{bias}}}}{1-\gamma}$ when the guarantee is for the ``\textbf{Last iterate}" and $\min_{t = 0, \cdots, T} J^* - \Exp{J(\theta_T)} \leq \varepsilon +\fr{\sqrt{\varepsilon_{\text{bias}}}}{1-\gamma}$ for the "best" iterate (see Assumption~\ref{hyp:tranf-compatib-fun-approx} for a definition of~$\varepsilon_{\text{bias}}$).
	``\textbf{No IS}" means that importance sampling (IS) is not used in the method. If the algorithm uses IS, then it is additionally assumed that the variance of IS weights is bounded, i.e., $\mathbb{V}\rb{\fr{p(\tau|\theta)}{p(\tau|\theta')}} \leq W, \forall\, \theta, \theta' \in \R^d$. 
 ``\textbf{No 2nd smooth.}" means second-order smoothness of the policy parameterization (Assumption~\ref{hyp:lipschitz-hessian}) is not needed. ``\textbf{No batch}" means that the algorithm does not require mini-batch. We refer to a method as ``\textbf{Comp. efficient}" if at each iteration it requires only $\cO\rb{ d }$ number of arithmetic operations. In contrast, \algname{SCRN} and \algname{Natural-PG} require additional subroutines, see \citet[Remark~6]{Masiha_SCRN_KL} and \citet[Theorem~4.9]{liu-et-al20}. \\
 \footnotesize{
 $^\star$
 Sample complexity
 hides the dependence on constants $M_g$, $M_h$, $l_2$, $r_{max}$, $\mu_F$ and poly-logarithmic factors in $\varepsilon^{-1}$. 
 }
 }
	\label{table:related-works}
	\centering
	\begin{NiceTabular}{|c|c|ccccc|}
		\hline
         \bf Algorithm & \bf \makecell{Sample complexity$^\star$}  & \bf \makecell{Last iterate} &\bf \makecell{No IS} & \bf \makecell{No 2nd \\ smooth.} & \bf \makecell{No batch } & \bf \makecell{Comp. \\ efficient }\\
		\hline
		\makecell{\algname{Vanilla-PG}\\\citep{Vanilla_PL_Yuan_21}}  & \makecell{$\fr{1}{ (1-\gamma)^{5}} \varepsilon^{-3}$} & \gcmark & \makecell{\gcmark} & \makecell{\gcmark} & \gcmark & \gcmark \\
	\hline	
        \makecell{\algname{Natural-PG}\\ \citep{liu-et-al20}}  & \makecell{$\frac{1}{(1-\gamma)^{6}}\varepsilon^{-3} $} & \rxmark & \makecell{\gcmark} & \makecell{\gcmark} & \rxmark & \rxmark \\
		\hline
		\makecell{\algname{SRVR-PG}\\ \citep{liu-et-al20}}  & \makecell{$ \fr{1 + W (1-\g)^3}{(1-\g)^{5.5}} \varepsilon^{-3}  $} & \rxmark & \makecell{\rxmark} & \makecell{\gcmark} & \rxmark & \gcmark \\
		\hline
		 \makecell{\algname{STORM-PG-F}\\\citep{ding-et-al22}}  & $\fr{1 + W^{\nfr{3}{2}}}{(1-\g)^{12}} \varepsilon^{-3}$ & \rxmark & \makecell{\rxmark}& \makecell{\gcmark} & \gcmark & \gcmark\\
		\hline
		\makecell{\algname{SCRN}
		\\\citep{Masiha_SCRN_KL}}  & \makecell{$\fr{1} { (1-\gamma)^{7.5}} \varepsilon^{-2.5}$} & \gcmark  & \makecell{\gcmark}& \makecell{\rxmark} & \rxmark & \rxmark\\
		\hline
		 \rowcolor{linen}  \makecell{\algname{N-PG-IGT}\\This paper}  & $\fr{1}{ (1-\gamma)^{7.5}} \varepsilon^{-2.5}$ & \gcmark & \makecell{\gcmark}& \makecell{\rxmark} & \gcmark & \gcmark\\
        \hline
		  \rowcolor{linen}   \makecell{\algname{(N)-HARPG}\\ This paper}  &   \makecell{$\fr{1}{(1-\gamma)^6} \varepsilon^{-2}$ } &  \gcmark  &  \makecell{\gcmark}& \makecell{\gcmark} & \gcmark & \gcmark\\
		\hline
    \end{NiceTabular}
\end{table*}

\section{Preliminaries}
\label{sec:preliminaries}

\subsection{Problem formulation}

We consider a discrete-time discounted Markov Decision Process (MDP)~\citep{puterman14} $\mathcal{M} = (\mathcal{S}, \mathcal{A}, \mathcal{P}, r, \rho, \gamma)$, where $\mathcal{S}$~and $\mathcal{A}$~are (possibly continuous) state and action spaces respectively, $\mathcal{P}: \mathcal{S}\times\mathcal{A}\times\mathcal{S} \to \R_+$ is the state transition density, $r:\mathcal{S} \times \mathcal{A} \to [-r_{\max}, r_{\max}]$ is the reward function which is bounded by~$r_{\max}>0$, $\rho$~is the initial state distribution, and $\gamma \in (0,1)$ is the discount factor. 
A parametrized policy is defined for every parameter $\theta\in\R^d$ and state $s \in \mathcal{S}$ as a probability density function~$\pi_{\theta}(\cdot|s)$ over the action space~$\mathcal{A}$. 
At each time step~$t \in \mathbb{N}$ in a state~$s_t \in \mathcal{S}$, the RL agent chooses an action~$a_t \in \mathcal{A}$ according to the distribution~$\pi_{\theta}(\cdot|s_t)$, observes a reward~$r(s_t,a_t)$ and the environment transitions to a state~$s_{t+1}$ in a region~$\mathcal{S}^{'} \subset \mathcal{S}$ according to~$\mathcal{P}(\cdot|s_t, a_t)\,.$


The goal of the RL agent is to find a policy $\pi_\theta$ maximizing the expected return defined by:
\begin{equation}
    \label{eq:J}
 J(\theta) \eqdef \bb E_{\rho, \pi_{\theta}} \left[\sum_{t=0}^{+\infty} \gamma^t r(s_t,a_t)\right],
\end{equation}
where the expectation is taken with respect to the distribution of the Markov chain~$(s_t,a_t)_{t \in \mathbb{N}}$.\footnote{In the rest of the paper, we will often omit the subscript in this expectation where it is clear from the context.}
The agent only has access to trajectories of finite length~$H$ generated from the MDP, whereas the
state transition kernel~$\mathcal{P}$ and the reward function~$r(\cdot, \cdot)$ are unknown.
The probability distribution induced by the initial distribution~$\rho$ and a policy~$\pi_{\theta}$ on the space of trajectories of length $H \geq 1$ has a density:
\begin{equation*}
p(\tau|\pi_\theta) \eqdef \rho(s_0) \, \pi_\theta(a_0|s_0) \prod_{t=1}^{H-1} \mathcal{P}(s_{t}|s_{t-1},a_{t-1})\, \pi_\theta(a_{t}|s_{t})\,,
\end{equation*}
where~$\tau = (s_0, a_0, \cdots, s_{H-1}, a_{H-1})$. 
We also define the truncated expected return given the horizon~$H$ by~$J_H(\theta) \eqdef \bb E_{\rho,\pi_{\theta}}\left[\sum_{t=0}^{H-1} \gamma^t r(s_t,a_t)\right]\,.$


\subsection{First order stationarity and global optimality}

The policy optimization problem is differentiable and smooth (under some regularity conditions we precise later on), but non-concave in general. Therefore, several prior works (see Appendix for a review) focused on showing convergence to an $\varepsilon$-approximate FOSP of the expected return function~$J$, i.e. to a parameter~$\theta$ such that (s.t.) $\|\nabla J(\theta)\| \leq \varepsilon\,.$ In this paper, we focus on finding a parameter~$\theta$ s.t. 
$
J^* - J(\theta) \leq \varepsilon + \mathcal{O}(\sqrt{\varepsilon_{\text{bias}}})\,,
$ 
where~$J^*$ is the optimal expected return (maximizing~$J(\pi)$) 
and the~$\mathcal{O}(\sqrt{\varepsilon_{\text{bias}}})$ error term is due to the approximation power of the policy parameterization (see Assumption~\ref{hyp:tranf-compatib-fun-approx}).

\subsection{Policy gradient and Hessian}

Under mild regularity assumptions on the policy parameterization,
it can be shown 
that the gradient of the truncated expected return function~$J_H$ w.r.t. (with respect to) the policy parameter~$\theta$ and the Hessian at~$\theta$ (see \cite{shen-et-al19}) can be written:
\begin{eqnarray}
\nabla J_H(\theta) &=& \bb E_{\rho,\pi_{\theta}}[g(\tau, \theta)], \label{eq:g-tau-theta}\\
g(\tau, \theta) & =& \sum_{t=0}^{H-1} \left( \sum_{h=t}^{H-1} \gamma^h r(s_h,a_h) \right) \nabla \log \pi_\theta(a_t|s_t)\,,\nonumber
\end{eqnarray}
\begin{eqnarray}
\nabla^2 J_H (\theta) &=& \bb E_{\rho,\pi_{\theta}}[B(\tau, \theta)],\label{eq:B-tau-theta}\\
B(\tau, \theta) &\eqdef& \nabla \Phi(\tau, \theta)\,  \nabla \log p(\tau|\pi_{\theta})^T + \nabla^2 \Phi(\tau, \theta) \, \nonumber
\end{eqnarray}
with 
$$\Phi(\tau, \theta) \eqdef \sum_{t=0}^{H-1} \left( \sum_{h=t}^{H-1} \gamma^h r(s_h,a_h) \right) \log \pi_\theta(a_t|s_t)$$ 
and~$\tau$ is an $H$-length simulated trajectory from the MDP with policy~$\pi_{\theta}$ and initial distribution~$\rho$.
Note that all the derivatives are w.r.t.~$\theta$ and the full gradient of~$J$ cannot be computed due to the infinite horizon length. An unbiased estimator of the gradient~$\nabla J_H(\theta)$ above is then given by~$g(\tau, \theta)$ as defined in Eq.~\ref{eq:g-tau-theta}.

\section{Normalized Momentum-Based Policy Gradient Algorithms}
\label{sec:methods}

In this section, we present our two momentum-based policy gradient algorithms: \algname{N-PG-IGT} and \algname{N-HARPG}. These methods share two key algorithmic features: \textit{normalization} and \textit{momentum}. Moreover, while \algname{N-PG-IGT} uses a look-ahead mechanism to leverage second-order smoothness, \algname{N-HARPG} uses variance reduction with Hessian correction in order to improve sample complexity. 



\subsection{Normalized PG with Implicit Gradient Transport}

We start by introducing a policy gradient algorithm for policy optimization called Normalized-PG-with Implicit Gradient Transport (\algname{N-PG-IGT} for short). 
We propose the \algname{N-PG-IGT} algorithm for solving our RL policy optimization problem, inspired by the recently proposed \algname{NIGT} method which was considered in the context of stochastic optimization (with i.i.d samples from a fixed distribution) \citep{MomentumImprovesNSGD_Cutkosky_2020,IGT_Arnold_2019}. 

\noindent\textbf{Motivation.} Before describing the \algname{N-PG-IGT} algorithm, we provide some motivation. When having access to exact policy gradients, simple normalized PG enjoys a geometric global convergence rate \citep{mei-et-al21} thanks to the non-uniform smoothness of the policy optimization objective. It is then natural to investigate the properties of this algorithm in the more realistic setting where only stochastic policy gradients can be computed using samples from the MDP. In this stochastic setting, the understanding of this method is limited. Therefore, it is instructive to discuss known results from stochastic optimization regarding the normalized stochastic gradient method. As explained by \citet[Sec.~2]{MomentumImprovesNSGD_Cutkosky_2020}, the issue with the simple normalized stochastic gradient method is that the gradient estimate error before normalization may be larger than the true gradient. To fix this issue, several works in the literature propose to reduce the variance of the stochastic gradients either by using very large batch sizes (see for e.g. \cite{Beyond_Convexity_Hazan_2015}) or by using momentum without resorting to large batches as recently suggested in \citep{MomentumImprovesNSGD_Cutkosky_2020}. Back to our policy optimization problem, we investigated the sample complexity of the normalized PG with momentum for global optimality beyond the first-order stationarity convergence rates.  
However, our results for this algorithm only match the performance of vanilla PG. We defer the discussion of these results to Appendix~\ref{sec:NSGDM_app}, where we formally introduce this Algorithm~\ref{alg:N-MPG} and prove convergence guarantees in Theorem~\ref{thm:N-MPG}. Given these elements of motivation, we can now present the~\algname{N-PG-IGT} algorithm which also utilizes momentum with normalization and will allow us to derive improved convergence rates.
\begin{algorithm}[t]
        \caption{\algname{N-PG-IGT} \\(Normalized-PG-with Implicit Gradient Transport)}\label{alg:N-PG-IGT}
        \begin{algorithmic}[1]
            \STATE \textbf{Input}: $\theta_0$, $\theta_1$, $d_0$, $T$, $\cb{\eta_t}_{t\geq 0}$, $\cb{\gamma_t}_{t\geq 0}\,$
            \FOR{$t=1, \ldots, T -  1$}
                \STATE \label{n-pg-igt:theta-tilde} $\tilde{\theta}_t = \theta_t + \fr{1- \eta_t}{\eta_t}  \rb{\theta_t - \theta_{t-1}}$ 
                \STATE $\tilde{\tau}_t \sim p(\cdot|\pi_{\tilde{\theta}_t})$ 
                \STATE \label{n-pg-igt:step-d-t} $d_t = (1 - \eta_t) d_{t-1} + \eta_t g(\tilde{\tau}_t, \tilde{\theta}_t)$ 
                \STATE \label{n-pg-igt:step-update} $\theta_{t+1} = \theta_t + \gamma_t \frac{d_t}{\|d_t\|}$ 
		    \ENDFOR
		 \RETURN $\theta_T$
        \end{algorithmic}
\end{algorithm}

\textbf{\algname{N-PG-IGT} (Algorithm~\ref{alg:N-PG-IGT}).} This algorithm performs a normalized gradient ascent in the policy parameter space following the direction of a normalized moment-based stochastic policy gradient. 
Unlike in a simple normalized momentum PG algorithm, the direction~$d_t$ is computed using the stochastic policy gradient~$g(\tilde{\tau}_t, \tilde{\theta}_t)$ instead of~$g(\tau_t,\theta_t)$ with~$\tau_t \sim p(\cdot|\pi_{\theta_t})$.
Here, the auxiliary sequence~$\tilde{\theta}_t$ is crucial for faster convergence. Observe that its update rule (Step~\ref{n-pg-igt:theta-tilde}) can be rewritten~$\theta_t = \eta_t \tilde{\theta}_t + (1-\eta_t) \theta_{t-1}$ which shows that using~$\tilde{\theta}_t$ can be seen as performing a \textit{lookahead} step extrapolating from previous iterates~$\theta_t$ and~$\theta_{t-1}$.  We refer the reader to \citep[Sec.~3]{MomentumImprovesNSGD_Cutkosky_2020} for more intuition regarding this scheme. 
Notice that the algorithm does not implement any importance sampling mechanism to account for the non-stationarity of our RL problem (which makes RL different from standard stochastic optimization). The algorithm samples a single trajectory at each iteration and does not require large batches.

\subsection{Hessian-Aided Recursive Policy Gradient}
\label{subsec:harpg}

We now present two variants of Hessian-Aided policy gradient algorithms, namely Hessian-Aided Recursive Policy Gradient (\algname{HARPG}) and its normalized version  
(\algname{N-HARPG}). 
We remark here that~\algname{N-HARPG} was previously proposed in~\citet{salehkaleybar-et-al22} under the name~\algname{SHARP}
\footnote{We use a different name for this algorithm since we are also studying the algorithm without normalization (namely \algname{HARPG}) and we use different step-sizes and momentum paramemeters for the normalized algorithm in our analysis.}.

\textbf{\algname{(N)-HARPG} (see Algorithm~\ref{alg:(N)-HARPG}).} The algorithm is based on an idea similar to the \algname{STORM} variance reduction method \citep{cutkosky-orabona19} and uses  second-order information \citep{BetterSGDUsingSOM_Tran_2021} 
instead of the difference between consecutive stochastic gradients. Interestingly, this allows to bypass the use of IS and the resulting need for controlling IS weights via unverifiable assumptions. To compute the update direction~$d_t$ of the algorithm in step 7, a second-order correction~$(1-\eta_t)v_t$ is added to the momentum stochastic gradient~$(1-\eta_t)d_{t-1} + \eta_t g(\tau_t,\theta_t)\,.$ Crucially, the uniform sampling procedure in steps 3-6 guarantees that~$v_t$ is an unbiased estimator of~$\nabla J(\theta_t) - \nabla J(\theta_{t-1})$, a term which is reminiscent of the one in the original \algname{STORM}.
The last step of the algorithm (step 8) uses the direction~$d_t$ with or without normalization to update the parameter~$\theta_t\,.$

\begin{algorithm}[t]
        \caption{\algname{(N)-HARPG} \\ ((Normalized)-Hessian-Aided Recursive Policy Gradient)}\label{alg:(N)-HARPG}
        \begin{algorithmic}[1]
            \STATE \textbf{Input}: $\theta_0$, $\theta_1$, $d_0$, $T$, $\cb{\eta_t}_{t\geq 0}$, $\cb{\gamma_t}_{t\geq 0}\,$
            \FOR{$t=1, \ldots, T -  1$}
                \STATE $q_t \sim \mathcal{U}([0,1])$
                \STATE $\hat{\theta}_{t} = q_t \theta_{t} + (1-q_t) \theta_{t-1}$
                \STATE $\tau_t \sim p(\cdot|\pi_{\theta_t}); \, \hat{\tau}_t \sim p(\cdot|\pi_{\hat{\theta}_{t}})$
                \STATE \label{n-harpg:hvp} $v_t = B(\hat{\tau}_t, \hat{\theta}_{t})(\theta_t - \theta_{t-1})$
                \STATE $d_t = (1 - \eta_t) \rb{ d_{t-1} +  v_t } + \eta_t g(\tau_t, \theta_t) $
                \STATE $\theta_{t+1} = \begin{cases}
                &\theta_t + \gamma_t d_t \qquad  \text{(\algname{HARPG}) } \\
                &\theta_t + \gamma_t \frac{d_t}{\|d_t\|} \quad  \text{(\algname{N-HARPG}) }
                \end{cases}$
		    \ENDFOR
		 \RETURN $\theta_T$
        \end{algorithmic}
\end{algorithm}

\noindent\textbf{Hessian-vector implementation (Step~\ref{n-harpg:hvp}).} We insist on the fact that the Hessian-vector product defining the correction~$v_t$ can be easily and efficiently implemented in~$\cO(Hd)$ computational complexity. 
Observe for this that for any arbitrary vector~$u \in \R^d$, any trajectory~$\tau$ and any parameter~$\theta \in \R^d$, we can write the Hessian-vector product~$B(\tau,\theta)\, u$ using Eq.~\ref{eq:B-tau-theta} as follows:  
\begin{equation*}
B(\tau, \theta) \, u =  \langle \nabla \log p(\tau|\pi_{\theta}),  u \rangle \, g(\tau, \theta)  + \nabla\, \langle  g(\tau, \theta) , u  \rangle\,.
\end{equation*}
The second term can be easily computed via automatic differentiation of the scalar quantity~$\langle  g(\tau, \theta) , u  \rangle$ instead of using a finite difference approximation as in prior work \citep{shen-et-al19,huang-et-al20}. 
Computing and storing the full Hessian estimate~$B(\hat{\tau}_t,\hat{\theta}_t)$ is not needed. 
Exploiting curvature information from the policy Hessian does not compromise the~$\mathcal{O}(Hd)$ per-iteration computation cost.


\section{Global Convergence Analysis}
\label{sec:convergence_analysis}

In this section, we establish the sample complexity of our algorithms to reach global optimality, i.e., the number of 
samples 
needed to achieve an $\varepsilon$-approximate global optimum in expectation up to a bias due to policy parameterization. 
To provide such global optimality guarantees, we 
leverage a structural property of the non-concave policy optimization problem under the form of a relaxed weak gradient dominance condition satisfied by the expected return~$J$. 

\subsection{Assumptions}
\label{subsec:assumptions_main}

Our first assumption is a standard assumption about the regularity of the policy parameterization.\footnote{See Appendix~\ref{sec:discussion-assumptions} for a more detailed discussion of all our assumptions.}

\begin{assumption}[Policy parameterization regularity]
  \label{hyp:parameterization-reg}
  For every~$s, a \in \mathcal{S} \times \mathcal{A},$ the function~$\theta \mapsto \pi_{\theta}(a|s)$ is positive and continuously differentiable. Moreover, 
    \begin{enumerate}[(i)]
    \item \label{hyp:param-reg-1st-order} there exists~$M_g>0$ such that for every~$\theta \in \R^d, s \in \mathcal{S}, a \in \mathcal{A}$,
$\|\nabla \log \pi_{\theta}(a|s) \| \leq M_g\,.$

    \item \label{hyp:param-reg-2nd-order} the function~$\theta \mapsto \pi_{\theta}(a|s)$ is twice continuously differentiable for every~$s, a \in \mathcal{S} \times \mathcal{A},$ and there exists~$M_h>0$ such that for every~$\theta \in \R^d$ and every $s \in \mathcal{S}, a \in \mathcal{A}$,
$\|\nabla^2 \log \pi_{\theta} (a|s)\| \leq M_h \,.$
    \end{enumerate}
\end{assumption}
Under this assumption, we collect some known results
from the literature (see for e.g., \cite{papini-et-al18,shen-et-al19,huang-et-al20,xu-et-al20iclr,Vanilla_PL_Yuan_21}) regarding properties of the objective function~$J$ and its stochastic policy 
gradient in the next proposition.

\begin{proposition}
  \label{prop:stoch-grad-and-J}
Under Assumption~\ref{hyp:parameterization-reg}-(\ref{hyp:param-reg-1st-order}), 
\begin{enumerate}[(i)]
    \item \label{prop-i} Function~$J$ is $L_g$-smooth with~$L_g = \frac{r_{\max}(M_g^2 + M_h)}{(1-\gamma)^2}$.
    \item \label{prop-iii} $\Exp{\|g(\tau|\theta) - \nabla J_H(\theta)\|^2} \leq \sigma_g^2$ with~$\sigma_g^2 = \frac{r_{\max}^2 M_g^2}{(1-\gamma)^3}\,.$
    \item \label{prop-iv} $\Exp{\|B(\tau|\theta) - \nabla^2 J_H(\theta)\|^2 } \leq \sigma_h^2$ with~$\sigma_h^2 = \frac{r_{\max}^2 (H^2 M_g^4 + M_h^2)}{(1-\gamma)^4}$ if  Assumption~\ref{hyp:parameterization-reg}-(\ref{hyp:param-reg-2nd-order}) also holds.
\end{enumerate}
\end{proposition}

Items~(\ref{prop-i}) and~(\ref{prop-iii}) are stated in \citet[Lemma 4.2, 4.4]{Vanilla_PL_Yuan_21}, item~(\ref{prop-iv}) 
in \citet[Lemma 4.1]{shen-et-al19}.
  
The following assumption about the second-order regularity of the policy parameterization will be useful to establish some of our results.  This assumption is also used in several other works (see for e.g., \cite{zhang-et-al20glob-conv-pg,Masiha_SCRN_KL,yang-zheng-pan21}).
\begin{assumption}
\label{hyp:lipschitz-hessian}
For every~$(s, a) \in \mathcal{S} \times \mathcal{A},$ the function~$\theta \mapsto \pi_{\theta}(a|s)$ is positive, twice continuously differentiable and there exists~$l_2>0$ such that for every~$\theta, \theta' \in \R^d$ and every~$s \in \mathcal{S}, a \in \mathcal{A}$,
$$
\|\nabla^2 \log \pi_{\theta} (a|s) - \nabla^2 \log \pi_{\theta'} (a|s)\| \leq l_2\, \|\theta - \theta'\|\,.
$$
\end{assumption}

Reinforcing our regularity assumption~\ref{hyp:parameterization-reg} with Assumption~\ref{hyp:lipschitz-hessian}, the expected return~$J$ will further have a Lipschitz Hessian.
Notice that both assumptions can be satisfied by a Gaussian policy for instance.  

\begin{lemma}[\citet{zhang-et-al20glob-conv-pg}]
\label{lem:2nd-order-smoothness}
Let Assumptions~\ref{hyp:parameterization-reg} and~\ref{hyp:lipschitz-hessian} hold. Then for every~$\theta, \theta' \in \R^d$,
$
\|\nabla^2 J(\theta)- \nabla^2 J(\theta')\| \leq L_h\|\theta - \theta'\|\,, 
$ 
with $L_h = \cO\rb{(1-\g)^{-3}}$.
\end{lemma}

In this paper, we focus on the general class of Fisher-non-degenerate parameterized policies for our convergence analysis. To state our assumption, we define the Fisher information matrix~$F_{\rho}(\theta)$ induced by a policy~$\pi_{\theta}$ for some~$\theta \in \R^d$  and an initial distribution~$\rho$ as follows: 
\begin{equation*}
F_{\rho}(\theta) \eqdef \bb E_{s \sim d_{\rho}^{\pi_\theta},\, a \sim \pi_{\theta}(\cdot|s)} [\nabla \log \pi_{\theta}(a|s) \nabla \log \pi_{\theta}(a|s)^{\top}]\,,
\end{equation*}
where~$d_{\rho}^{\pi_\theta}(\cdot) \eqdef (1-\gamma) \sum_{t = 0}^{\infty} \gamma^t \bb P_{\rho,\pi_{\theta}}(s_t \in \cdot)$ is the discounted state visitation measure.
\begin{assumption}[Fisher-non-degenerate policy]
\label{hyp:fisher-non-degenerate}
There exists~$\mu_F >0$ s.t. for all~$\theta \in \R^d,$  
$F_{\rho}(\theta) \succcurlyeq  \mu_F \cdot I_d\,.$ 
\end{assumption}

Assumption~\ref{hyp:fisher-non-degenerate} is satisfied by Gaussian policies~$\pi_{\theta}(\cdot|s) = \mathcal{N}(\mu_{\theta}(s),\Sigma)$ when the parameterized mean~$\mu_{\theta}(s)$ has a full-row rank Jacobian and the covariance matrix~$\Sigma$ is fixed. It also holds for a broader subclass of exponential family parameterized policies (see~\citet[Sec.~8]{ding-et-al22} for more details) and for certain neural policies (see also~\citet[Sec.~B.2]{liu-et-al20}).
This assumption is common in the literature \citep{zhang-et-al20glob-conv-pg,liu-et-al20,ding-et-al22,Vanilla_PL_Yuan_21,Masiha_SCRN_KL}. We highlight that the global convergence results we will derive in this paper may not hold for softmax parameterization which does not satisfy Assumption~\ref{hyp:fisher-non-degenerate} as soon as it gets closer to a deterministic policy. 

Following prior work \citep{agarwal-et-al21,liu-et-al20, ding-et-al22,Vanilla_PL_Yuan_21}, we introduce an assumption reflecting the expressivity of our policy parameterization class via the framework of compatible function approximation \citep{PGM_Sutton_1999,agarwal-et-al21}. 
\begin{assumption}
\label{hyp:tranf-compatib-fun-approx}
There exists~$\varepsilon_{\text{bias}} \geq 0$ s.t. for every~$\theta \in \R^d$, the transfer\footnote{in the sense that the compatible function approximation error is evaluated under the \textit{shifted} distribution induced by~$d_{\rho}^{\pi^*}$ and~$\pi^*$.}
error satisfies: 
\begin{equation*}
\bb E[ (A^{\pi_{\theta}}(s,a) - (1-\gamma)w^*(\theta)^{\top}\, \nabla \log \pi_{\theta}(a|s))^2] \leq \varepsilon_{\text{bias}}\,,
\end{equation*}
where~$A^{\pi_{\theta}}$ is the advantage function\footnote{We delay its definition to the appendix, it is only used here.}, 
$w^*(\theta) \eqdef F_{\rho}(\theta)^{\dag} \nabla J(\theta)$ where~$F_{\rho}(\theta)^{\dag}$ is the pseudo-inverse of the matrix~$F_{\rho}(\theta)$ and expectation is taken over $ s \sim d_{\rho}^{\pi^*},\, a \sim \pi^*(\cdot|s)$ where~$\pi^*$ is an optimal policy (maximizing~$J(\theta)$). 
\end{assumption}

The common Assumption~\ref{hyp:tranf-compatib-fun-approx} states that the policy parameterization~$\pi_{\theta}$ allows to nearly approximate the advantage function~$A^{\pi_{\theta}}$ from the score function~$\nabla \log \pi_{\theta}$. 
Importantly, we precise that~$\varepsilon_{\text{bias}}$ is positive for a policy parameterization~$\pi_{\theta}$ that does not cover the set of all stochastic policies and~$\varepsilon_{\text{bias}}$ is small for a rich neural policy \citep{wang-et-al20}. 
We refer the reader to the aforementioned references for further details regarding Assumption~\ref{hyp:tranf-compatib-fun-approx}.

Combining Assumptions~\ref{hyp:fisher-non-degenerate} and~\ref{hyp:tranf-compatib-fun-approx}, 
and following the derivations of \citet{ding-et-al22}, we can obtain the following relaxed weak gradient dominance inequality. 
\begin{lemma}[Relaxed weak gradient domination, \citep{ding-et-al22}] 
\label{lem:relaxed-w-grad-dom}
Let Assumptions~\ref{hyp:parameterization-reg}-(\ref{hyp:param-reg-1st-order}), \ref{hyp:fisher-non-degenerate} and~\ref{hyp:tranf-compatib-fun-approx} hold. 
Then
\begin{eqnarray}
\forall\, \theta \in \R^d, \quad 
\varepsilon' + \|\nabla J(\theta)\| \geq \sqrt{2\mu}\, (J^* - J(\theta))\,, \notag 
\end{eqnarray}
where~$\varepsilon^{\prime} = \frac{\mu_F \sqrt{\varepsilon_{\text{bias}}}}{M_g (1-\gamma)}\,$ and~$\mu = \frac{\mu_F^2}{2M_g^2}.$ 
\end{lemma}
This result is the key tool in convergence analysis of our algorithms. See Appendix~\ref{sec:appendix_proof_sketch} for more details and the proof sketch.

\subsection{Convergence analysis of \algname{N-PG-IGT}}

Our first main result characterizes the sample complexity of the \algname{N-PG-IGT} algorithm to produce an~$\varepsilon$-globally optimal policy up to a bias error term due to the approximation power of the policy parameterization.\footnote{We also analyze convergence of this method to a FOSP for a wider class of policy parameterizations (under Assumptions~\ref{hyp:parameterization-reg} and \ref{hyp:lipschitz-hessian}) in Theorem~\ref{thm:NIGT_FOSP} in Appendix~\ref{sec:nigt_app}.} 

\begin{theorem}[\textbf{Global convergence of \algname{N-PG-IGT}}]
    \label{thm:N-PG-IGT}
    Let Assumptions~\ref{hyp:parameterization-reg} to~\ref{hyp:tranf-compatib-fun-approx} hold.
    Set~$\gamma_t = \fr{6 M_g}{\mu_F (t+2)},\,\eta_t = \rb{ \fr{2}{t+2} }^{\nfr{4}{5}}$ and~$H = \rb{1-\g}^{-1}{\log(T + 1)}$. Then for every integer~$T \geq 1$, the output~$\theta_T$ of \algname{N-PG-IGT} satisfies 
    \begin{eqnarray}
         J^* -  \Exp{J(\theta_T)} \leq   \cO\rb{ \fr{ \sigma_g  + L_h}{(T+1)^{\nfr{2}{5}}} } + \fr{\sqrt{\varepsilon_{\text{bias}}}}{1-\gamma} , \notag
    \end{eqnarray}
    where~$\sigma_g, L_h$ are defined in Proposition~\ref{prop:stoch-grad-and-J} and Lemma~\ref{lem:2nd-order-smoothness} respectively. 
    Hence, for every~$\varepsilon >0,$ the output~$\theta_T$ satisfies~$J^* - \Exp{J(\theta_T)} \leq \varepsilon + \sqrt{\varepsilon_{\text{bias}}}/(1-\gamma)$ for~$T = \cO(\varepsilon^{-2.5}),$ i.e., the total sample complexity is of the order~$\wcO(\varepsilon^{-2.5})\,.$ 
\end{theorem}
Our analysis of \algname{N-PG-IGT} crucially relies on  second-order smoothness of~$J(\theta)$, which stems from Lemma~\ref{lem:2nd-order-smoothness}. 
\algname{N-PG-IGT} improves over the best known sample complexity of (stochastic) Vanilla-PG \citep{Vanilla_PL_Yuan_21} and \algname{STORM-PG-F} \citep{ding-et-al22}  by a factor of~$\mathcal{O}(\varepsilon^{-0.5})$ for Fisher-non-degenerate
policies. The recently proposed \algname{SCRN} algorithm \citep{Masiha_SCRN_KL} also achieves a~$\tilde{\mathcal{O}}(\varepsilon^{-2.5})$ sample complexity under the similar set of assumptions. However, their algorithm needs a subroutine to approximately solve a cubic regularized subproblem. Moreover, \algname{SCRN} requires large batch sizes and makes use of second-order information. In constrast, \algname{N-PG-IGT} is single-loop, only requires to sample a single trajectory per iteration and does not need any second-order information to be implemented.


\subsection{Convergence analysis of \algname{(N)-HARPG}}

In this section, we further improve the sample complexity for global convergence to~$\tilde{\mathcal{O}}(\varepsilon^{-2})$. We show this result for both \algname{HARPG} and its normalized version~\algname{N-HARPG}.
\begin{theorem}[\textbf{Global convergence of \algname{HARPG}}]
    \label{thm:HARPG}
    Let Assumptions~\ref{hyp:parameterization-reg}, \ref{hyp:fisher-non-degenerate} and~\ref{hyp:tranf-compatib-fun-approx} hold.
    Set~$\gamma_t = \g_0 \eta_t^{\nfr{1}{2}} $, $\eta_t = \fr{2  }{t+2} $ with $\g_0 = \min\cb{ \fr{1}{ 8 \sqrt{6} ( L_g + \sigma_g + D_h \g^H ) } , \fr{\sqrt{2}M_g}{\sqrt{3}\sigma_g \mu_F}  } $,
    $H = \rb{1-\g}^{-1}{\log(T + 1)}$. Then for every integer~$T \geq 1$, the output~$\theta_T$ of \algname{HARPG} (see Algorithm~\ref{alg:(N)-HARPG}) satisfies
    \begin{multline*}
        J^* - \Exp{ J(\theta_T) } \leq \cO
        \left(
        \fr{  \sigma_g  + L_g + \sigma_h }{(T+1)^{\nfr{1}{2}}} \right) + \fr{2 \sqrt{\varepsilon_{\text{bias}}}}{1-\gamma} ,
    \end{multline*}
    where~$\sigma_g, \sigma_h, L_g$ are defined in Proposition~\ref{prop:stoch-grad-and-J} and $D_h$~is defined in Lemma~\ref{le:trunc_grad_hess} (see Appendix~\ref{sec:notation_useful_lemma}).
\end{theorem}
\begin{theorem}[\textbf{Global convergence of \algname{N-HARPG}}]
    \label{thm:N-HARPG}
    Let Assumptions~\ref{hyp:parameterization-reg}, \ref{hyp:fisher-non-degenerate} and~\ref{hyp:tranf-compatib-fun-approx} hold.
    Set~$\gamma_t = \fr{6 M_g}{\mu_F(t+2)}$ and~$\eta_t = \fr{2}{t+2} $, $H = \rb{1-\g}^{-1}{\log(T + 1)}$. Then for every integer~$T \geq 1$, the output~$\theta_T$ of \algname{N-HARPG} (see Algorithm~\ref{alg:(N)-HARPG}) satisfies
    \begin{multline*}
        J^* - \Exp{ J(\theta_T) } \leq \cO
        \left( \fr{  \sigma_g  + L_g + \sigma_h }{(T+1)^{\nfr{1}{2}}} \right) + \fr{\sqrt{\varepsilon_{\text{bias}}}}{1-\gamma} ,
    \end{multline*}
    where~$\sigma_g, \sigma_h, L_g$ are defined in Proposition~\ref{prop:stoch-grad-and-J}.
\end{theorem}

\begin{corollary}
In the setting of Theorem~\ref{thm:HARPG} (Theorem~\ref{thm:N-HARPG}), for every~$\varepsilon >0,$ the output~$\theta_T$ of \algname{(N)-HARPG} satisfies~$J^* - \Exp{J(\theta_T)} \leq \varepsilon + 2\sqrt{\varepsilon_{\text{bias}}}/(1-\gamma)$ for~$T = \cO(\varepsilon^{-2}),$ i.e., the total sample complexity is of the order~$\wcO(\varepsilon^{-2})\,.$
\end{corollary}

We remark that while \algname{(N)-HARPG} uses Hessian information, our analysis does not require the second-order smoothness of~$J(\theta)$ and, thus, Assumption~\ref{hyp:lipschitz-hessian}. This is achieved by the additional uniform sampling in steps 3-5 of the method. 

\begin{figure*}[ht]
    \centering 
    \includegraphics[width=0.97\textwidth]{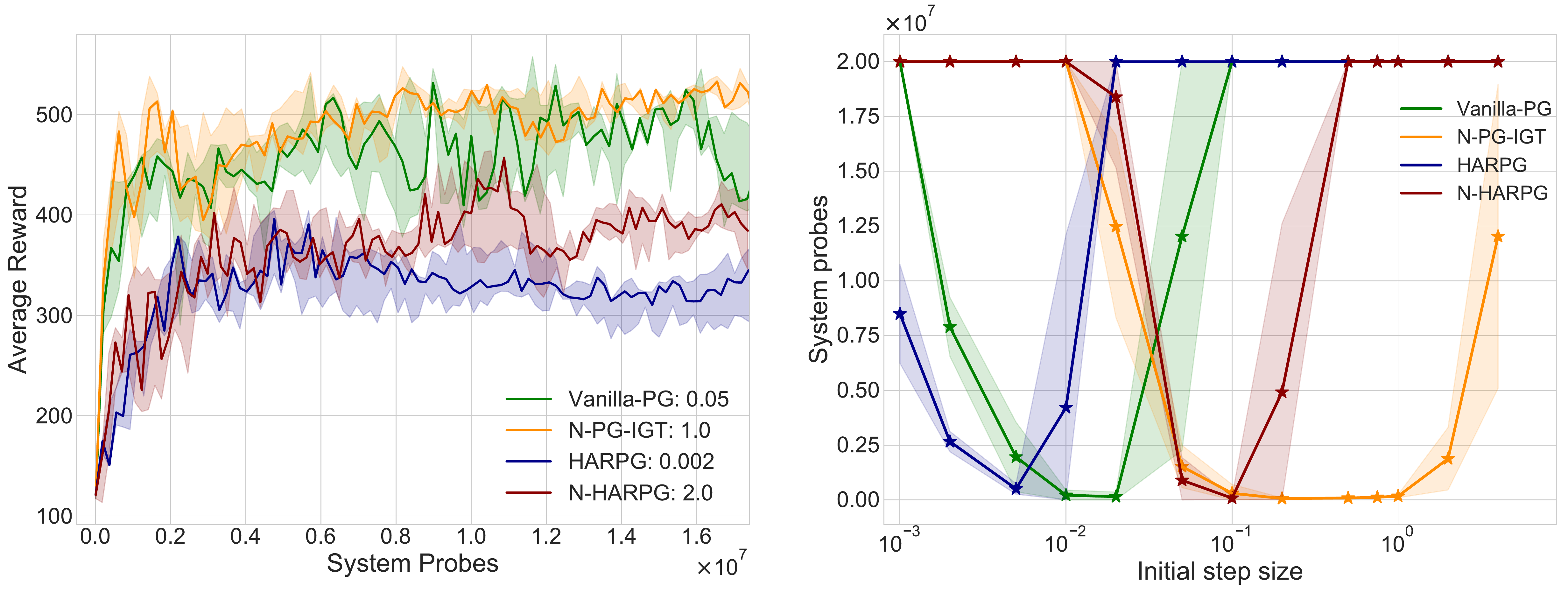}
        \caption{Left: performance of \algname{Vanilla-PG}, \algname{N-PG-IGT}, \algname{HARPG}, \algname{N-HARPG} on \texttt{Humanoid} environment. Right: robustness to initial step-sizes for proposed methods on \texttt{Reacher} environment. } 
        \label{fig:sz_robust}
\end{figure*}

\vspace{-.2cm}
\section{Experiments}
\label{sec:experiments}

We empirically test the performance of the methods on benchmark RL tasks with continuous state-action spaces. We use the class of diagonal Gaussian policies where actions are sampled from
$$
a=\mu_\theta(s)+\sigma_\theta(s) \odot z, \quad z \sim \mathcal{N}(0, I), 
$$
for every~$\theta \in \R^d, s \in \mathcal{S} = \R^p, a \in \mathcal{A} = \R^q\,,$
where $\mathcal{N}(0, I)$ is the multivariate normal distribution, $\odot$ denotes the elementwise product of two vectors, $\mu_\theta(s)$ and $\sigma_{\theta}(s)$ are parameterized by fully-connected neural networks. We implement the algorithms based on \algname{Vanilla-PG} (\algname{REINFORCE}) implementation in the garage library \citep{garage} and test the methods on the commonly used MuJoCo environments. In this section, we present the results on \texttt{Humanoid} and \texttt{Reacher} environments and defer the results on other tasks to Appendix~\ref{sec:extra_experiments}. For all methods, at each iteration we sample $20$ trajectories per iteration, and each trajectory has a maximum length $H = 500$. To ensure the same per-iteration cost for all methods, in 
\algname{(N)-HARPG} we sample half of the trajectories to compute the stochastic gradient and another half to estimate the Hessian-vector product. For a fair comparison of the different methods, we start all our runs from the same initial policy $\pi_{\theta_0}$, where~$\theta_0$ is randomly initialized. 
For all methods, we use time-varying step-sizes $\g_t$ and momentum parameters $\eta_t$, which are chosen according to theory for each method, see Appendix~\ref{sec:extra_experiments}. 

\subsection{Comparison with tuned initial step-sizes}
In the first experiment, we compare the methods with individually tuned initial step-sizes on \texttt{Humanoid} task. We run each algorithm with $13$ different values of initial step-size $\gamma_0$ in the range~$\sb{10^{-3}, 4}$ and select the one with the best \textit{average reward} (averaged over $5$ independent runs for each algorithm and initial step-size) after $2\cdot 10^{7}$ \textit{system probes} (i.e., state transitions). In Figure~\ref{fig:sz_robust} (left), we show the median as well as the~$1/4$ and $3/4$ quantiles (shaded area) of ~the \textit{average reward} depending on the number of \textit{system probes}. We observe that \algname{N-PG-IGT} outperforms other methods and quickly increases the reward above $450$. Compared to \algname{Vanilla-PG}, it stabilizes better around level $500$. Surprisingly, theoretically better \algname{(N)-HARPG} methods perform consistently worse in our experiments with tuned step-sizes. We believe this phenomenon could happen because the stochastic Hessian vector products sampled in \algname{(N)-HARPG} have much higher variance than stochastic gradients in practice, which prevents us from choosing larger step-sizes, see  Remark~\ref{remark:HARPG_performance}.

\subsection{Robustness to initial step-sizes}
In this experiment, we test the robustness of the methods to the choice of initial step-size on \texttt{Reacher} environment. In Figure~\ref{fig:sz_robust} (right), we report the number of \textit{system probes} in order to achieve \textit{average reward} above~$-11$ points depending on the initial step-size~$\g_0$. We plot the median together with the~$1/4$ and~$3/4$ quantiles based on~$5$ independent runs. We observe that (variance-reduced) \algname{HARPG} can reach the desired accuracy faster than \algname{Vanilla-PG} for small initial step-sizes~$\gamma_0$. In contrast, \algname{N-PG-IGT} tolerates larger step-sizes than other methods and exhibits more robustness to initial step-sizes.

\begin{remark}\label{remark:HARPG_performance}
It was previously reported that variance-reduced methods including \algname{N-HARPG} (named \algname{SHARP} in \citep{salehkaleybar-et-al22}) outperform \algname{Vanilla-PG} \citep{shen-et-al19,huang-et-al20} without tuned step-sizes. We also observe similar behavior for small untuned step-sizes, see Figure~\ref{fig:small_step_size} in Appendix~\ref{sec:extra_experiments}. 

\end{remark}
\section{Concluding Remarks}

In this work, we improved the existing~$\tilde{\mathcal{O}}(\varepsilon^{-3})$ global optimality sample complexity 
for Fisher-non-degenerate parametrized policies. 
We proposed two single-loop momentum-based PG methods respectively achieving~$\tilde{\mathcal{O}}(\varepsilon^{-2.5})$ and~$\tilde{\mathcal{O}}(\varepsilon^{-2})$ sample complexities while only using one or two trajectories per iteration. Notably, our analysis does not rely on the unrealistic assumption of bounded IS weights variance unlike most prior work. To achieve our best~$\tilde{\mathcal{O}}(\varepsilon^{-2})$ sample complexity, our analysis hinges on a careful combination of a structural weak gradient dominance condition with the design of variance-reduced policy gradients. Future directions of research include relaxing our policy regularity conditions \citep{zhang-et-al22pg-weakly-smooth} and studying the convergence of more efficient policy search methods for continuous control with non-compact spaces and heavy-tailed policy parameterization, enhancing exploration in the light of  recent results in this direction~\citep{bedi-et-al21heavy-tailed-policy-search}. 

\section*{Acknowledgements}
This work was supported by ETH AI Center doctoral fellowship, ETH Foundations of Data Science (ETH-FDS), and ETH Research Grant funded through the ETH Zurich Foundation.


\bibliographystyle{plainnat}
\bibliography{biblio.bib}

\clearpage
\onecolumn

\tableofcontents

\newpage
\begin{center}
    \bfseries\Large Appendix
\end{center}

\appendix

\section{Additional Experiments and Implementation Details}\label{sec:extra_experiments}

Complementing the empirical results presented in Section~\ref{sec:experiments}, in this section, we present an additional set of experiments on MuJoCo environments and provide further implementation details. 


\noindent\textbf{Implementation details.} The parameters $\mu_\theta(s)$ and $\sigma_\theta(s)$ of Gaussian policy are parametrized with the feed forward neural network with two hidden layers of size 64 and $\tanh$ activation function. 
The summary of parameters for each environment is presented in Table~\ref{tab:hyperparams}.

\begin{table}[H]
	\centering
	\begin{NiceTabular}{|c|c|c|c|c|c|c| }
		\hline
           \makecell{ Environments } & \makecell{ \texttt{Walker2d} }  &  \makecell{ \texttt{Hopper} } & \makecell{\texttt{Halfcheetah}} & \makecell{\texttt{Reacher}} & \makecell{\texttt{Humanoid}} & \makecell{\texttt{Cartpole}} \\
		\hline
		\makecell{ Horizon }  & \makecell{500} &  \makecell{1000} & \makecell{500} & \makecell{500}  & \makecell{500} & \makecell{200} \\
        \makecell{ Number of timesteps }  & \makecell{$2 \times 10^7$ } & \makecell{$2 \times 10^7$ } &\makecell{$2\times 10^7$} &\makecell{$2\times 10^7$} &\makecell{$2\times 10^7$} &\makecell{$2\times 10^6$}   \\
        \makecell{ Batch size }  & \makecell{20 } & \makecell{10 } &\makecell{10} &\makecell{200} &\makecell{10} &\makecell{5}  \\
		\makecell{ best $\g_0$ for \algname{Vanilla-PG}}  & \makecell{0.02 } & \makecell{0.02 } &\makecell{0.02} &\makecell{0.02} &\makecell{0.05} &\makecell{0.1}  \\
		\makecell{ best $\g_0$ for \algname{N-PG-IGT} }  & \makecell{0.75 } & \makecell{0.5 } &\makecell{1.0} &\makecell{0.2} &\makecell{1.0}&\makecell{0.2}    \\
		\makecell{ best $\g_0$ for \algname{HARPG} }  & \makecell{0.002 } & \makecell{0.002 } &\makecell{0.001}  &\makecell{0.002}  &\makecell{0.002} &\makecell{0.02}   \\
		\makecell{ best $\g_0$ for \algname{NHARPG} }  & \makecell{0.1 } & \makecell{0.05 } &\makecell{0.2}  &\makecell{0.1}  &\makecell{2.0} &\makecell{0.1}   \\
		\hline
    \end{NiceTabular}
    \caption{Hyper-parameters and step-size choice. The initial step-size is chosen from the set of $13$ values by the best performance in the last iteration.  }
\label{tab:hyperparams}
\end{table}

\subsection{Experiment 1: Comparison with tuned initial step-sizes}
We further test the algorithms with the best tuned step-sizes on \texttt{Walker2d} and \texttt{Reacher}. Similarly to the experiment in Section~\ref{sec:experiments}, we run each algorithm five times for every initial step-size $\g_0$ from the set 
\begin{equation}\label{eq:set_of_stepsizes}
\{ 0.001, 0.002, 0.005, 0.01, 0.02, 0.05, 0.1, 0.2, 0.5, 0.75, 1, 2, 4 \} .
\end{equation}
Then the reported step-size is chosen by the performance of the algorithm in the last iteration. In Table~\ref{tab:step_sizes}, we report the choice of time-varying step-sizes and momentum parameters. We refer the reader to Sections~\ref{sec:nigt_app}, \ref{sec:stormhess_app} and \ref{sec:nstormhess_app} for justification of the choice of parameters presented in the table.


\begin{figure}[ht]

       \begin{minipage}{.475\textwidth}
        \centering
        \includegraphics[width=\textwidth]{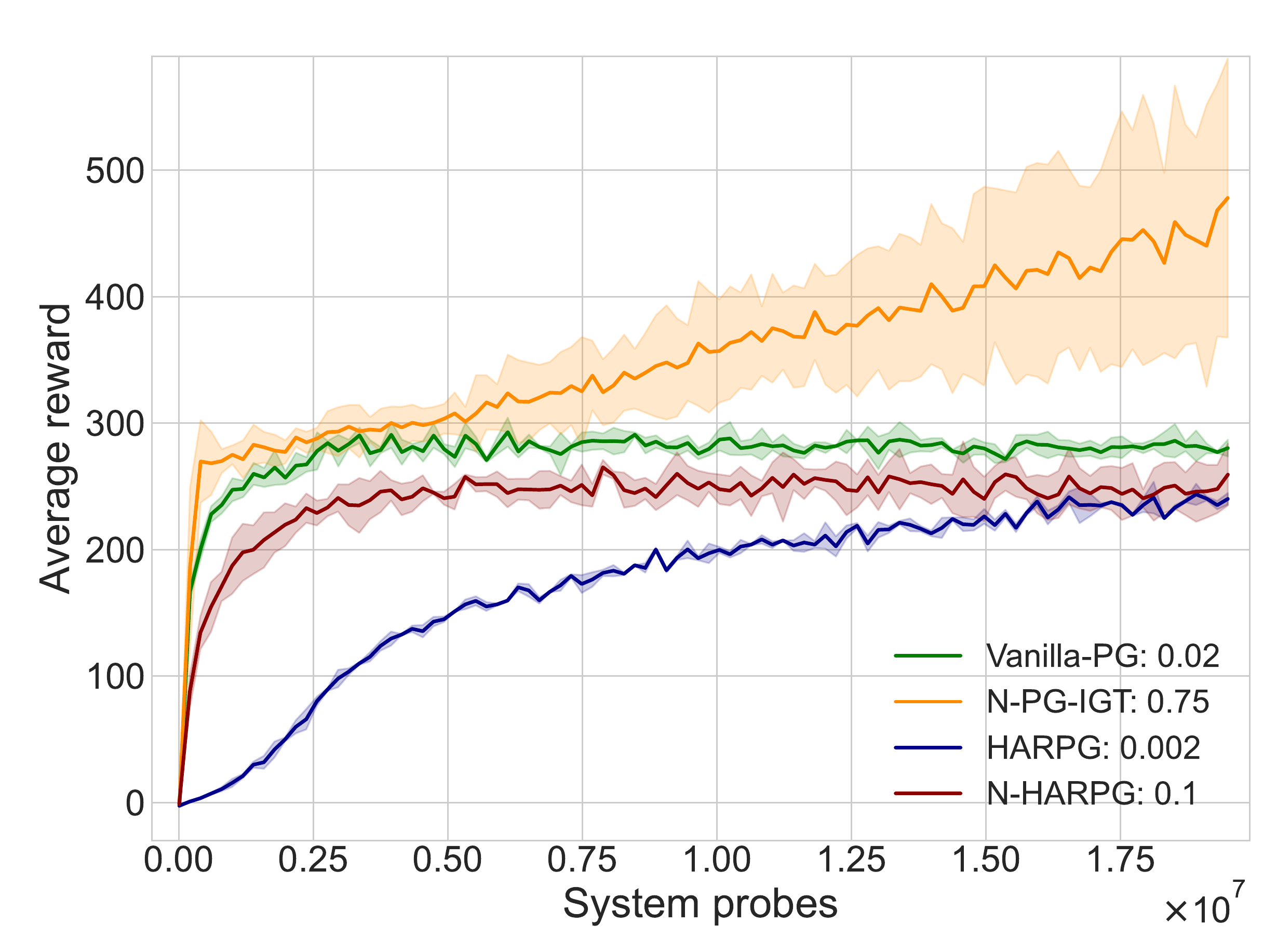}
        \end{minipage}
        \begin{minipage}{.475\textwidth}
        \vspace{8pt}
        \includegraphics[width=1.05\textwidth]{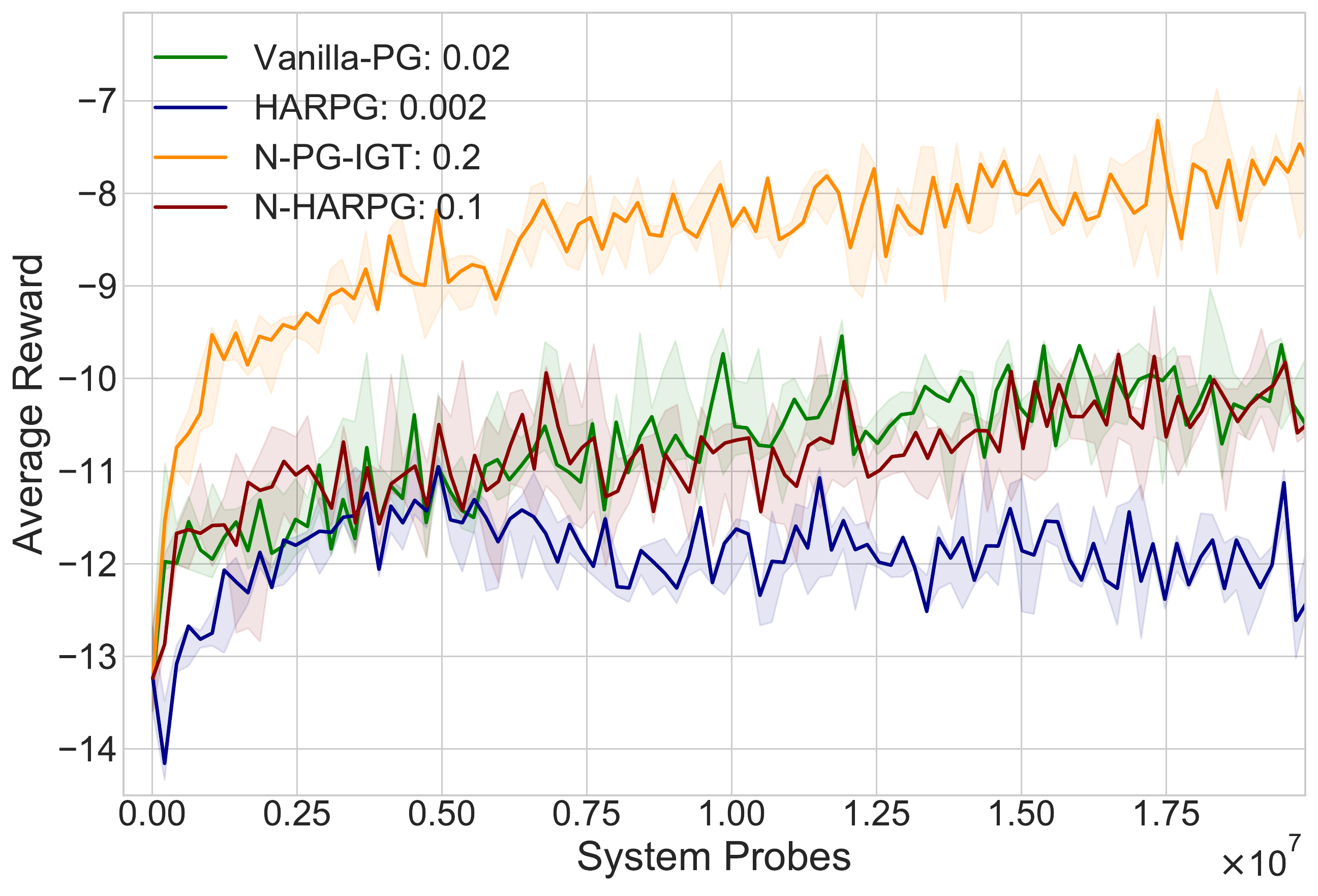}
       \end{minipage}
        \caption{ The performance of \algname{Vanilla-PG}, \algname{N-PG-IGT}, \algname{HARPG}, \algname{N-HARPG} on \texttt{Walker2d} (left) and \texttt{Reacher} (right) environment.}
        \label{fig:sz_walker_reacher}
\end{figure}

Similarly to \texttt{Humanoid} environment presented in Section~\ref{sec:experiments}, we observe that \algname{N-PG-IGT} outperforms other methods and quickly increases the reward above $450$ for \texttt{Walker2d} and above $-9$ for \texttt{Reacher}, see Figure~\ref{fig:sz_walker_reacher}. Unfortunately, theoretically better \algname{HARPG} and \algname{N-HARPG} methods do not show the improvement in practice over \algname{Vanilla-PG}, which might be due to large variance of Hessian-vector estimator.


\begin{table}[H]
	\centering
	\begin{NiceTabular}{|c|cccc|}
		\hline
           & \makecell{ \algname{Vanilla-PG} }  &  \makecell{ \algname{N-PG-IGT} } & \makecell{\algname{HARPG}} & \bf \makecell{ \algname{N-HARPG} } \\
		\hline
		\makecell{ Step-size, $\g_t$ }  & \makecell{$\g_0 \rb{\fr{2}{t+2}}^{\nfr{2}{3}}$} &  \makecell{$ \fr{2\g_0 }{t+2} $} & \makecell{$\g_0 \rb{\fr{2}{t+2}}^{\nfr{1}{2}}$} & \makecell{$ \fr{2 \g_0 }{t+2} $}  \\
	\hline	
        \makecell{ Momentum \\ parameter, $\eta_t$ }  & \makecell{N/A} & \makecell{$ \rb{\fr{2}{t+2}}^{\nfr{4}{5}}$} &\makecell{$ \fr{2}{t+2} $} & \makecell{$ \fr{2}{t+2} $}   \\
		\hline
    \end{NiceTabular}
    \caption{Summary of the sequences of step-sizes and momentum parameters for policy gradient methods. The initial step-size parameter $\gamma_0$ was tuned within the set \eqref{eq:set_of_stepsizes}. The best value of $\gamma_0$ for each algorithm and environment is reported in Table~\ref{tab:hyperparams}.}
    \label{tab:step_sizes}
\end{table}

\subsection{Experiment 2: Robustness to initial step-sizes}
In addition to the experiment on \texttt{Reacher} environment (see Figure~\ref{fig:sz_robust}), we report the sensitivity of the studied methods to initial step-sizes on another two tasks: \texttt{Walker2d} and \texttt{Humanoid}. For these tasks, we observe a similar trend, which shows that \algname{HARPG} can be faster than \algname{Vanilla-PG}, which happens in the small step-size regime, see Figure~\ref{fig:robustness}. On the other hand, the normalized methods tolerate larger step-sizes and \algname{N-PG-IGT}, in particular, demonstrates more robustness to step-size choice. Overall, we observe that, naturally, normalized and non-normalized methods have different range of optimal step-sizes, which makes it necessary to extensively tune the step-sizes for comparing between these two types of algorithms. Another observation is that for the tested environments, the variance reduced \algname{HARPG} method performs better than \algname{Vanilla-PG} only for small non-tuned step-sizes.  Therefore, in the next experiment, we compare these methods in the small step-size regime.  

\begin{figure}[ht]
       \begin{minipage}{.475\textwidth}
       \hspace{5pt}
        \includegraphics[width=\textwidth]{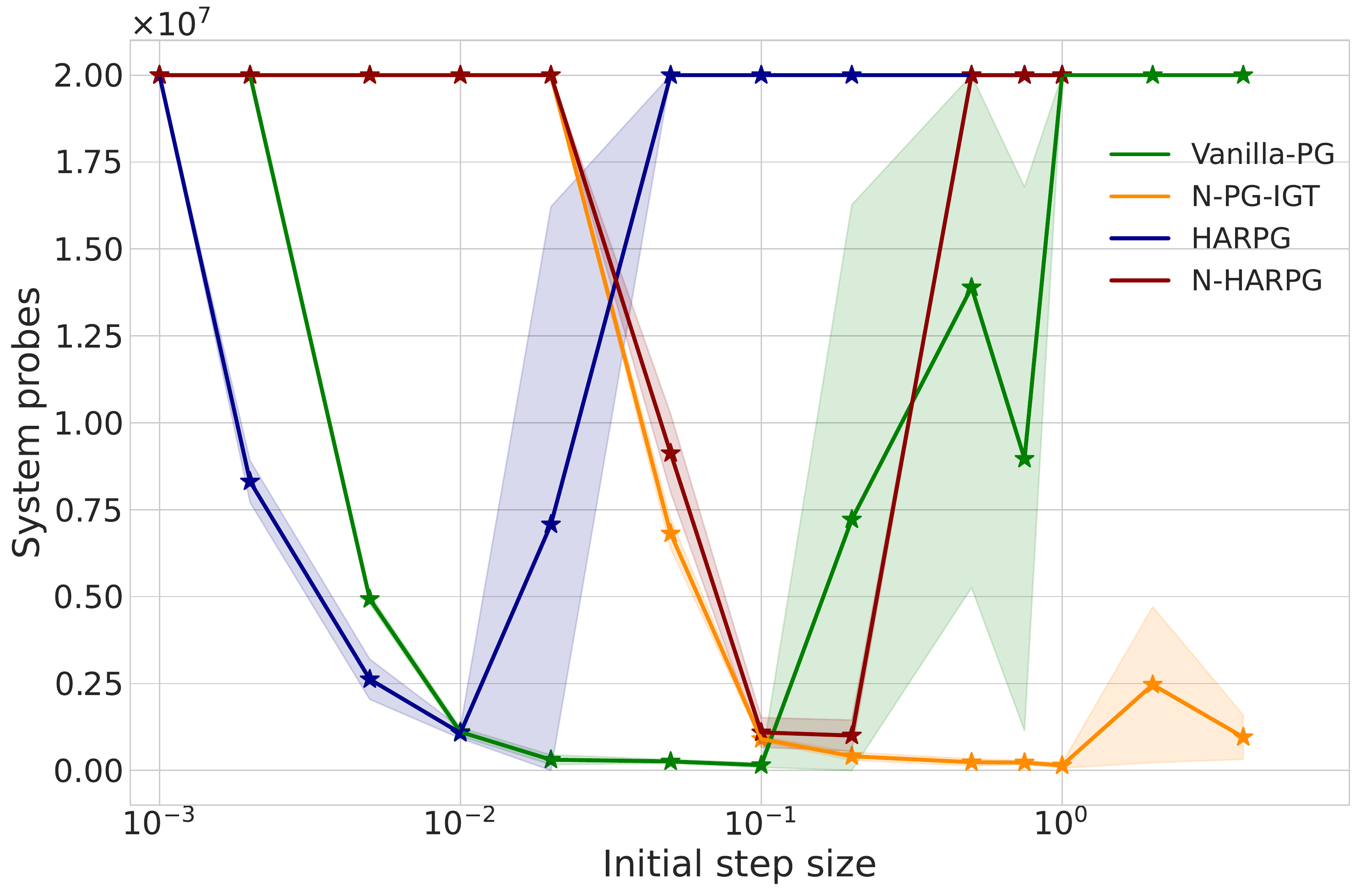}
        \end{minipage}
        \begin{minipage}{.475\textwidth}
        \hspace{10pt}
        \includegraphics[width=0.97\textwidth]{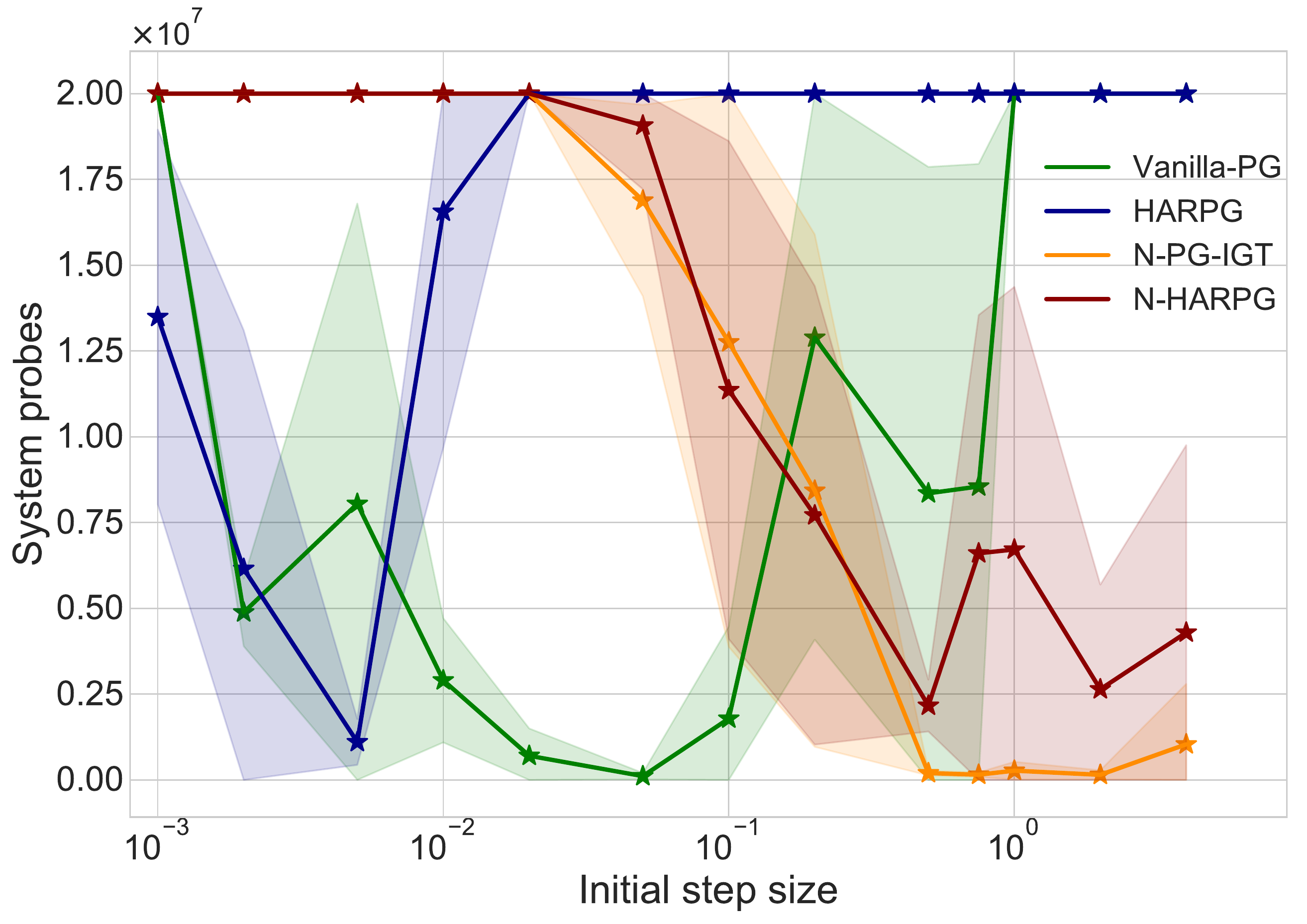}
       \end{minipage}
        \caption{Robustness to initial step-sizes on \texttt{Walker2d} (left) and \texttt{Humanoid} (right) environments.}
        \label{fig:robustness}
\end{figure}

\subsection{Experiment 3: Comparison of \algname{Vanilla-PG} and \algname{HARPG} with small untuned step-sizes}

In this experiment, we demonstrate that the variance reduced \algname{HARPG} method can outperform \algname{Vanilla-PG} for a fixed small initial step-size $\gamma_0$, see Figure~\ref{fig:small_step_size}.\footnote{We do not include normalized algorithms in this comparison since normalization introduces an additional scaling of the step-size.  } Our observations here are consistent with those reported in the previous work for other variance reduced policy gradient algorithms, e.g., \citep{shen-et-al19,huang-et-al20,salehkaleybar-et-al22}. However, when the step-sizes are tuned, \algname{Vanilla-PG} preforms consistently better than variance reduced algorithms, see also Experiment 4 (Figure~\ref{fig:sz_hopper_halfcheetah}) for additional comparison to other methods with tuned step-sizes.

\begin{figure}[h]
    \centering 
    \includegraphics[width=0.95\textwidth]{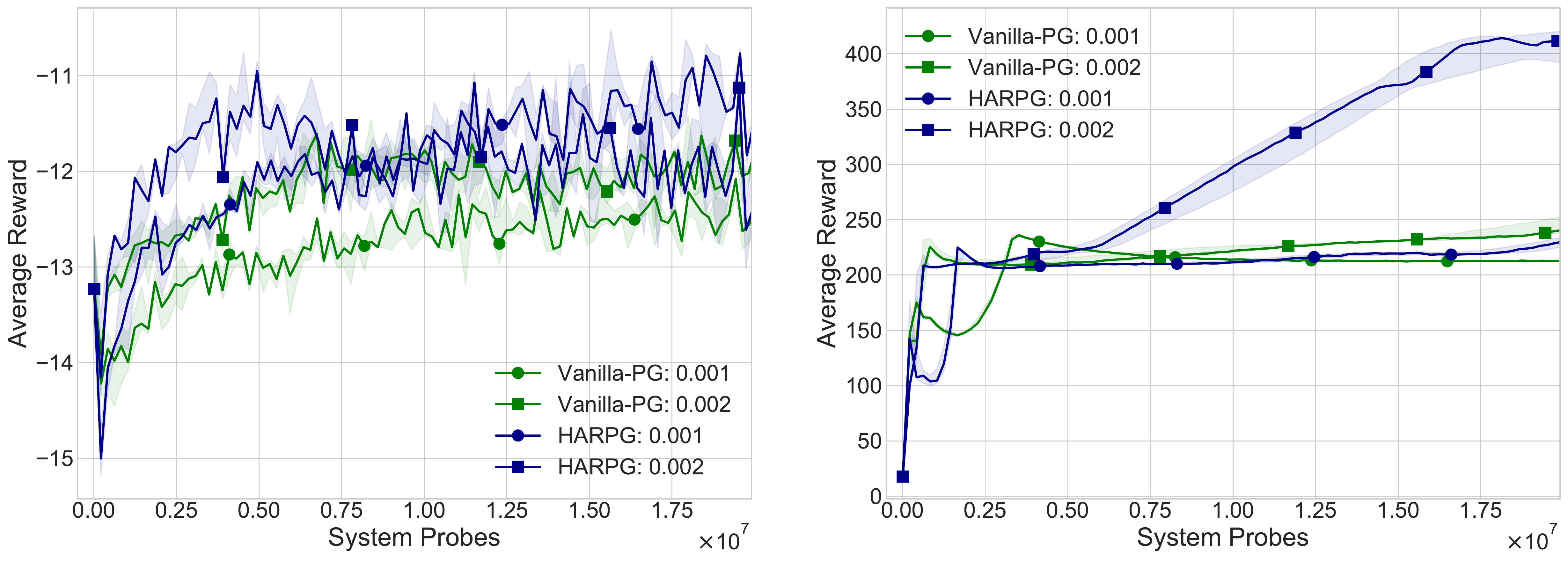}
        \caption{The performance of \algname{Vanilla-PG}, and \algname{HARPG} on \texttt{Reacher} (left) and \texttt{Hopper} (right) environments for non-tuned initial step-sizes.}
        \label{fig:small_step_size}
\end{figure}

\subsection{Experiment 4: Additional comparison to other methods with tuned step-sizes}
In this experiment, we test our algorithms on \texttt{Hopper} and \texttt{Halfcheetah} tasks with tuned step-sizes. We compare the performance of presented methods with two additional algorithms. \algname{VR-PG} method, also known as \algname{STORM-PG} \citep{yuan-et-al20} implements a variance reduced policy gradient method using the importance sampling (IS) weights instead of the Hessian correction. We also consider a normalized version of this method for a fair comparison with \algname{N-HARPG}. \footnote{We note that theoretical analysis of these methods require an additional assumption on boundedness of IS weights. } We observe in Figure~\ref{fig:sz_hopper_halfcheetah}, that while \algname{VR-PG} and \algname{N-VR-PG} perform similarly or better than  \algname{HARPG} and \algname{N-HARPG}, our method \algname{N-PG-IGT} still converges faster. 

\begin{figure}[h]
    \centering 
    \includegraphics[width=0.95\textwidth]{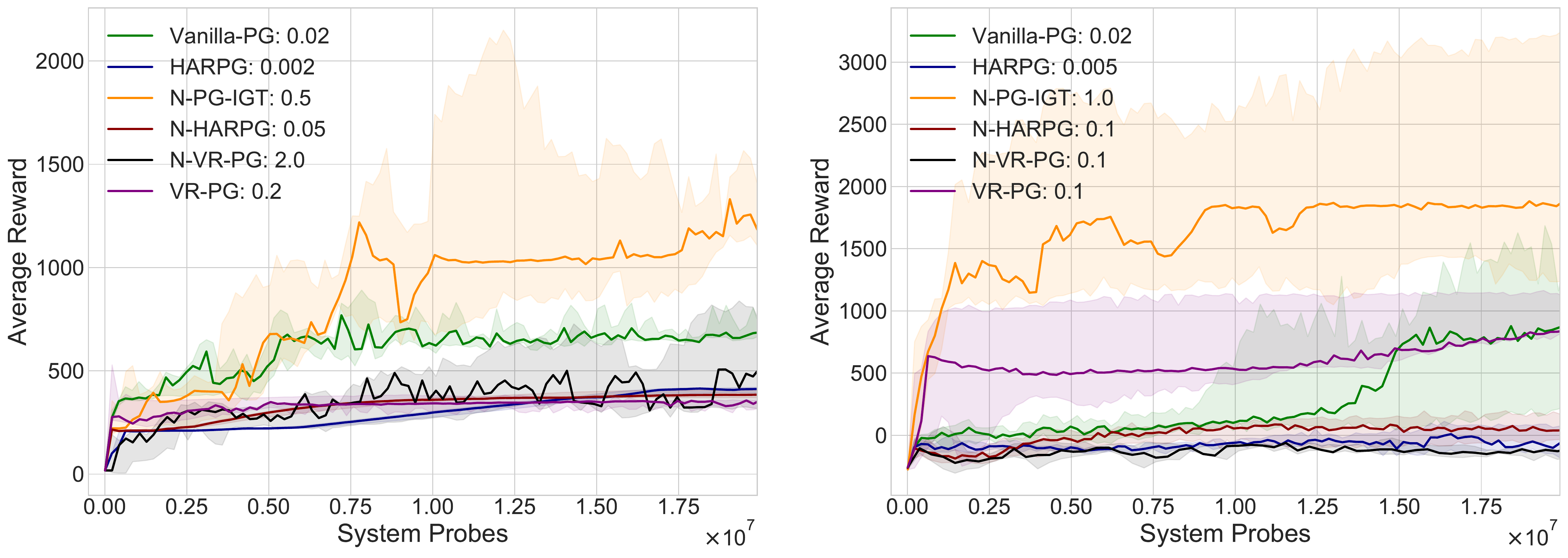}
        \caption{The performance of \algname{Vanilla-PG}, \algname{N-PG-IGT}, \algname{HARPG}, \algname{N-HARPG}, \algname{VR-PG} and \algname{N-VR-PG} on \texttt{Hopper} (left) and \texttt{Halfcheetah} (right) environments.}
        \label{fig:sz_hopper_halfcheetah}
\end{figure}

\subsection{Experiment 5: Performance under the soft-max parameterization for discrete state-action space}
Although our main theoretical results (except for FOSP of \algname{N-PG-IGT} in Theorem~\ref{thm:NIGT_FOSP}) only hold under the Fisher-non-degenerate (FND) parameterization Assumption~\ref{hyp:fisher-non-degenerate}, all methods presented in this work can be applied to other policy classes without any modifications. In this experiment, we test the proposed algorithms on the popular \texttt{Cartpole} environment with soft-max policy parameterization. We observe in Figure~\ref{fig:cartpole} that the behavior of methods is similar to other environments. 
\begin{figure}[h]
    \centering 
    \includegraphics[width=0.45\textwidth]{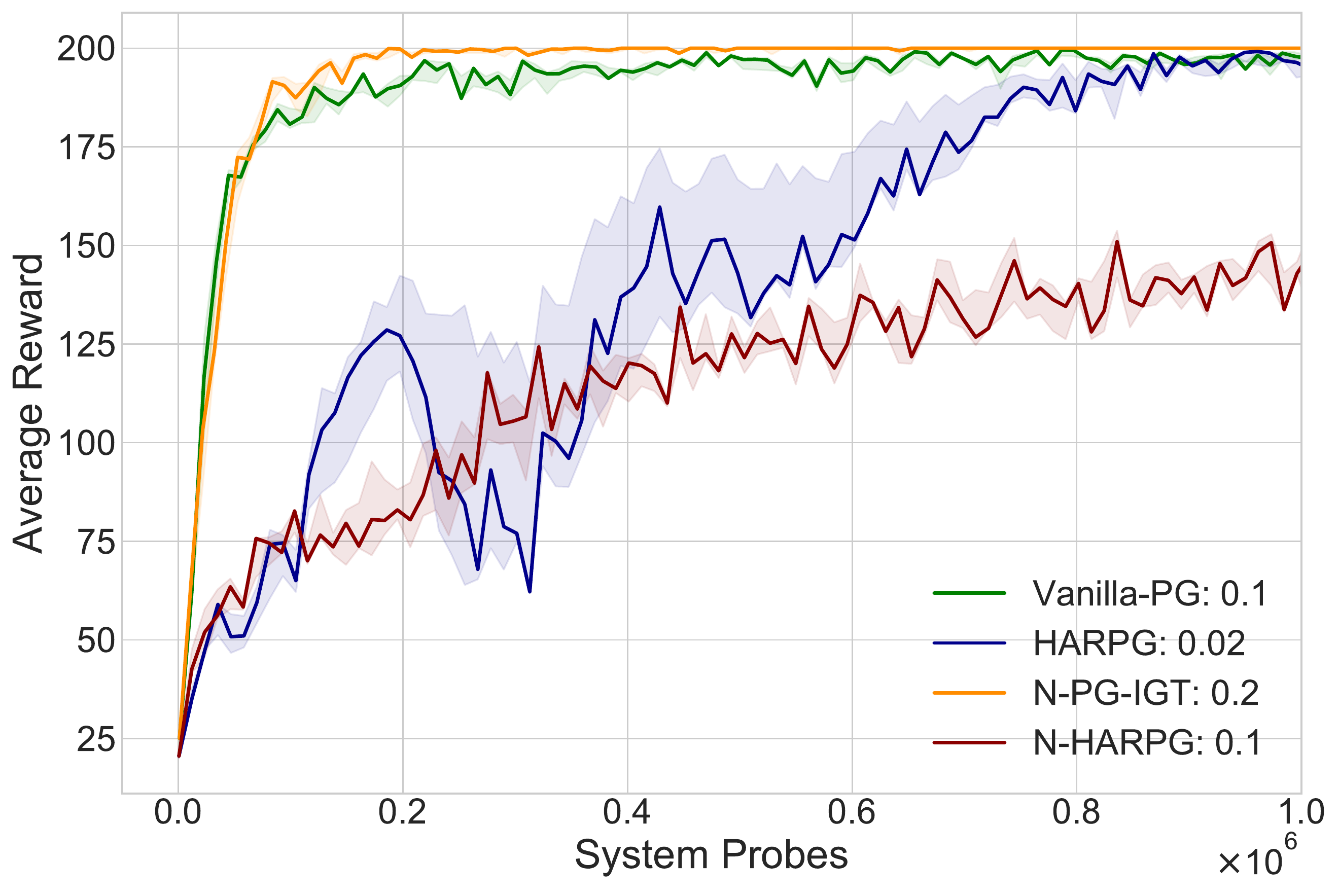}
        \caption{The performance of \algname{Vanilla-PG}, \algname{N-PG-IGT}, \algname{HARPG}, and \algname{N-HARPG} on \texttt{Cartpole} environment for tuned initial step-sizes.}
        \label{fig:cartpole}
\end{figure}

\newpage

 \section{Further Discussion of Assumptions from Section~\ref{subsec:assumptions_main}}
 \label{sec:discussion-assumptions}

In this section, we discuss our assumptions in the light of two simple examples while we also comment on the generality of our assumptions beyond those examples. However, we would like to highlight that our set of assumptions is general and standard in the recent studies of convergence of PG methods (see~\citet{xu-et-al20iclr,zhang-et-al20glob-conv-pg,liu-et-al20,ding-et-al22,gargiani-et-al22,Masiha_SCRN_KL,salehkaleybar-et-al22}, to name a few). We believe that relaxing some of the assumptions is an interesting question that we defer for future work. 

 \subsection{Gaussian policy} 
  Let~$\varphi: \mathcal{S} \to \R^d$ be a feature map and set\footnote{The more general case~$\mathcal{A} = \R^q (q >1)$ can be treated similarly and only needs slight adjustments in the definition of the Gaussian policy such as considering a matrix feature map (instead of a vector one) of proper dimensions. We briefly comment on this case below when discussing Assumption~\ref{hyp:fisher-non-degenerate}.}~$\mathcal{A} = \R$.
 We consider the common Gaussian policy parameterized by a linear mean parameterization and a constant variance parameter $\sigma^2 > 0$ defined for every~$\theta \in \R^d, (s,a) \in \mathcal{S} \times \mathcal{A}$ by:
 \begin{equation}
 \label{eq:gaussian-policy}
 \pi_\theta(a | s)=\frac{1}{\sigma \sqrt{2 \pi}} \exp \left(-\frac{ ( a - \varphi(s)^{\top}\theta )^2 }{2 \sigma^2}\right)\,.
 \end{equation}
 More general forms of Gaussian policies use neural networks for the mean and the standard deviation parameterizations. Note that one can consider even more general parameterizations such as the exponential family or symmetric $\alpha$-stable policies which include the Gaussian policy as a particular case. We refer the interested reader to the nice exposition in \cite{bedi-et-al21heavy-tailed-policy-search} for a discussion around such heavy-tailed policy parameterizations. 
In the case of~\eqref{eq:gaussian-policy}, the score function can be written for every~$\theta \in \R^d, (s,a) \in \mathcal{S} \times \mathcal{A}$ as follows: 
 $$
 \nabla \log \pi_{{\theta}}(s, a) = \frac{a-\varphi(s)^{\top} {\theta}}{\sigma^2}\varphi(s)\,.
$$

 \noindent\textbf{Assumption~\ref{hyp:parameterization-reg}-\eqref{hyp:param-reg-1st-order} (Bounded score function):} A few comments are in order regarding this assumption:

 \begin{enumerate}

     \item In the specific case of the Gaussian policy, we notice that strictly speaking this assumption is not satisfied for all $a \in \R$ and~$\theta\in\R^d$ even if the feature map is bounded as it was previously noted in the literature~\citep{Vanilla_PL_Yuan_21,bedi-et-al22contin-action-space}. However, a number of recent works 
     \citep[Section C]{papini-et-al18}, \citep[Section 7]{xu-et-al20uai}, \citep[Section 4]{liu-et-al20}, \citep[Corollary 4.8]{xu-et-al20iclr} assume boundedness of the score function. It is easy to see that such an assumption can be guaranteed if one additionally assumes the bound on sampled actions and on the mean parameterization $\varphi(s)^{\top}\theta$. In practice, this is usually enforced by clipping the actions selected by the policy~$\pi_{\theta}$ and projecting the update of the gradient method into the ball of fixed radius. 

     \item It was previously shown that this standard assumption can be relaxed to hold in expectation (see Assumption 4.1 \citep{Vanilla_PL_Yuan_21}) instead of for all state-action pairs. Such relaxed version allows to accomodate the Gaussian parameterization without additional assumptions, while it is sufficient to validate smoothness of $J$ and the bounded variance of stochastic gradients, see~\citep[Lemma 4.2 and 4.4]{Vanilla_PL_Yuan_21}. However, in this work, we do not focus on such technical relaxations and work with the standard assumption of bounded score function. 
     
     \item For a policy inducing an unbounded score function, an alternative is to partition the action space for instance according to when the score function is bounded and integrable using an exploration parameter. We refer the reader to Section~5.1 in \citet{bedi-et-al21heavy-tailed-policy-search}.
     Adopting this approach requires adjustments in the analysis. 
 \end{enumerate}

 \noindent\textbf{Assumption~\ref{hyp:parameterization-reg}-\eqref{hyp:param-reg-2nd-order}(Lipschitzness of the score function):} 
 We compute 
 $\nabla^2 \log \pi_{{\theta}}(s, a) = - \fr{1}{\sigma^2} \varphi(s)\varphi(s)^{\top}$ which is bounded if the feature map is bounded.

 \noindent\textbf{Assumption~\ref{hyp:lipschitz-hessian}(Second-order smoothness of the score function):} 
 This assumption trivially holds since $\nabla^2 \log \pi_{{\theta}}(s, a)$ is independent of $\theta$.
 Regarding this assumption, we further highlight that it is not required for \algname{(N)-HARPG} (see Theorems~\ref{thm:HARPG} and~\ref{thm:N-HARPG}). It is only required for \algname{N-PG-IGT}. While this assumption is key for the improvement for \algname{N-PG-IGT}, we point out that (a) it is a mild assumption satisfied  for the Gaussian policy as we just mentioned and (b) Table~\ref{table:related-works}  (column ‘No 2nd. smooth’) highlights it as an additional assumption compared to existing results for a fair comparison. In particular, we compare to~\cite{Masiha_SCRN_KL} which makes the same additional assumption.

 \noindent\textbf{Assumption~\ref{hyp:fisher-non-degenerate}(Fisher non-degeneracy):}
 The Fisher information matrix for a Gaussian policy can be written as $F_{\rho}(\theta) = \fr{1}{\sigma^2} \varphi(s)\varphi(s)^{\top}$, which is positive definite for a full row-rank feature map. 
 More generally\footnote{We report some comments from the discussion in Appendix~B.2 in \cite{liu-et-al20} for completeness and for the convenience of the reader.}, 
 \begin{enumerate}

    \item \textbf{Linear mean parameterization}: when the Gaussian policy is multivariate ($\mathcal{A} = \R^q$) and parametrized by~$\mu_{\theta}(s) = \Phi(s)^T \theta$ where~$\Phi: \mathcal{S} \to \R^{d \times q}$ ($d$ being the dimension of the parameter~$\theta$) and~$\Sigma$ is a fixed covariance matrix, the Fisher information matrix is independent of~$\theta$ as it is given by~$\Phi(s) \Sigma^{-1} \Phi(s)^T\,.$ Hence, Assumption~\ref{hyp:fisher-non-degenerate} is satisfied if~$\Phi(s)$ is full-row-rank. For the usual setting where the dimension of the parameter~$\theta$ is smaller than the dimension of the action space (i.e., $d < q$), choosing the rows of~$\Phi(s)$ to be linearly independent would satisfy this full-rank requirement. This is a common assumption in the RL literature when considering linear function approximation (see e.g., \citet{tsitsiklis-vanroy97,melo-et-al08}).  

    \item \textbf{Neural net mean parameterization}: when~$\mu_{\theta}(s)$ is nonlinear in~$\theta$ as in the celebrated case of neural networks, Assumption~\ref{hyp:fisher-non-degenerate} can be satisfied under adequate (uniform in~$\theta$) assumptions on the Jacobian of~$\mu_{\theta}(s)$ akin to the conditions for the linear mean parameterization. 

    \item Full-rank exponential family distributions \citep{dasgupta11exponential-family} with a parameterized mean~$\mu_{\theta}(s)$ do also satisfy Assumption~\ref{hyp:fisher-non-degenerate}. Let us mention that the Fisher information matrix is positive definite for many regular statistical models \citep{kullback97information-theory-stats}. 
 \end{enumerate}

 \subsection{Cauchy policy} 
 Let~$\varphi: \mathcal{S} \to \R^d$ be a feature map and set~$\mathcal{A} = \R$. We suppose this feature map to be bounded (i.e., $\norm{\phi(s)} \leq D$ for every~$s \in \mathcal{S}$ for some positive constant~$D$) and full-row-rank.  The Cauchy distribution with linear mean parameterization $\varphi(s)^{\top}\theta $ and the scale parameter $\sigma > 0$ is given by:
 $$
 \pi_{{\theta}}(a | s)=\frac{1}{\pi \sigma \left(1+\left( \frac{a-\varphi(s)^{\top} {\theta}}{\sigma}\right)^2\right)} \,, 
 $$
 for every~$ \theta \in \R^d, s \in \mathcal{S}, a \in \mathcal{A}\,.$
In this case the score function can be written as
 $$
 \nabla \log \pi_{{\theta}}( a| s )=\frac{2 x_{\theta} }{1+x_{\theta}^2} \, \frac{\varphi(s)}{\sigma} ,
 $$
 where $x_{\theta} = \left(a-\varphi(s)^{\top} {\theta}\right)/\sigma$.

Following \citep[Lemma 4.7]{bedi-et-al22contin-action-space}, we verify our assumptions under this policy parameterization. 

\noindent\textbf{Assumption~\ref{hyp:parameterization-reg}-\eqref{hyp:param-reg-1st-order} (Bounded score function):} Since $ \frac{2 x }{1+x^2}  \leq 1 $ for all $x\in\R$, we have $\left\|\nabla \log \pi_{{\theta}}( a | s )\right\| \leq M_g$ with $M_g = \frac{D}{\sigma}$ for all ${\theta}\in\R^d$ and $s\in\mathcal{S}$, $a \in \mathcal{A}$. 

 \noindent\textbf{Assumption~\ref{hyp:parameterization-reg}-\eqref{hyp:param-reg-2nd-order}(Lipschitzness of the score function):} 
 We compute 
 $
 \nabla^2 \log \pi_{{\theta}}(a | s) = \fr{2(1 - x_{\theta}^2)}{(1+x_{\theta}^2)^2} \left(\frac{\varphi(s) \varphi(s)^{\top}}{\sigma^2}\right).
 $
 Then since $\fr{2(1 - x^2)}{(1+x^2)^2} \leq \fr{3}{2}$ for all $x \in\R$, we can bound $\left\|\nabla^2 \log \pi_{{\theta}}( a | s)\right\| \leq M_h$ with $M_h = \frac{3 D^2}{ 2 \sigma^2} $.

 \noindent\textbf{Assumption~\ref{hyp:lipschitz-hessian}(Second-order smoothness of the score function):} We take an arbitrary non-zero vector $u\in \R^d$, and compute $\nabla ( \nabla^2 \log \pi_{{\theta}}(a | s) \, u ) = \fr{4 x_{\theta} (x_{\theta}^2 - 3 )}{(1 + x_{\theta}^2)^3} \fr{1}{\sigma^3} \varphi(s) \varphi(s)^{\top} \varphi(s)^{\top} u $ . 
 Therefore,
 $\|\nabla^2 \log \pi_{\theta} (a|s) - \nabla^2 \log \pi_{\theta'} (a|s)\| \leq l_2\, \|\theta - \theta'\|\,$ with $l_2 = \fr{4 D^3}{\sigma^3}$.

 \noindent\textbf{Assumption~\ref{hyp:fisher-non-degenerate}(Fisher non-degeneracy):}
 Using reparameterization theorem for the Fisher information of Cauchy distribution, the Fisher information matrix for $\pi_{{\theta}}( a | s)$ can be written as $F_{\rho}(\theta) = \fr{1}{2\sigma^2} \varphi(s)\varphi(s)^{\top}$, which is positive definite for full row-rank feature map.

\noindent\textbf{About Assumption~\ref{hyp:tranf-compatib-fun-approx} (Bounded compatible function approximation transfer error):} As previously mentioned, this assumption has been used in several works~\citep{agarwal-et-al21,liu-et-al20, ding-et-al22,Vanilla_PL_Yuan_21}. 

\newpage
\section{Proof Sketch of the Main Results}\label{sec:appendix_proof_sketch}

In this section, we provide a brief proof sketch of our results to highlight the proof technique leading to improved sample complexity compared to prior work. The sample complexity improvement comes from a combined effect of (a) exploiting the relaxed weak gradient dominance and (b) deriving a better variance error control. The complete formal proofs are presented in the subsequent sections.

\noindent\textbf{Step 1. Ascent-like lemma.} 
Define~$\hat e_t \eqdef d_t - \nabla J_H(\theta_t)$. First, in Lemma~\ref{le:NSGD_descent_trunc} we prove the following inequality for any sequence~$\theta_t$ updated via~$\theta_{t+1} = \theta_t + \gamma_t d_t/\|d_t\|$ for any nonzero update direction~$d_t$ (to be given by~\algname{N-PG-IGT} or \algname{N-HARPG}): 
\begin{equation}\label{eq:ascent-proof-sketch}
      J(\theta_{t+1}) - J^* \geq  J(\theta_t) - J^* + \frac{\gamma_t}{3} \|\nabla J(\theta_t)\| - \frac{8 \gamma_t}{3} \|\hat e_t\| - \mathcal{O}(\gamma_t^2 + \gamma^H\gamma_t)\,.
\end{equation}
A similar lemma holds for the variant~\algname{HARPG} without normalization, see Lemma~\ref{le:SGD_descent}.

\noindent\textbf{Step 2. Relaxed weak gradient domination inequality.} 
Combining Assumptions~\ref{hyp:fisher-non-degenerate} and~\ref{hyp:tranf-compatib-fun-approx}, 
and following the derivations of \citet{ding-et-al22}, we obtain
the relaxed weak gradient dominance inequality, satisfied by the expected return function~$J(\theta)$ (Lemma~\ref{lem:relaxed-w-grad-dom}):
\begin{eqnarray}
\label{eq:relaxed-w-grad-dom}
\forall\, \theta \in \R^d, \quad 
\varepsilon' + \|\nabla J(\theta)\| \geq \sqrt{2\mu}\, (J^* - J(\theta))\,,
\end{eqnarray}
where~$\varepsilon^{\prime} = \frac{\mu_F \sqrt{\varepsilon_{\text{bias}}}}{M_g (1-\gamma)}\,$ and~$\mu = \frac{\mu_F^2}{2M_g^2}.$ 
This structural property of the objective function is one of the key tools of our analysis. 
Note that we also have~$\varepsilon^{\prime} = \frac{\mu_F \sqrt{\varepsilon_{\text{bias}}}}{M_g (1-\gamma)}$ and~$\mu = \frac{\mu_F^2}{2M_g^2}\,.$ In particular, 
the error~$\varepsilon^{\prime}$ is intimately related to the expressivity of the policy parameterization quantified by~$\varepsilon_{\text{bias}}$ as discussed in Assumption~\ref{hyp:tranf-compatib-fun-approx}. Notice that compared to~\citet{ding-et-al22}, we do not derive a bound for the expected gradient norm and inject it in~(\ref{eq:relaxed-w-grad-dom}) to obtain our final global convergence rates. Instead, 
we carefully incorporate~(\ref{eq:relaxed-w-grad-dom}) into the ascent-like lemma derived for the (normalized) gradient method to obtain a recursion on~$\delta_t \eqdef  J^* - \Exp{ J(\theta_t) }$, which leads to improved convergence rates.  

\begin{remark}[About gradient dominance]
Notice that our relaxed weak gradient dominance condition features the norm of the gradient instead of the norm squared. Such a structural property for the standard RL setting we consider here has been unveiled and exploited in several recent works \citep{agarwal-et-al21,xiao22}. Note that this condition is different from the Polyak-\L{}ojasiewicz (P\L{}) like condition which features a gradient norm squared in~\eqref{eq:relaxed-w-grad-dom} (with~$\epsilon' = 0$). A similar P\L{} condition was shown for the state feedback version of LQR problem in continuous control \citep{fazel-et-al18,Optimizing_LQR_21} which differs from our current setting. 
\end{remark}

\noindent\textbf{Step 3. Variance reduction control.} This crucial step depends on the algorithm used and consists in carefully deriving a recursion on $\Exp{\|e_t\|}$ or~$\Exp{\|e_t\|^2}\,.$ For instance, in Lemma~\ref{le:NIGT_first_err_bound} for \algname{N-PG-IGT} method, we derive the following recursion in vector form
\begin{equation}\label{eq:hat-e-t-proof-sketch}
\hat{e}_t = (1 - \eta_t) \hat{e}_{t-1}  + \eta_t e_{t} + (1 - \eta_t) S_t +  \eta_t Z_t ,
\end{equation}
where $\Exp{ \norm{e_{t}} } \leq \sigma_g$, and the terms $S_t$ and $Z_t$ correspond to the second-order Taylor approximation of $J$ at different points. This allows us to utilize Assumption~\ref{hyp:lipschitz-hessian} together with the normalized update rule (which guarantees~$\norm{\theta_{t+1} - \theta_t} = \gamma_t $), and derive the control for $ \norm{S_{t}}  \leq \cO\rb{  \g_{t}^2  + \g^H } $, $\norm{Z_t} \leq \cO\rb{ \g_t^2 \eta_t^{-2} + \g^H }$. 

By choosing $\gamma_t = \cO\rb{ t^{-1} } ,\,\eta_t = \cO\rb{ t^{-{4}/{5} } }$, $H = (1-\gamma)^{-1}\log(T+1)$ we guarantee 
$$ \Exp{ \norm{\hat e_t } } \leq    \cO\rb{ \eta_t^{\nfr{1}{2}}   +  \g_t^2 \eta_t^{-2}  +  \g^H \eta_t^{-1} +   \g^{H}  \g_t \eta_t^{-1} } = \cO\rb{ t^{-2/5} }\,.
$$
We emphasize the importance of using time varying step-sizes and momentum parameters for our analysis as opposed to the constant step-size depending on the total number of iterations $T$. In particular, selecting time varying $\eta_t$ allows us to resolve \eqref{eq:hat-e-t-proof-sketch} for all $t \geq 1$ and plug it in \eqref{eq:ascent-proof-sketch}. 

\noindent\textbf{Step 4. Combine the variance-reduced estimate and the ascent-like lemma.} 
Finally, we incorporate the results of Steps~2-3 into the ascent-like lemma derived in Step 1 to obtain a recursion on the sequence~$(\delta_t)$. For example, for the \algname{N-PG-IGT} method, we arrive at
$$
\delta_{t+1} \leq ( 1 - \Omega( \g_t ) ) \delta_t + \cO( \g_t t^{-2/5} + \g_t^2 + \varepsilon' \g_t ).
$$
It remains to resolve the resulting recursion on the sequence~$(\delta_t)$ using technical lemma from Section~\ref{sec:tech_lemmas} in order to obtain the final rate.  

\newpage

\section{Further Related Work}
\label{sec:related-work}

In this section, we provide a brief additional discussion  
including first-order stationarity of variance-reduced PG methods, global optimality of exact PG methods and some additional related works complementing the main discussion in Section~\ref{sec:introduction} of the paper.


\subsection{First-order stationarity for variance-reduced PG methods}
In the last few years, several variance-reduced PG methods were proposed in the literature to improve sample efficiency over Vanilla-PG which achieves a~$\tilde{\mathcal{O}}(\varepsilon^{-4})$ sample complexity.
\citet{papini-et-al18} proposed a stochastic variance reduced policy gradient (\algname{SVRPG}) method achieving a~$\tilde{\mathcal{O}}(\varepsilon^{-4})$ sample complexity. \citet{xu-et-al20uai} revisited this algorithm and established  a~$\tilde{\mathcal{O}}(\varepsilon^{-10/3})$ sample complexity. This result was further improved to~$\tilde{\mathcal{O}}(\varepsilon^{-3})$ by using second-order information in~\citet{shen-et-al19} or variants of variance reduction methods such as \algname{SARAH}~\citep{nguyen-et-al17}, \algname{SPIDER}~\citep{fang-et-al18} or \algname{PAGE}~\citep{li-et-al21page} in \citet{xu-et-al20iclr,pham-et-al20,gargiani-et-al22}. Yuan et al. and Huang et al.~\cite{yuan-et-al20,huang-et-al20} proposed momentum-based policy gradient methods based on \algname{STORM} \citep{cutkosky-orabona19}. In particular, \citet{huang-et-al20} removed the need of large batches of trajectories for variance reduction. All the aforementioned works in this section proposing variance-reduced PG methods (except \citet{shen-et-al19} which uses second-order information) use importance sampling to account for the non-stationarity of the sampling distribution. As a consequence of the use of the importance sampling mechanism to address the non-stationarity issue, most of previous works require an unverifiable assumption stating that the variance of the importance sampling weights is bounded. To address this issue, \citet{zhang-et-al21} provide a gradient truncation mechanism complementing importance sampling for the softmax parameterization whereas \citet{shen-et-al19} and \citet{salehkaleybar-et-al22} incorporate second-order information. Both approaches lead to the same~$\tilde{\mathcal{O}}(\varepsilon^{-3})$ FOSP sample complexity.

\subsection{Global optimality of exact PG methods}
More recently, a line of research focused on establishing global optimality convergence rates i.e. deriving the sample complexity to reach an $\varepsilon$-approximate global optimum of the expected return function \citep{agarwal-et-al21,fazel-et-al18,bhandari-russo19,mei-et-al20,zhang-et-al20glob-conv-pg,Zhang_Kim_O’Donoghue_Boyd21} instead of the mere $\varepsilon$-approximate first order stationary point. This line of research leverages additional structure of the policy optimization problem under the form of gradient dominance \citep{agarwal-et-al21,fazel-et-al18}, \L{}ojasiewicz-like inequalities \citep{mei-et-al20,mei-et-al21} or hidden convexity \citep{zhang-et-al21}. These structural properties guarantee that the non-concave policy optimization objective has no suboptimal stationary points. Most of these works assume the access to exact policy gradients to show fast convergence rates \citep{mei-et-al20,mei-et-al21,xiao22}. Other works also establish global optimality guarantees for the Natural PG method supposing access to exact gradients \citep{khodadadian-et-al22,bhandari-russo21}. More precisely, \citet{khodadadian-et-al22} and~\citet{bhandari-russo21} show a linear convergence rate for the Natural PG method with adaptive step sizes by drawing a connection between policy gradient and policy iteration. These results are further sharpened in \citet{lan22,xiao22} in which it is shown that a general class of policy mirror descent methods, including the natural policy gradient method and a projected Q-descent method, enjoy a linear rate of convergence with geometrically increasing step sizes in the tabular setting, improving over the previously known sublinear rate of simpler PG methods \citep{agarwal-et-al21,mei-et-al20}. \citet{zhang-et-al20variational} consider the more general problem which consists in maximizing a general concave utility function of the state-action occupancy measure. Leveraging the hidden convexity structure of the problem, they propose a globally convergent algorithm derived from a novel variational policy gradient theorem for RL with general utilities. 

\subsection{Further discussion about stochastic policy optimization} 
In the case of stochastic on-policy gradients, the acceleration brought by exploiting geometric information in the true gradient setting is not obviously exploitable as this was accurately reported in \citet{mei-et-al21stoch} for the tabular softmax parameterization. In particular, \citet{mei-et-al21stoch} exhibit a trade-off between exploiting geometry to accelerate convergence and overcoming the noise due to stochastic gradients (which is possibly infinite). 
We also mention that several works (see for e.g., \citet{chen-khodadadian-et-al22,chen-maguluri22,khodadadian-et-al22ac,hong-et-al20,qiu-et-al21,wu-et-al20,xu-et-al20ac}) proposed finite-time analysis for actor-critic methods to solve the policy optimization problem using (linear) function approximation for the Q-function or advantage function under ergodicity assumptions on the underlying Markov chain. Our work is focused on Monte-Carlo based PG methods (which are \algname{REINFORCE}-like) for which no critic is needed (no function approximation error relative to this critic process is incurred) and no ergodicity assumption is needed. We rather rely on the compatible function approximation framework. 
Besides this line of works, \citet{Zhang_Kim_O’Donoghue_Boyd21} provide global convergence rates for the \algname{REINFORCE} algorithm for episodic RL with a fixed mini-batch instead of arbitrarily large batch sizes. These results take the form of an anytime sublinear high probability regret bound as well as an almost sure global convergence of the average regret with an asymptotically sublinear rate. In particular, for their regret analysis, they consider the softmax parameterization with a vanishing log barrier regularization and the result translates to a~$\tilde{\mathcal{O}}(\varepsilon^{-6})$ global sample complexity as also reported in~\cite{Vanilla_PL_Yuan_21}.
Recently, several works considered heavy-tailed policy parameterizations offering better exploration in continuous control problems \citep{bedi-et-al21heavy-tailed-policy-search,bedi-et-al22contin-action-space}.

\subsection{Comparison to prior work in stochastic optimization}
In the stochastic optimization setting, several recent works consider optimizing functions satisfying a weak gradient domination inequality (see Lemma~\ref{lem:relaxed-w-grad-dom}~\eqref{eq:relaxed-w-grad-dom} with $\varepsilon' = 0$). In particular, \citep{Fontaine_SGD_dynamics,KL_PAGER_Fatkhullin,Scaman_SGD_KL_2022} obtain a~$\cO(\varepsilon^{-3})$ sample complexity for finding a global optimum in expectation for Stochastic Gradient Descent\footnote{These works consider more general assumptions, such as $\alpha$-P{\L}, i.e., there exists $\alpha \in [1,2]$ such that for all $\theta \in \R^d$ it holds $\norm{\nabla J(\theta)}^{\alpha} \geq (2\mu)^{\nfr{\alpha}{2}} (J^* - J(\theta))$, and derive the sample complexities depending on the value of $\alpha$. It turns out that when $\alpha = 1$ as in our setting, the sample complexity is the worst (in $\varepsilon$) compared to other $\alpha > 1$.}, which is the same complexity as for (stochastic) \algname{Vanilla-PG} algorithm in the RL setting~\citep{Vanilla_PL_Yuan_21}. \citet{KL_PAGER_Fatkhullin} further improves this complexity to $\cO(\varepsilon^{-2})$ under additional average smoothness assumption by proposing a variance reduced algorithm named \algname{PAGER}. Their algorithm has a double loop structure and adopts a probabilistic switch between two different batch-size sequences (increasing over iterations), which makes it more difficult to implement and analyze. Another work by \citet{Masiha_SCRN_KL} considers stochastic second-order methods under similar assumptions. Their Stochastic Cubic Regularized Newton's algorithm (\algname{SCRN}) achieves a~$\tilde{\mathcal{O}}(\varepsilon^{-2.5})$ sample complexity under a similar set of assumptions. However, \algname{SCRN} uses a subroutine to approximately solve a cubic regularized subproblem (which makes it computationally inefficient), it uses samples of stochastic Hessian and requires large batches for both stochastic gradient and Hessian at every iteration. To the best of our knowledge \citep{KL_PAGER_Fatkhullin} and \citep{Masiha_SCRN_KL} are the only works, which provide the sample complexity beyond $\cO(\varepsilon^{-3})$ for stochastic optimization for functions satisfying weak gradient dominance condition. However, these optimization works consider complicated algorithms involving parameter restart, cubic regularization subproblem and/or large mini-batch sizes at each iteration. In contrast, our proposed algorithms improve over $\cO(\varepsilon^{-3})$ sample complexity with simple algorithms (\algname{N-PG-IGT}, \algname{HARPG}, \algname{N-HARPG}), which are computationally efficient and batch-free (only require to sample one trajectory per iteration). In addition, our analysis is much simpler than the analysis of algorithms in \citep{KL_PAGER_Fatkhullin} and \citep{Masiha_SCRN_KL}. Our analysis is also of independent interest in the study of stochastic optimization under the weak gradient domination assumption. 

\newpage

\section{Notation and Useful Lemma}\label{sec:notation_useful_lemma}

\noindent\textbf{Notation.} 
For any integer~$p$, the euclidean space~$\R^p$ is equipped with its usual inner product~$\ps{\cdot, \cdot}$ and its corresponding 2-norm~$\|\cdot\|\,.$ The transpose of a vector~$x \in \R^p$ for any integer~$p$ is denoted by~$x^{\top}\,$ and we denote by~$I_p$ the identity matrix of dimension~$p$. For two matrices~$A, B \in \R^{p \times q}$ for two integers~$p, q \geq 1$, the notation~$A \succcurlyeq B$ indicates that the matrix~$A- B$ is positive semi-definite. 
For two sequences of nonnegative reals $\rb{a_n}, \rb{b_n}$, we use the notation~$a_n = \cO(b_n)$ if there exists~$c>0$ such that $a_n \leq c \,b_n$ for every integer~$n\,.$ Hence, unless otherwise specified, we use the $\mathcal{O}(\cdot)$ to hide only absolute constants which do not depend on any problem parameter and we further use the notation~$\tilde{\mathcal{O}}(\cdot)$ to hide only absolute constants and logarithmic factors.  

\noindent\textbf{Advantage function.} 
For completeness, we define the advantage function that appeared in the transferred compatible function approximation error in the statement of Assumption~\ref{hyp:tranf-compatib-fun-approx} in the main part of this paper. For this purpose, we first define for every policy~$\pi$ the state-action value function~$Q^{\pi}: \mathcal{S} \times \mathcal{A} \to \R$ for every~$s \in \mathcal{S}, a \in \mathcal{A}$ as: 
\begin{equation*}
Q^{\pi}(s,a) \eqdef \mathbb E_{\pi}\left[\sum_{t=0}^{\infty} \gamma^t r(s_t,a_t) | s_0 = a, a_0 = a \right]\,.
\end{equation*}
Under the same policy~$\pi$, the state-value function~$V^{\pi}: \mathcal{S} \to \R$ and the advantage function~$A^{\pi}:  \mathcal{S} \times \mathcal{A} \to \R$ are defined for every~$s \in \mathcal{S}, a \in \mathcal{A}$ as follows:
\begin{align*}
V^{\pi}(s) &\eqdef \mathbb E_{a \sim \pi(\cdot|s)}[Q^{\pi}(s,a)]\,, \\ 
A^{\pi}(s, a) &\eqdef Q^{\pi}(s, a) - V^{\pi}(s)\,.
\end{align*}

We now state two well-known useful lemma for our analysis. 
The first lemma shows that the Hessian of the expected return function is Lipschitz under a second-order regularity condition on the policy parameterization. 
\begin{lemma}[(A.68) in \citep{zhang-et-al20glob-conv-pg}]
Let Assumptions~\ref{hyp:parameterization-reg} and~\ref{hyp:lipschitz-hessian} hold, then we have
$$
\left\|\nabla^2 J(\theta)-\nabla^2 J\left(\theta^{\prime}\right)\right\| \leq L_h \left\|\theta-\theta^{\prime}\right\|_2,
$$
where
$
L_h=\frac{r_{\max } M_g M_h}{(1-\gamma)^2}+\frac{r_{\max } M_g^3(1+\gamma)}{(1-\gamma)^3}+\frac{r_{\max } M_g}{1-\gamma} \max \left\{M_h, \frac{\gamma M_g^2}{1-\gamma}, \frac{l_2}{M_g}, \frac{M_h \gamma}{1-\gamma}, \frac{M_g(1+\gamma)+M_h \gamma(1-\gamma)}{1-\gamma^2}\right\}\,.
$
\end{lemma}

The second lemma controls the truncation error due to truncating simulated trajectories to the horizon~$H$ in our infinite horizon setting. Notably, this error vanishes geometrically fast with the horizon~$H$\,.
\begin{lemma}[Lemma~2 in \citep{Masiha_SCRN_KL}]
\label{le:trunc_grad_hess}
    Let Assumption~\ref{hyp:parameterization-reg} be satisfied, then for all $\theta \in \R^d$ and every $ H \geq 1$, we have
    $$
    \norm{ \nabla J_{H}(\theta) - \nabla J(\theta) } \leq D_g \g^{H},  \qquad \norm{ \nabla^2 J_{H}(\theta) - \nabla^2 J(\theta) } \leq D_h \g^{H} ,
    $$
    where $D_g \eqdef \frac{M_g r_{\max }}{1-\gamma} \sqrt{\frac{1}{1-\gamma}+H}$, $ D_h \eqdef \frac{r_{\max }\left(M_h + M_g^2\right)}{1-\gamma}\left(H + \frac{1}{1-\gamma}\right)$.
\end{lemma}


\newpage
\section{Proof of Theorem~\ref{thm:N-PG-IGT} (\algname{N-PG-IGT})}
\label{sec:nigt_app}

\subsection{Global convergence}

We first state a more detailed version of Theorem~\ref{thm:N-PG-IGT}.

\begin{theorem}\label{thm:NIGT}

   Let Assumptions~\ref{hyp:parameterization-reg}, \ref{hyp:lipschitz-hessian}, \ref{hyp:fisher-non-degenerate} and \ref{hyp:tranf-compatib-fun-approx} hold. 
    Set~$\gamma_t = \fr{6  }{\sqrt{2\mu}( t+2 ) } = \fr{6 M_g}{\mu_F (t+2)} $, $\eta_t = \rb{ \fr{2}{t+2} }^{\nfr{4}{5}}$, $H = \fr{9}{5} \rb{1-\g}^{-1}{\log(T + 1)}$ . Then for every integer~$T \geq 1$, the output~$\theta_T$ of \algname{N-PG-IGT} (Algorithm~\ref{alg:N-PG-IGT}) satisfies 
    \begin{eqnarray}
        J^* -  \Exp{J(\theta_T)} \leq \cO\rb{ \fr{ J^* -  J(\theta_0) }{(T+1)^2} + \fr{ D_h }{\mu (T+1)^2 } + \fr{ L_g + \sqrt{\mu} D_g }{\mu (T+1) } +  \fr{ \sigma_g \mu + L_h  }{\mu^{\nfr{3}{2}} (T+1)^{\nfr{2}{5}}} + \fr{ \varepsilon^{\prime}}{\sqrt{\mu}} }.
    \end{eqnarray}
    where~$\sigma_g, L_g, L_h$ are defined in Proposition~\ref{prop:stoch-grad-and-J} and Lemma~\ref{lem:2nd-order-smoothness} respectively, $D_g$ is defined in Lemma~\ref{le:trunc_grad_hess}, $\varepsilon^{\prime}$ and $\mu$ are defined in Lemma~\ref{lem:relaxed-w-grad-dom}. Here $\cO(\cdot)$ only hides absolute numerical constants.
\end{theorem}

We start by establishing an ascent-like lemma on the function~$J$ (we are maximizing~$J$). The following result holds for a general normalized update rule for any sequence of update directions~$(d_t)$. 

\begin{lemma}\label{le:NSGD_descent_trunc}
Let Assumption~\ref{hyp:parameterization-reg} hold. Let~$(d_t)$ be any sequence of vectors in~$\R^d$ and consider the sequence~$(\theta_t)$ defined by:
$$
\theta_{t+1} = \theta_t + \g_t \fr{d_t}{\norm{d_t}}\,,
$$
where~$\theta_0 \in \R^d$ (and~$\theta_{t+1} = \theta_t$ if~$d_t = 0$)\,.
Then we have for every integer~$t \geq 1$, for any positive step-size~$\g_t$, 
\begin{equation}
\label{lem:NSGD-descent-trunc-ineq}
 -J(\theta_{t+1}) \leq  - J(\theta_{t}) - \fr{ \g_t }{3 } \norm{ \nabla J(\theta_{t})} + \fr{ 8 \g_t }{3 } \norm{d_t - \nabla J_H(\theta_t)} + \fr{L_g \g_t^2 }{2} + \frac{4 \g_t}{3} \norm{ \nabla J_{H}(\theta_{t}) - \nabla J(\theta_{t})}\,.
\end{equation}
\end{lemma}

\begin{proof}
By smoothness of the expected return function~$J$ (Proposition~\ref{prop:stoch-grad-and-J}) and using the update rule for $\theta_t$, we get
\begin{eqnarray}
    -J(\theta_{t+1}) &\leq& - J(\theta_{t}) - \langle \nabla J(\theta_{t}), \theta_{t+1} - \theta_t \rangle + \fr{L_g}{2} \sqnorm{ \theta_{t+1} - \theta_t } \notag \\
    &=& - J(\theta_{t}) - \g_t \fr{\langle  \nabla J(\theta_{t}) , d_t \rangle}{\norm{d_t}} + \fr{L_g \g_t^2}{2} \notag \\
    &\leq& - J(\theta_{t}) - \g_t \fr{\langle  \nabla J_{H}(\theta_{t}) , d_t \rangle}{\norm{d_t}} + \fr{L_g \g_t^2}{2} + \g_t \norm{ \nabla J_{H}(\theta_{t}) - \nabla J(\theta_{t})}. \notag
\end{eqnarray}
Now let us bound the second term in the above inequality. Define $\hat{e}_t = d_t -  \nabla J_H(\theta_{t})$. We consider two cases. First, if $\norm{\hat{e}_t} \leq \fr{1}{2} \norm{  \nabla J_H(\theta_{t})  }$, then 
\begin{eqnarray}
    - \fr{\langle  \nabla J_{H}(\theta_{t}), d_t \rangle}{\norm{d_t}} &=& \fr{ -\sqnorm{ \nabla J_{H}(\theta_{t})} - \langle  \nabla J_{H}(\theta_{t}), \hat{e}_t \rangle}{\norm{d_t}} \notag \\
    &\leq& \fr{ -\sqnorm{ \nabla J_{H}(\theta_{t})} + \norm{  \nabla J_{H}(\theta_{t})} \norm{ \hat{e}_t }}{\norm{d_t}} \notag \\
    &\leq& \fr{ -\sqnorm{ \nabla J_{H}(\theta_{t})} + \fr{1}{2}\sqnorm{  \nabla J_{H}(\theta_{t}) } }{\norm{d_t}} \notag \\
    &\leq& -\fr{ \sqnorm{ \nabla J_{H}(\theta_{t})} }{2 \rb{ \norm{ \nabla J_{H}(\theta_{t}) } + \norm{\hat{e}_t} }  } \notag \\
    &\leq& -\fr{ 1 }{3 } \norm{ \nabla J_{H}(\theta_{t}) }  \notag .
\end{eqnarray}
Otherwise, if $\norm{\hat{e}_t} \geq \fr{1}{2} \norm{  \nabla J_H(\theta_{t}) }$, we have
\begin{eqnarray}
    - \fr{\langle  \nabla J_{H}(\theta_{t}), d_t \rangle}{\norm{d_t}} &\leq & \norm{  \nabla J_{H}(\theta_{t}) } \notag \\
    & = & -\fr{ 1 }{3 } \norm{ \nabla J_{H}(\theta_{t})} + \fr{ 4 }{3 } \norm{ \nabla J_{H}(\theta_{t})} \notag \\
    & \leq & -\fr{ 1 }{3 } \norm{ \nabla J_{H}(\theta_{t})} + \fr{ 8 }{3 } \norm{\hat{e}_t} \notag ,
\end{eqnarray}
Combining the two cases concludes the proof of the lemma:
\begin{eqnarray}
    -J(\theta_{t+1}) &\leq&  - J(\theta_{t}) - \fr{ \g_t }{3 } \norm{ \nabla J_{H}(\theta_{t})} + \fr{ 8 \g_t }{3 } \norm{\hat{e}_t} + \fr{L_g \g_t^2 }{2} + \g_t \norm{ \nabla J_{H}(\theta_{t}) - \nabla J(\theta_{t})} \notag \\
    &\leq&  - J(\theta_{t}) - \fr{ \g_t }{3 } \norm{ \nabla J(\theta_{t})} + \fr{ 8 \g_t }{3 } \norm{\hat{e}_t} + \fr{L_g \g_t^2 }{2} + \fr{4 \g_t }{3 } \norm{ \nabla J_{H}(\theta_{t}) - \nabla J(\theta_{t})}\label{eq:NSGD_descent_FOSP}\,.
\end{eqnarray}
\end{proof}
We introduce a convenient shorthand notation for the rest of the paper: 
\begin{align}
\delta_t &\eqdef J^{\star} - \Exp{ J(\theta_t)}\,,\label{eq:delta_t}
\end{align}
where~$J^{\star}$ is the optimal expected return and $(d_t), (\theta_t)$ are the sequences computed by the algorithm (\algname{N-PG-IGT} later in this section, and \algname{(N)-HARPG} or others in the following sections). 

In the next lemma, we incorporate the relaxed weak gradient dominance inequality into the previous ascent-like lemma to bound the gradient norm. The resulting inequality is a recursive bound in the expected return gap which is important to derive our convergence rate. 
\begin{lemma}
\label{lem:descent-with-grad-dom}
Let Assumptions~\ref{hyp:parameterization-reg}, \ref{hyp:fisher-non-degenerate} and \ref{hyp:tranf-compatib-fun-approx} be satisfied. In the setting of Lemma~\ref{le:NSGD_descent_trunc}, we have for every integer~$t \geq 1$, for any positive step-size~$\g_t$, 
\begin{eqnarray}\label{eq:NSGD_descent2}
    \delta_{t+1} - \delta_t \leq  - \fr{\sqrt{2 \mu} \g_t}{3}  \delta_t + \fr{8 \g_t }{3} \mathbb E[\norm{d_t - \nabla J_H(\theta_t)}] + \fr{L_g \g_t^2}{2} + \fr{\varepsilon^{\prime}\g_t}{3} + \frac{4}{3} \g_t D_g \g^{{H}}\,,
\end{eqnarray}
where we recall that $(\delta_t)$ is defined in~\eqref{eq:delta_t}.
\end{lemma}

\begin{proof}
Under Assumptions~\ref{hyp:fisher-non-degenerate} and~\ref{hyp:tranf-compatib-fun-approx}, we can use Lemma~\ref{lem:relaxed-w-grad-dom} and combine it with Lemma~\ref{le:NSGD_descent_trunc} to bound the gradient norm and obtain 
\begin{eqnarray}
 -J(\theta_{t+1})
    &\leq&  - J(\theta_{t}) - \fr{\sqrt{2 \mu} \g_t}{3}  \rb{ J^{\star} - J(\theta_{t})  } + \fr{ 8 \g_t }{3 } \norm{d_t - \nabla J_H(\theta_t)} + \fr{L_g\g_t^2}{2} + \fr{\varepsilon^{\prime} \g_t}{3} + \frac{4}{3} \g_t D_g \g^{H} , \notag
\end{eqnarray}
where we have used Lemma~\ref{le:trunc_grad_hess} to control the truncation error in the last term of~\eqref{lem:NSGD-descent-trunc-ineq}. Adding~$J^{\star}$ to both sides and taking expectation,
we obtain \eqref{eq:NSGD_descent2}.
\end{proof}

Given Lemma~\ref{lem:descent-with-grad-dom}, we now control the remaining error term~$\mathbb E[\norm{d_t - \nabla J_H(\theta_t)}]\,.$ 

\begin{lemma}\label{le:NIGT_first_err_bound}
Let Assumptions~\ref{hyp:parameterization-reg} and \ref{hyp:lipschitz-hessian} be satisfied. Let~$(\tilde{\theta}_t), (d_t)$ and~$(\theta_t)$ be the sequences generated by the~\algname{N-PG-IGT} algorithm (see Algorithm~\ref{alg:N-PG-IGT}) 
with $\eta_t = \rb{ \fr{2}{t+2} }^{\nfr{4}{5}}$ and~$\gamma_t = \fr{6  }{\sqrt{2\mu}( t+2 ) }  $. Then for any integer~$t \geq 0$ it holds 
\begin{eqnarray}
     \Exp{ \norm{d_t - \nabla J_H(\theta_t)} } &\leq&     2 \cdot \sqrt{ C\rb{\fr{8}{5}, \fr{4}{5}} } \sigma_g \eta_t^{\nfr{1}{2}}   + C\rb{\fr{6}{5}, \fr{4}{5}} L_h \g_t^2 \eta_t^{-2} \notag \\
     && \qquad + 2 D_g \g^H C\rb{ 0, \fr{4}{5} } \eta_t^{-1} +  D_h \g^{H} C\rb{ 1, \fr{4}{5} } \g_t \eta_t^{-1},
\end{eqnarray}
where $C(p,q)$ for $q \in [0, 1)$, $p > 0 $ is a numerical constant defined in Lemma~\ref{le:sum_prod_bound1}.
\end{lemma}

\begin{proof}
    Define $\hat{e}_t \eqdef d_t - \nabla J_H(\theta_t)$, ${e}_t \eqdef  g(\tilde{\tau}_t, \tilde{\theta}_t) -\nabla J_H(\tilde{\theta}_t)$, 
    $$
    S_{t} \eqdef \nabla J_H({\theta}_{t-1}) - \nabla J_H({\theta}_t) + \nabla^2 J_H(\theta_{t}) \rb{\theta_{t-1} - \theta_{t}}  ,
    $$
    $$
    \bar{S_{t}} \eqdef \nabla J({\theta}_{t-1}) - \nabla J({\theta}_t) + \nabla^2 J(\theta_{t}) \rb{\theta_{t-1} - \theta_{t}}
    $$
    $$
    Z_t \eqdef \nabla J_H(\tilde{\theta}_{t}) - \nabla J_H({\theta}_t) + \nabla^2 J_H(\theta_{t}) \rb{\tilde{\theta}_t - \theta_{t}}.
    $$
    $$
    \bar{Z_t} \eqdef \nabla J(\tilde{\theta}_{t}) - \nabla J({\theta}_t) + \nabla^2 J(\theta_{t}) \rb{\tilde{\theta}_t - \theta_{t}}.
    $$
    Notice that by Lemma~\ref{lem:2nd-order-smoothness} (second-order smoothness), Lemma~\ref{le:trunc_grad_hess} and triangle inequality, we have
    \begin{eqnarray}\label{eq:St_bound}
        \norm{S_{t}} &\leq& L_h \sqnorm{\theta_t - \theta_{t-1}} + \norm{S_t - \bar{S_t} } = L_h \g_{t-1}^2  + \norm{S_t - \bar{S_t} } , \\
        \norm{S_t - \bar{S_t} } & \leq & 2 D_g \g^H + D_h \g^H \g_{t-1} 
    \end{eqnarray}
    \begin{eqnarray}\label{eq:Zt_bound}
    \norm{Z_t} &\leq& L_h \sqnorm{\tilde{\theta}_t - \theta_{t}} +  \norm{Z_t - \bar{Z_t} } = L_h \fr{(1-\eta_t)^2}{\eta_t^2} \sqnorm{\theta_{t} - \theta_{t-1}} +  \norm{Z_t - \bar{Z_t} },   \\
   \norm{Z_t - \bar{Z_t} } &\leq& 2 D_g \g^H + D_h \g^H \fr{\g_{t-1}}{\eta_t}  .
    \end{eqnarray}
    Then, we derive a recursion on the sequence~$(\hat{e}_t)$ by writing 
\begin{align}
    \hat{e}_{t} &= d_{t} - \nabla J_H(\theta_t) \notag \\
    &\overset{(i)}{=} (1 - \eta_t) d_{t-1} + \eta_t  g(\tilde{\tau}_t, \tilde{\theta}_t) - \nabla J_H(\theta_t) \notag \\
    &=  (1 - \eta_t) \rb{\hat{e}_{t-1} + \nabla J_H(\theta_{t-1}) } + \eta_t  g(\tilde{\tau}_t, \tilde{\theta}_t) - \nabla J_H(\theta_t) \notag \\
    &= (1 - \eta_t) \hat{e}_{t-1}  + \eta_t e_{t} + (1 - \eta_t) \rb{  \nabla J_H(\theta_{t-1})  - \nabla J_H(\theta_t)  } 
    +  \eta_t \rb{  \nabla J_H(\tilde{\theta}_t)  - \nabla J_H(\theta_t)  } \notag \\
    &\overset{(ii)}{=} (1 - \eta_t) \hat{e}_{t-1}  + \eta_t e_{t} + (1 - \eta_t) \rb{  \nabla J_H({\theta}_{t-1}) - \nabla J_H({\theta}_t) + \nabla^2 J_H(\theta_{t}) \rb{\theta_{t-1} - \theta_{t}} } \notag\\
    & \qquad +  \eta_t \rb{  \nabla J_H(\tilde{\theta}_{t}) - \nabla J_H({\theta}_t) + \nabla^2 J_H(\theta_{t}) \rb{\tilde{\theta}_t - \theta_{t}} } \notag \\
    &\qquad -(1-\eta_t) \nabla^2 J_H(\theta_t) (\theta_{t-1} - \theta_{t}) - \eta_t \nabla^2 J_{H}(\theta_t) (\tilde{\theta}_t - \theta_t)\\
    &\overset{(iii)}{=} (1 - \eta_t) \hat{e}_{t-1}  + \eta_t e_{t} + (1 - \eta_t) S_t +  \eta_t Z_t \notag ,
\end{align}
where $(i)$ follows from the update rule of the sequence~$(d_t)$ and~$(iii)$ stems from recognizing $S_t$, $Z_t$ and observing that the last term in $(ii)$ is zero thanks to the update rule of the sequence~$(\tilde{\theta}_t)$ in~\algname{N-PG-IGT} (the Hessian correction terms in the last term disappear). Notice that this last simplification is the main reason explaining the update rule of the sequence~$\tilde{\theta}_t$. 

Consider an integer~$T \geq 1\,.$ Unrolling the recursion and using the notation~$\zeta_{t+1,T} \eqdef \prod_{\tau = t+1}^{T-1} (1 - \eta_{\tau+1})$ (with $\zeta_{T,T} = 1$ and~$\zeta_{0,T} = \prod_{t = 0}^{T-1} (1 - \eta_{t+1})$), we obtain
\begin{equation*}
    \hat{e}_{T} = \zeta_{0,T} \hat{e}_0 + \sum_{t = 0}^{T-1} \eta_{t+1} \zeta_{t+1,T} e_{t+1} + \sum_{t = 0}^{T-1} (1-\eta_{t+1}) \zeta_{t+1,T} S_{t+1} 
     + \sum_{t = 0}^{T-1} \eta_{t+1} \zeta_{t+1,T} Z_{t+1}\,.  
\end{equation*}

Define for every integer~$t \geq 1$ the $\sigma$-field $\mathcal{F}_{t} \eqdef \sigma(\cb{\tilde{\theta}_0, \tilde{\tau}_0, \dots, \tilde{\tau}_{t-1} })$ where~$\tilde{\tau}_s \sim p(\cdot|\pi_{\tilde{\theta}_s})$ for every~$0 \leq s \leq t-1$. Notice that for any integers~$t_2 > t_1 \geq 1$  we have 
$$
\Exp{ \langle e_{t_1}, e_{t_2} \rangle } = \Exp{ \Exp{ \langle e_{t_1}, e_{t_2} \rangle | \mathcal{F}_{t_2} } } = \Exp{ \langle e_{t_1}, \Exp{  e_{t_2}  | \mathcal{F}_{t_2} }  \rangle  } = 0.
$$
Then using triangle inequality, taking expectation and applying Jensen's inequality, we get

\begin{eqnarray}
    \Exp{\norm{\hat{e}_{T}}} &\leq & \zeta_{0,T} \Exp{\norm{\hat{e}_0}}  +  \Exp{\norm{ \sum_{t = 0}^{T-1} \eta_{t+1} \zeta_{t+1,T} e_{t+1} }} \notag \\
    &&\qquad 
    +  \Exp{\norm{  \sum_{t = 0}^{T-1} (1-\eta_{t+1}) \zeta_{t+1,T} S_{t+1}  } } +  \Exp{\norm{  \sum_{t = 0}^{T-1} \eta_{t+1} \zeta_{t+1,T} Z_{t+1}  } } \notag \\
    &\leq & \zeta_{0,T} \sigma_g +  \rb{ \Exp{ \sqnorm{ \sum_{t = 0}^{T-1} \eta_{t+1} \zeta_{t+1,T} e_{t+1} }}}^{\nfr{1}{2}} \notag \\
    &&\qquad +  \sum_{t = 0}^{T-1} (1-\eta_{t+1}) \zeta_{t+1,T} \Exp{\norm{S_{t+1}}}  +  \sum_{t = 0}^{T-1} \eta_{t+1} \zeta_{t+1,T} \Exp{\norm{Z_{t+1}}}   \notag \\    
    &\overset{(i)}{\leq} & \zeta_{0,T} \sigma_g  +  \rb{ \sum_{t = 0}^{T-1} \eta_{t+1}^2 \zeta_{t+1,T}^2 \Exp{\sqnorm{ e_{t+1}}} }^{\nfr{1}{2}} \notag \\
    &&\qquad + L_h \sum_{t = 0}^{T-1} (1-\eta_{t+1}) \zeta_{t+1,T} \gamma_{t}^2 
    + L_h \sum_{t = 0}^{T-1} \fr{ \gamma_{t}^2 }{\eta_{t+1}} \zeta_{t+1,T} \notag \\    
    &&\qquad +  \sum_{t = 0}^{T-1} (1-\eta_{t+1}) \zeta_{t+1,T} \norm{S_{t+1} - \bar{S}_{t+1} } 
    +  \sum_{t = 0}^{T-1} \eta_{t+1} \zeta_{t+1,T} \norm{Z_{t+1} - \bar{Z}_{t+1} } \notag \\    
    &\leq & \zeta_{0,T} \sigma_g +  \rb{ \sum_{t = 0}^{T-1} \eta_{t+1}^2  \zeta_{t+1,T}^2 }^{\nfr{1}{2}} \sigma_g  
    + 2 L_h \sum_{t = 0}^{T-1}  \fr{\gamma_{t}^2}{\eta_{t+1}}  \zeta_{t+1,T}  \notag \\
    && \qquad + 2 D_g \g^H \sum_{t = 0}^{T-1} \zeta_{t+1,T}  + D_h \g^H \sum_{t = 0}^{T-1} \gamma_t \zeta_{t+1,T} 
    \label{eq:NIGT_first_erbound_FOSP}  ,   
\end{eqnarray}
where in $(i)$ we apply \eqref{eq:St_bound} and \eqref{eq:Zt_bound}, and use independence $\Exp{ \langle e_{t_1}, e_{t_2} \rangle } = 0$ for any integers~$t_1 \neq t_2$ (larger than 1). In the last inequality, we use $\Exp{\sqnorm{ e_{t+1}}} \leq \sigma_g^2$ for any~$t \in \{0, \cdots, T-1 \}$.

Further, using Lemma~\ref{le:prod_bound} and \ref{le:sum_prod_bound1} we have
\begin{eqnarray}
    \Exp{\norm{\hat{e}_{T}}}  &\leq & \eta_T \sigma_g +  \sqrt{ C\rb{\fr{8}{5}, \fr{4}{5}} } \eta_T^{\nfr{1}{2}} \sigma_g   +  C\rb{\fr{6}{5}, \fr{4}{5}} \g_T^2 \eta_T^{-2} L_h  \notag\\
    && \qquad + 2 D_g \g^H C\rb{ 0, \fr{4}{5} } \eta_T^{-1} +  D_h \g^{H} C\rb{ 1, \fr{4}{5} } \g_T \eta_T^{-1} \notag\\
    &\leq &  2 \cdot \sqrt{ C\rb{\fr{8}{5}, \fr{4}{5}} }  \eta_T^{\nfr{1}{2}} \sigma_g   + C\rb{\fr{6}{5}, \fr{4}{5}} \g_T^2 \eta_T^{-2} L_h  \notag\\
    && \qquad + 2 D_g \g^H C\rb{ 0, \fr{4}{5} } \eta_T^{-1} +  D_h \g^{H} C\rb{ 1, \fr{4}{5} } \g_T \eta_T^{-1}  , \notag 
\end{eqnarray}
where $C(p,q)$ for $q \in [0, 1)$, $p > 0 $ is defined in Lemma~\ref{le:sum_prod_bound1}.
\end{proof}

\noindent\textbf{End of Proof of Theorem~\ref{thm:NIGT} (and hence Theorem~\ref{thm:N-PG-IGT}).} We now combine the ascent-like lemma incorporating the relaxed weak gradient dominance with the estimate of the vanishing error obtained in the previous lemma. 
\begin{proof}
    Combining the results of Lemma~\ref{lem:descent-with-grad-dom} with Lemma~\ref{le:NIGT_first_err_bound} we get 
\begin{eqnarray}\label{eq:NSGD_descent2_applied}
    \delta_{t+1}  &\leq&  \rb{ 1 - \fr{\sqrt{2 \mu} \g_t}{3} } \delta_t + \fr{8 \g_t }{3} \Exp{ \norm{d_t - \nabla J_H(\theta_t)} } + \fr{L_g \g_t^2}{2} + \fr{\g_t \varepsilon^{\prime}}{3} + \fr{4 \g_t D_g}{3 } \g^{{H}} \notag \\
    &\leq&  \rb{ 1 - \fr{\sqrt{2 \mu} \g_t}{3} } \delta_t + \fr{16 }{3}  \sqrt{ C\rb{\fr{8}{5}, \fr{4}{5}} } \sigma_g  \g_t   \eta_t^{\nfr{1}{2}} + \fr{8 }{3} C\rb{\fr{6}{5}, \fr{4}{5}} L_h \g_t^3 \eta_t^{-2} + \fr{L_g \g_t^2}{2} + \fr{\g_t \varepsilon^{\prime}}{3}  + \fr{4 \g_t D_g}{3 } \g^{{H}} \notag \\
    && \qquad + \fr{16  }{3} D_g \g^H C\rb{ 0, \fr{4}{5} } \g_t \eta_t^{-1} +  \fr{8 }{3}  D_h \g^{H} C\rb{ 1, \fr{4}{5} } \g_t^2 \eta_t^{-1}.
\end{eqnarray}
Using Lemma~\ref{le:aux_rec0} with $\al_t = \g_t$, $a = \fr{3}{\sqrt{2\mu}}$, $t_0 = 0$, $\tau = 2$, $\beta_t = \fr{16 }{3}  \sqrt{ C\rb{\fr{8}{5}, \fr{4}{5}} } \sigma_g  \g_t   \eta_t^{\nfr{1}{2}} + \fr{8 }{3} C\rb{\fr{6}{5}, \fr{4}{5}} L_h \g_t^3 \eta_t^{-2} + \fr{L_g \g_t^2}{2} + \fr{\g_t \varepsilon^{\prime}}{3}  + \fr{4 \g_t D_g}{3 } \g^{{H}} + \fr{16  }{3} D_g \g^H C\rb{ 0, \fr{4}{5} } \g_t \eta_t^{-1} +  \fr{8 }{3}  D_h \g^{H} C\rb{ 1, \fr{4}{5} } \g_t^2 \eta_t^{-1} $
\begin{eqnarray}
    \delta_T 
    &\leq& \fr{ \delta_0 }{(T+1)^2} +  \fr{  \sum_{t=0}^{T-1} \beta_t (t + 2 )^2 }{(T+1)^2} \notag \\
    &\leq& \fr{ \delta_0 }{(T+1)^2} +  \fr{16 }{3}  \sqrt{ C\rb{\fr{8}{5}, \fr{4}{5}} } \fr{ \sigma_g \g_0}{(T+1)^{\nfr{2}{5}}} +   \fr{16 }{3} C\rb{\fr{6}{5}, \fr{4}{5}} \fr{ L_h \g_0^3}{(T+1)^{\nfr{2}{5}}} \notag \\
    && \qquad +  \fr{L_g \g_0^2}{2} \fr{ 1}{T+1}  + \fr{\g_0 \varepsilon^{\prime}}{3}  + \fr{4 \g_0 D_g}{3 } \g^{{H}} + \fr{16 \g_0 }{3} D_g \g^H C\rb{ 0, \fr{4}{5} } (T+1)^{\nfr{4}{5}} +  \fr{8 \g_0^2 }{3}  D_h \g^{H} C\rb{ 1, \fr{4}{5} } \fr{1}{(T+1)^{\nfr{1}{5}}} \notag .
\end{eqnarray}
\end{proof}

\subsection{Convergence to first order stationary point}


\begin{theorem}\label{thm:NIGT_FOSP}

   Let Assumptions~\ref{hyp:parameterization-reg} and \ref{hyp:lipschitz-hessian} hold. 
    Set~$\gamma_t = \rb{ \fr{2  }{ t+2  } }^{\nfr{5}{7}} $, $\eta_t = \rb{ \fr{2}{t+2} }^{\nfr{4}{7}}$, $H = \rb{1-\g}^{-1}{\log(T + 1)}$. Let $\bar{\theta}_T$ be sampled from the iterates of \algname{N-PG-IGT} (Algorithm~\ref{alg:N-PG-IGT}) $\cb{\theta_0, \ldots, \theta_{T-1}}$ with probability distribution $\mathbb{P}(\bar{\theta}_T = \theta_t) = \fr{\g_t}{\sum_{t=0}^{T-1} \g_t }$. Then for every integer~$T \geq 1$,  $\bar{\theta}_T$ satisfies
    \begin{eqnarray}
        \Exp{\norm{\nabla J(\bar{ \theta}_T ) }} \leq {\cO}\rb{ \fr{ J^* -  J(\theta_0) }{(T+1)^{\nfr{2}{7}}} + \fr{L_g}{(T+1)^{\nfr{5}{7}}} + \fr{D_g}{(T+1)} + \fr{D_h}{(T+1)^{\nfr{13}{7}}} + \fr{ (L_h + \sigma_g) \log(T+1) }{ (T+1)^{\nfr{2}{7}} }  } . \notag 
    \end{eqnarray}
    where~$\sigma_g, L_g, L_h$ are defined in Proposition~\ref{prop:stoch-grad-and-J} and Lemma~\ref{lem:2nd-order-smoothness} respectively, $D_g$ is defined in Lemma~\ref{le:trunc_grad_hess}. Here $\cO(\cdot)$ only hides absolute numerical constants.
\end{theorem}

\begin{proof}
    It follows from Lemma~\ref{le:NSGD_descent_trunc} combined with Lemma~\ref{le:trunc_grad_hess} that
    \begin{eqnarray}\label{eq:NSGD_descent_bis}
    -J(\theta_{t+1}) \leq  - J(\theta_{t}) - \fr{\g_t}{3}  \norm{ \nabla J(\theta_t) } + \fr{8 \g_t }{3} \norm{ d_t - \nabla J_{H}(\theta_{t}) } + \fr{L_g \g_t^2}{2} + \fr{4 \g_t D_g}{3 } \g^{{H}}.
\end{eqnarray}
    We now control the error term~$\|\hat{e}_t\| = \|d_t - \nabla J_H(\theta_t)\|$ as follows. 
    By \eqref{eq:NIGT_first_erbound_FOSP} in the proof of Lemma~\ref{le:NIGT_first_err_bound}, we have
    \begin{eqnarray}
    \Exp{\norm{\hat{e}_{T}}} &\leq & 
     \zeta_{0,T} \sigma_g +  \rb{ \sum_{t = 0}^{T-1} \eta_{t+1}^2  \zeta_{t+1,T}^2 }^{\nfr{1}{2}} \sigma_g  
    + 2 L_h \sum_{t = 0}^{T-1}  \fr{\gamma_{t}^2}{\eta_{t+1}}  \zeta_{t+1,T}  + 2 D_g \g^H \sum_{t = 0}^{T-1} \zeta_{t+1,T}  + D_h \g^H \sum_{t = 0}^{T-1} \gamma_t \zeta_{t+1,T}  \notag \\
    & \overset{(i)}{\leq} & 2 \cdot \sqrt{ C\rb{\fr{8}{7}, \fr{4}{7}} }  \eta_T^{\nfr{1}{2}} \sigma_g   + C\rb{\fr{6}{7}, \fr{4}{7}} \g_T^2 \eta_T^{-2} L_h  + 2 D_g \g^H C\rb{ 0, \fr{4}{7} } \eta_T^{-1} +  D_h \g^{H} C\rb{ \fr{5}{7}, \fr{4}{7} } \g_T \eta_T^{-1}  \notag,
\end{eqnarray}
    where in $(i)$ we applied Lemma~\ref{le:prod_bound} and \ref{le:sum_prod_bound1}. The value of $C(p,q)$ for $q \in [0, 1)$, $p > 0 $ is defined in Lemma~\ref{le:sum_prod_bound1}, $\hat{e}_t = d_t - \nabla J_H (\theta_t)$


    Combining the above inequality with~(\ref{eq:NSGD_descent_bis}), we obtain
    \begin{eqnarray}
    \Exp{ J(\theta_{t}) -J(\theta_{t+1}) } &\leq &  - \fr{\g_t}{3}  \Exp{ \norm{ \nabla J(\theta_t) } } + {8 \g_t } \sqrt{ C\rb{\fr{8}{7}, \fr{4}{7}} }  \eta_T^{\nfr{1}{2}} \sigma_g + {8 \g_t } C\rb{\fr{6}{7}, \fr{4}{7}} \g_T^2 \eta_T^{-2} L_h + \fr{L_g \g_t^2}{2} + \fr{\varepsilon^{\prime}\g_t }{3} \notag \\
    && \qquad + \fr{4 \g_t D_g}{3 } \g^{{H}}  + \fr{16}{3} D_g \g^H C\rb{ 0, \fr{4}{7} } \g_t \eta_t^{-1} +  \fr{8 }{3} D_h \g^{H} C\rb{ \fr{5}{7}, \fr{4}{7} } \g_t^2 \eta_T^{-1} .
\end{eqnarray}
    Summing up the above inequality for $t = 0, \ldots, T-1$ and rearranging the terms, we arrive at
    \begin{eqnarray}
        \Exp{\norm{\nabla J(\bar{ \theta}_T ) }} & = & \fr{\sum_{t=0}^{T-1} \g_t \Exp{ \norm{ \nabla J(\theta_t) } }}{ \sum_{t=0}^{T-1} \g_t  } \notag \\
        &\leq&  {\cO}\rb{ \fr{ J^* -  J(\theta_0) }{(T+1)^{\nfr{2}{7}}} + \fr{L_g}{(T+1)^{\nfr{5}{7}}} + \fr{D_g}{(T+1)} + \fr{D_h}{(T+1)^{\nfr{13}{7}}} + \fr{ (L_h + \sigma_g) \log(T+1) }{ (T+1)^{\nfr{2}{7}} }  } .
    \end{eqnarray}
\end{proof}

\newpage
\section{Proof of Theorem~\ref{thm:HARPG} (\algname{HARPG})}
\label{sec:stormhess_app}
In this section, for the proof of Theorem~\ref{thm:HARPG} (for~\algname{HARPG}), we introduce the following additional useful shorthand notations:
\begin{align}
\label{eq:notations-harpg-1}
    \delta_t & \eqdef J^{\star} - \Exp{ J(\theta_t) }\,,\\
    V_t &\eqdef \Exp{\sqnorm{ d_t - \nabla J(\theta_t)}}\,,\label{eq:notations-harpg-2}\\
    \Lambda_t &\eqdef \delta_t + \lambda_{t-1} V_t\,,\label{eq:notations-harpg-3}\\
    R_t &\eqdef \Exp{\sqnorm{\theta_{t+1} - \theta_{t}}}\label{eq:notations-harpg-4}\,,
\end{align}
where~$(\theta_t), (d_t)$ are the sequences generated by~\algname{HARPG} and $(\lambda_t)$ is a sequence of positive scalars to be determined later. 

First, we state a detailed version of Theorem~\ref{thm:HARPG} before providing a complete proof. 

\begin{theorem}\label{thm:STORM_H}
Let Assumptions~\ref{hyp:parameterization-reg}, \ref{hyp:fisher-non-degenerate} and~\ref{hyp:tranf-compatib-fun-approx} hold.
    Set~$\gamma_t = \g_0 \eta_t^{\nfr{1}{2}} $, $\eta_t = \fr{5 }{t+5} $ with $\g_0 = \min\cb{ \fr{\eta_0^{\nfr{1}{2}}}{ 8 \sqrt{3} ( L_g + \sigma_g + D_h \g^{2 H} ) } , \fr{1}{ \sigma_g \sqrt{3\mu}   }  } $ and~$H = 2 \rb{1-\g}^{-1}{\log(T + 4)}$. Let~$\lambda_t = 4 \gamma_0 \eta_t^{-\nfr{1}{2}}$ . Then for every integer~$T \geq 1$, the output~$\theta_T$ of \algname{HARPG} (see Algorithm~\ref{alg:(N)-HARPG}) satisfies
    \begin{eqnarray}
         J^{\star} - \Exp{J(\theta_T) } &\leq& \cO \rb{\fr{\Lambda_0}{(T+4)^2} + \fr{ \sigma_g \sqrt{\mu}+ L_g + \sigma_h  }{(T+4)^{\nfr{1}{2}} }  + \frac{\lambda_0 D_g^2}{(T+4)^{\nfr{5}{2}}}} + \frac{\sqrt{2} \varepsilon^\prime}{\sqrt{\mu}} \notag,
    \end{eqnarray}
    where~$\sigma_g, \sigma_h, L_g$ are defined in Proposition~\ref{prop:stoch-grad-and-J} and Lemma~\ref{lem:2nd-order-smoothness} respectively, $D_g$ and $D_h$ are defined in Lemma~\ref{le:trunc_grad_hess}, $\varepsilon^{\prime}$ and $\mu$ are defined in Lemma~\ref{lem:relaxed-w-grad-dom}.

\end{theorem}

Our first step in the proof is an ascent-like lemma which is different from the previously used similar results of Lemma~\ref{le:NSGD_descent_trunc} and Lemma~\ref{lem:descent-with-grad-dom} since the \algname{HARPG} algorithm does not use a normalized direction to update the policy parameter. 

\begin{lemma}\label{le:SGD_descent}
Let Assumptions~\ref{hyp:parameterization-reg}, \ref{hyp:fisher-non-degenerate} and \ref{hyp:tranf-compatib-fun-approx} be satisfied. 
Let~$(d_t)$ be any sequence of vectors in~$\R^d$ and consider the sequence~$(\theta_t)$ defined by:
$$
\theta_{t+1} = \theta_t + \g_t d_t\,,
$$
where~$\theta_0 \in \R^d$\,.
Then for every integer~$t \geq 1$, for any positive step-size~$\g_t$ s.t. $\g_t \leq \fr{1}{ 2 L_g}$, we have
\begin{eqnarray}\label{eq:SGD_descent1}
    -J(\theta_{t+1}) &\leq & - J(\theta_{t}) - \fr{\mu \g_t}{2} \rb{ J^{\star} - J(\theta_{t})  }^2 + \g_t \sqnorm{d_t - \nabla J_{H}(\theta_{t})} - \fr{1}{4\g_t} \sqnorm{\theta_{t+1} - \theta_{t}} + \fr{\rb{\varepsilon^{\prime}}^2 \g_t }{2} + {\g_t} D_g^2 \g^{2H} . \notag \\
\end{eqnarray}
Furthermore, we also have
\begin{eqnarray}\label{eq:SGD_descent2}
    \delta_{t+1} - \delta_{t} \leq - \fr{\g_t \mu}{2} \delta_t^2 + {\gamma_t} V_t - \fr{1}{4 \gamma_t} R_t + \fr{\rb{\varepsilon^{\prime}}^2 \g_t}{2} + {\g_t} D_g^2 \g^{2H}\,,
\end{eqnarray}
where we use the same notations as in~\eqref{eq:notations-harpg-1} to~\eqref{eq:notations-harpg-4}.
\end{lemma}

\begin{proof}
Using the smoothness of~$J$ (Proposition~\ref{prop:stoch-grad-and-J}) and the update rule of the sequence~$(\theta_t)$, we get
\begin{eqnarray}
    -J(\theta_{t+1}) &\leq& - J(\theta_{t}) - \langle \nabla J(\theta_{t}), \theta_{t+1} - \theta_t \rangle + \fr{L_g}{2} \sqnorm{ \theta_{t+1} - \theta_t } \notag \\
       &= & - J(\theta_{t}) -  \g_t \left( \fr{1}{2} \sqnorm{\nabla J(\theta_{t}) } +  \fr{1}{2} \sqnorm{d_t} -  \fr{1}{2} \sqnorm{d_t - \nabla J(\theta_{t})} \right) + \fr{L_g}{2} \sqnorm{ \theta_{t+1} - \theta_t} \notag\\
     &= & - J(\theta_{t}) -  \fr{\g_t}{2} \sqnorm{ \nabla J(\theta_{t}) } + \fr{\g_t}{2}\sqnorm{d_t - \nabla J(\theta_{t}) } - \rb{ \fr{1}{2\g_t} - \fr{L_g}{2} } \sqnorm{ \theta_{t+1} - \theta_t } \notag\\
     &\leq & - J(\theta_{t}) -  \fr{\g_t}{2} \sqnorm{ \nabla J(\theta_{t}) } + {\g_t}\sqnorm{d_t - \nabla J_H(\theta_{t}) } - \rb{ \fr{1}{2\g_t} - \fr{L_g}{2} } \sqnorm{ \theta_{t+1} - \theta_t } + {\g_t}\sqnorm{\nabla J(\theta_{t}) - \nabla J_H(\theta_{t}) } \notag \\
     &\leq & - J(\theta_{t}) -  \fr{\g_t}{2} \sqnorm{ \nabla J(\theta_{t}) } + {\g_t}\sqnorm{d_t - \nabla J_H(\theta_{t}) } - \rb{ \fr{1}{2\g_t} - \fr{L_g}{2} } \sqnorm{ \theta_{t+1} - \theta_t } + {\g_t} D_g^2 \g^{2H} , \notag
\end{eqnarray}
where the last inequality holds by Lemma~\ref{le:trunc_grad_hess}. By Lemma~\ref{lem:relaxed-w-grad-dom}  we have for any $\theta \in \R^d$
$$
\rb{\varepsilon^{\prime}}^2 + \sqnorm{\nabla J(\theta)} \geq \fr{\rb{\varepsilon^{\prime} + \|\nabla J(\theta)\|}^2}{2} \geq \mu (J(\theta) - J^{\star})^2. 
$$
Using the above inequality and the assumption on the step-size, we arrive at \eqref{eq:SGD_descent1}.
Adding $J^{\star}$ to both sides, taking expectation and using Jensen's inequality for $x \mapsto x^2 $, we obtain \eqref{eq:SGD_descent2}.
\end{proof}

Given \eqref{eq:SGD_descent2} in Lemma~\ref{le:SGD_descent}, we now provide a recursive bound for the sequence~$(V_t)$ as defined in~\eqref{eq:notations-harpg-2}.  

\begin{lemma}\label{le:g_upd_rec_Hess_Aid_trunc}
Let Assumption~\ref{hyp:parameterization-reg} be satisfied and suppose~$\eta_t \le 1$. Let the sequence~$(d_t)$ be updated for every integer~$t \geq 1$ via
$$
d_t = (1 - \eta_t) \rb{ d_{t-1} +  B(\hat{\tau}_t, \hat{\theta}_{t})(\theta_t - \theta_{t-1}) } + \eta_t g(\tau_t, \theta_t)\,,
$$
as in Algorithm~\ref{alg:(N)-HARPG}.
Then, for every integer~$t$, the sequence~$(V_t)$ satisfies the following recursion 
\begin{eqnarray}
    V_{t+1} - V_{t} \leq - \eta_{t+1} V_{t} + 2\sigma_g^2 \eta_{t+1}^2  + 12 (2 L_g^2 + \sigma_h^2 + D_h^2 \g^{2 H})  \Exp{\sqnorm{ \theta_{t+1} - \theta_{t} } } + 24 D_g^2 \g^{2 {H}} .  \notag 
\end{eqnarray}
\end{lemma}

\begin{proof}
Define $\cV_t \eqdef g(\tau_t, \theta_t)  - \nabla J_{H}(\theta_{t})$, $\cW_t \eqdef \nabla J_{H}(\theta_{t-1}) - \nabla J_{H}(\theta_{t}) + B(\hat{\tau}_t, \hat{\theta}_{t})(\theta_t - \theta_{t-1})$ .
Notice that $\Exp{\cV_t} = 0$ and~$\Exp{\sqnorm{\cV_t}} \leq \sigma_g^2$. Moreover, we also have~$\Exp{\cW_t} = 0$ since 
\begin{eqnarray}
\Exp{\cW_t} & = & \Exp{ \nabla J_{H}(\theta_{t-1}) - \nabla J_{H}(\theta_{t}) + B(\hat{\tau}_t, \hat{\theta}_{t})(\theta_t - \theta_{t-1}) } \notag \\
& = & \Exp{ \Exp{ \nabla J_{H}(\theta_{t-1}) - \nabla J_{H}(\theta_{t}) + B(\hat{\tau}_t, \hat{\theta}_{t})(\theta_t - \theta_{t-1}) \mid \theta_{t-1}, \theta_{t}, \hat{\theta}_{t} }  }\notag \\
& = & \Exp{ \nabla J_{H}(\theta_{t-1}) - \nabla J_{H}(\theta_{t}) + \nabla^2 J_{H}(\hat{\theta}_{t}) ( \theta_{t} - \theta_{t-1} )  } \notag \\
& = & \Exp{ \nabla J_{H}(\theta_{t-1}) - \nabla J_{H}(\theta_{t}) } + \Exp{ \int_0^1 \nabla^2 J_{H}( q \theta_{t} + (1 - q) \theta_{t-1} ) (\theta_{t} - \theta_{t-1} ) dq } = 0. \notag
\end{eqnarray}

Furthermore, we can control the variance of the sequence~$(\cW_t)$ as follows: 
\begin{eqnarray}
\Exp{\sqnorm{ \cW_t } } & = & \Exp{ \sqnorm{  \nabla J_{H}(\theta_{t-1}) - \nabla J_{H}(\theta_{t}) + B(\hat{\tau}_t, \hat{\theta}_{t})(\theta_t - \theta_{t-1}) } } \notag \\
& \leq & 6\, \Exp{ \sqnorm{ \nabla J_H(\theta_{t-1}) - \nabla J(\theta_{t-1}) } } 
+  6 \, \Exp{ \sqnorm{ \nabla J(\theta_{t-1}) - \nabla J(\theta_{t}) } } \notag \\
&& + 6\, \Exp{ \sqnorm{ \nabla J(\theta_{t}) - \nabla J_H(\theta_{t}) } } 
+ 6\, \Exp{ \sqnorm{ \rb{B(\hat{\tau}_t, \hat{\theta}_{t}) -  \nabla^2 J_H(\hat{\theta}_{t}) } (\theta_t - \theta_{t-1})  } }  \notag \\
&& + 6\, \Exp{ \sqnorm{ \rb{ \nabla^2 J_{H} (\hat{\theta}_{t}) - \nabla^2 J (\hat{\theta}_{t}) } (\theta_t - \theta_{t-1}) } }
+  6\, \Exp{\sqnorm{ \nabla^2 J (\hat{\theta}_{t}) (\theta_t - \theta_{t-1}) } }  \notag \\
& \leq & 2 (6 L_g^2 + 3 \sigma_h^2) \Exp{\sqnorm{ \theta_t - \theta_{t-1} } } + 12 D_g^2 \g^{2 {H}} + 6 D_h^2 \g^{2 H} \Exp{\sqnorm{\theta_t - \theta_{t-1}}},
\end{eqnarray}
where the last inequality stems from using the smoothness of the function~$J$, the boundedness of the hessian estimate variance (see Proposition~\ref{prop:stoch-grad-and-J} for both results) and Lemma~\ref{le:trunc_grad_hess} (to control truncation errors). 

We now rewrite our quantity of interest~$d_t - \nabla J_H(\theta_t)$ using the update rule of the sequence~$(d_t)$ together with the definitions of the  sequences~$(\cV_t)$ and~$(\cW_t)$ that we introduced for the purpose of this proof. We have  
\begin{eqnarray}
    d_{t} - \nabla J_H(\theta_{t})  &= & (1 - \eta_t) \rb{ d_{t-1} +  B(\hat{\tau}_t, \hat{\theta}_{t})(\theta_t - \theta_{t-1}) } + \eta_t g(\tau_t, \theta_t)  - \nabla J_H(\theta_{t}) \notag \\
    &= & (1 - \eta_t)   d_{t-1}   + \eta_t g(\tau_t, \theta_t)
    + (1 - \eta_t)  B(\hat{\tau}_t, \hat{\theta}_{t})(\theta_t - \theta_{t-1}) - \nabla J_H(\theta_{t})  \notag \\
    &= & (1 - \eta_t)  \rb{ d_{t-1} -  \nabla J_H(\theta_{t-1}) }   + \eta_t g(\tau_t, \theta_t) \notag \\
    && \qquad +\, (1 - \eta_t)  B(\hat{\tau}_t, \hat{\theta}_{t})(\theta_t - \theta_{t-1}) -  \nabla J_H(\theta_{t}) + (1 - \eta_t) \nabla J_H(\theta_{t-1}) \notag \\
    &= & (1 - \eta_t)  \rb{ d_{t-1} -  \nabla J_H(\theta_{t-1}) }   + \eta_t \rb{ g(\tau_t, \theta_t) -  \nabla J_H(\theta_{t}) } \notag \\
    && \qquad +\, (1 - \eta_t)  \rb{  \nabla J_H(\theta_{t-1}) -   \nabla J_H(\theta_{t}) +  B(\hat{\tau}_t, \hat{\theta}_{t})(\theta_t - \theta_{t-1}) }   \notag \\
    &= & (1 - \eta_t) \rb{ d_{t-1} -  \nabla J_H(\theta_{t-1}) } + \eta_t \cV_{t} + (1 - \eta_t) \cW_t    \notag.
\end{eqnarray}

Using the above decomposition, we can derive the desired recursive upper bound on the sequence~$(V_t)$: 
\begin{eqnarray}
    V_{t} &=& \Exp{\sqnorm{d_{t} - \nabla J_H(\theta_{t})  } } \notag \\
    & = & \Exp{\sqnorm{ (1 - \eta_t) \rb{ d_{t-1} -  \nabla J_H(\theta_{t-1}) } + \eta_t \cV_{t} + (1 - \eta_t) \cW_t }  } \notag \\
    & = &  (1 - \eta_t)^2 \Exp{ \sqnorm{ d_{t-1} -  \nabla J_H(\theta_{t-1}) }  } + \Exp{ \sqnorm{ \eta_t \cV_{t} + (1 - \eta_t) \cW_t }  } \notag \\
    & \leq &  (1 - \eta_t)V_{t-1} + 2 \eta_t^2 \Exp{ \sqnorm{  \cV_{t} }  } + 2 \Exp{ \sqnorm{ \cW_t } } \notag \\
    & \leq &  (1 - \eta_t) V_{t-1} + 2\sigma_g^2 \eta_t^2  + 12 (2 L_g^2 + \sigma_h^2 + D_h^2 \g^{2 H})\,   \Exp{\sqnorm{ \theta_t - \theta_{t-1} } } + 24 D_g^2 \g^{2 H} , \notag 
\end{eqnarray}
where the third identity follows from observing that the conditional expectation of the random variable~$\eta_t \cV_t + (1-\eta_t) \cW_t$ w.r.t.~$\theta_{t-1}$ is equal to zero and using the tower property of the conditional expectation, the first inequality uses the assumption~$\eta_t \le 1$ and the last inequality utilizes the variance bounds established at the beginning of the proof of this lemma.
\end{proof}

Our next step is to combine the results of Lemma~\ref{le:SGD_descent} and Lemma~\ref{le:g_upd_rec_Hess_Aid_trunc}. For this purpose we use the sequence~$(\Lambda_t)$ defined in~\eqref{eq:notations-harpg-3} as a Lyapunov function. Notice that it combines the quantities~$\delta_t$ and~$V_t$ with a suitable weighting sequence~$(\lambda_t)$ (to be defined in the next result). For convenience, we recall here the definition of the sequence~$(\Lambda_t)$: 
\begin{equation*}
    \Lambda_{t} \eqdef \delta_t + \lambda_{t-1} V_t\,.
\end{equation*}

\begin{lemma}\label{le:key_recursion_Hess_Aid}
Let Assumptions~\ref{hyp:parameterization-reg}, \ref{hyp:fisher-non-degenerate} and \ref{hyp:tranf-compatib-fun-approx} be satisfied. Set~$\eta_t =  \fr{5 }{t+5}$ and~$\gamma_t \leq  \g_0 \eta_t^{\nfr{1}{2}} $. Set $\lambda_t = \lambda_0 \eta_t^{-\nfr{1}{2}}$ with $\lambda_0 = 4 \g_0 $ and~$\g_0 \leq \fr{1}{ 8 \sqrt{3} ( L_g + \sigma_h + D_h \g^{ 2 H} ) } $. Then for any integer~$t \geq 1$ either $\delta_t  \leq \fr{\sqrt{2} \varepsilon^{\prime}}{\sqrt{\mu}}$ or 
\begin{eqnarray}
    \Lambda_{t+1} \leq  \rb{ 1 - \fr{  \eta_t }{ 5 }  } \Lambda_t +  \rb{ \fr{1}{\mu \g_0}    + 2 \lambda_0 \sigma_g^2 } \eta_t^{\nfr{3}{2}} + 24 \lambda_0 \eta_t^{-\nfr{1}{2}} D_g^2 \g^{2 {H}}  . \notag
\end{eqnarray}
\end{lemma}
\begin{proof}
By Lemma~\ref{le:SGD_descent}, we have
\begin{eqnarray}\label{eq:STORMH_descent}
    \delta_{t+1} - \delta_{t} \leq - \fr{\g_t \mu}{2} \delta_t^2 + {\gamma_t} V_t - \fr{1}{4 \gamma_t} R_t + \fr{\rb{\varepsilon^{\prime}}^2 \g_t}{2} + {\g_t} D_g^2 \g^{2H}  .
\end{eqnarray}
Lemma~\ref{le:g_upd_rec_Hess_Aid_trunc} provides the following inequality: 
\begin{eqnarray}\label{eq:STORMH_Vt}
    V_{t+1} - V_{t} \leq - \eta_{t+1} V_{t} + 2\sigma_g^2 \eta_{t+1}^2  + 12 (2 L_g^2 + \sigma_h^2 + D_h^2 \g^{2 H})  R_t + 24 D_g^2 \g^{2 {H}} .   
\end{eqnarray}
Summing up \eqref{eq:STORMH_descent} with a $\lambda_t = \lambda_0 \eta_t^{-\nfr{1}{2}}$ multiple of \eqref{eq:STORMH_Vt}, we obtain
\begin{eqnarray}
    \Lambda_{t+1} - \Lambda_{t} & = & \delta_{t+1} - \delta_t + \lambda_{t} (V_{t+1} - V_t) + (\lambda_{t} - \lambda_{t-1}) V_t  \notag\\
     &\leq& - \fr{\g_t \mu}{2} \delta_t^2 + \rb{ \fr{\gamma_t}{2} +  (\lambda_t - \lambda_{t-1}) - \lambda_t \eta_{t+1}}  V_t + 2 \lambda_t \sigma_g^2 \eta_{t+1}^2 \notag \\
     && \qquad - \fr{1}{4 \gamma_t} R_t + 12\lambda_t (2 L_g^2 + \sigma_h^2 + D_h^2 \g^{2 H})  R_t + \fr{\g_t \rb{\varepsilon^{\prime}}^2}{2} + 24 \lambda_t D_g^2 \g^{2 {H}}   \notag \\
    &\overset{(i)}{\leq}& - \fr{\g_t \mu}{2}  \delta_t^2 + \rb{ \fr{\gamma_t}{2} +  (\lambda_t - \lambda_{t-1})  -\lambda_t \eta_{t+1}}  V_t + 2\lambda_t \sigma_g^2 \eta_{t+1}^2  + \fr{\g_t \rb{\varepsilon^{\prime}}^2}{2} \notag  + 24 \lambda_t D_g^2 \g^{2 {H}}  \\
    &\overset{(ii)}{\leq}& - \fr{\g_t \mu}{2}  \delta_t^2  -  \fr{5}{6} \gamma_t  V_t + 2\lambda_t \sigma_g^2 \eta_{t+1}^2  + \fr{\g_t \rb{\varepsilon^{\prime}}^2}{2} + 24 \lambda_t D_g^2 \g^{2 {H}} \notag \\
    &\overset{(iii)}{\leq}& - \fr{\g_t \mu}{4}  \delta_t^2  - \fr{5}{6} \gamma_t  V_t + 2\lambda_t \sigma_g^2 \eta_{t+1}^2 + 24 \lambda_t D_g^2 \g^{2 {H}} \notag \\
\end{eqnarray}
where $(i)$ follows by  $\lambda_0 = 4\gamma_0 \leq \rb{48 (2 L_g^2 + \sigma_h^2 + D_h^2 \g^{2 H}) \g_0}^{-1} $. The choice $\g_0 = \fr{\lambda_0 }{4}$ and the bound $\lambda_t - \lambda_{t-1} \leq \fr{\lambda_0}{2}\eta_{t}^{\nfr{1}{2}}$ implies $$\fr{\gamma_t}{2} +  (\lambda_t - \lambda_{t-1}) - \lambda_t \eta_{t+1} \leq \fr{\gamma_t}{2} + \fr{\lambda_0}{2}\eta_{t}^{\nfr{1}{2}}  -  \fr{5}{6} \lambda_0 \leq -\fr{5}{6} \g_t,$$
thus $(ii)$ holds. The last step $(iii)$ follows by the assumption that $\delta_t  \geq \fr{\sqrt{2} \varepsilon^{\prime}}{\sqrt{\mu}}$.

Noticing that  $\delta_t = \Lambda_t - \lambda_{t-1} V_t$, we have 

\begin{eqnarray}
     \Lambda_{t+1} - \Lambda_{t} &{\leq}& - \fr{\g_t \mu}{4}  \rb{ \Lambda_t - \lambda_{t-1} V_t }^2  - \fr{5}{6} \gamma_t V_t + 2\lambda_t \sigma_g^2 \eta_{t+1}^2  + 24 \lambda_t D_g^2 \g^{2 {H}} \notag \\
    &\leq&  - \fr{\g_t \mu}{4}  \Lambda_{t}^2 - \fr{\g_t \mu \lambda_{t-1}^2}{4}  V_{t}^2 + \fr{\g_t
    }{2}\rb{\mu \lambda_{t-1} \Lambda_t - \fr{5}{3} }  V_t + 2\lambda_t \sigma_g^2 \eta_{t+1}^2 + 24 \lambda_t D_g^2 \g^{2 {H}} \notag \\
    &\overset{(i)}{\leq}&  - \fr{\g_t \mu}{4}  \Lambda_{t}^2 + \fr{\gamma_t^2 }{4} \rb{\mu \lambda_{t-1} \Lambda_t - \fr{5}{3} }^2  \fr{1}{\mu \g_t \lambda_{t-1}^2} + 2\lambda_t \sigma_g^2 \eta_{t+1}^2  + 24 \lambda_t D_g^2 \g^{2 {H}} \notag \\
    &=& - \fr{\g_t \mu}{4}  \Lambda_{t}^2 + \rb{\mu \lambda_{t-1} \Lambda_t - \fr{5}{3} }^2   \fr{\g_t}{4 \mu \lambda_{t-1}^2} + 2\lambda_0 \sigma_g^2 \eta_{t+1}^{\nfr{3}{2}}  + 24 \lambda_t D_g^2 \g^{2 {H}}   \notag \\
     &{=}&   - \fr{5}{6} \fr{  \g_t }{ \lambda_{t-1}}  \Lambda_t + \fr{25}{36} \fr{\gamma_t}{ \mu \lambda_{t-1}^2}   + 2\lambda_0 \sigma_g^2 \eta_{t+1}^{\nfr{3}{2}}  + 24 \lambda_0 \eta_t^{-\nfr{1}{2}} D_g^2 \g^{2 {H}}  \notag \\
     &{=}&   - \fr{ 5  }{24 } \eta_t \Lambda_t + \fr{25 }{ 576 } \fr{1}{\mu \g_0}  \eta_t^{\nfr{3}{2}}  + 2 \lambda_0 \sigma_g^2 \eta_{t+1}^{\nfr{3}{2}}  + 24 \lambda_0 \eta_t^{-\nfr{1}{2}} D_g^2 \g^{2 {H}}  \notag \\
     &{\leq}&   - \fr{  \eta_t }{5 }  \Lambda_t +  \rb{ \fr{1}{\mu \g_0}    + 2 \lambda_0 \sigma_g^2 } \eta_t^{\nfr{3}{2}} + 24 \lambda_0 \eta_t^{-\nfr{1}{2}} D_g^2 \g^{2 {H}} ,  \notag 
\end{eqnarray}
   where in $(i)$ we use $- a x^2 + c x \leq \fr{c^2}{4 a}$ with $x = V_t$, $a =  \fr{\g_t \mu \lambda_{t-1}^2}{4}$, $c = \fr{1}{2}\g_t \rb{   \mu \lambda_{t-1} \Lambda_t - \fr{1}{6} } $.

\end{proof}

\noindent\textbf{End of the proof of Theorem~\ref{thm:STORM_H} (and hence Theorem~\ref{thm:HARPG}).}
    By Lemma~\ref{le:key_recursion_Hess_Aid}, we have that either $\delta_t \le \frac{\sqrt{2} \epsilon^\prime}{\sqrt{\mu}}$ or
    \begin{eqnarray}
    \Lambda_{t+1} \leq  \rb{ 1 - \fr{  \eta_t }{5 }  } \Lambda_t +   \rb{ \fr{1}{\mu \g_0}    + 8 \g_0 \sigma_g^2 } \eta_t^{\nfr{3}{2}} + 24 \lambda_0 \eta_t^{-\nfr{1}{2}} D_g^2 \g^{2 {H}}   . \notag
\end{eqnarray}
    Using Lemma~\ref{le:aux_rec0} with $t_0 = 0$, $\tau = 5$, $r_t = \Lambda_t$, $\alpha_t = \eta_t$, $a = \fr{2}{5}$, $\beta_t = c_1 \eta_t^{\nfr{3}{2}} + c_2\eta_t^{-\nfr{1}{2}} $, $c_1 = \rb{ \fr{1}{\mu \g_0}    + 8 \g_0 \sigma_g^2 }$, $c_2 = 24 \lambda_0  D_g^2 \g^{2 {H}} $ we have for every integer~$T \geq 1$
\begin{eqnarray}
\Lambda_{T} &\leq& \fr{\Lambda_0}{(T+4)^2} + \fr{\sum_{t=0}^{T-1} (c_1 \eta_t^{\nfr{3}{2}} + c_2\eta_t^{-\nfr{1}{2}}  ) (t+5)^2}{(T+4)^2} \notag \\
&\leq& \cO \rb{ \fr{\Lambda_0}{(T+4)^2} + \fr{4\sum_{t=0}^{T-1} c_1 (t+5)^{\nfr{1}{2}} + c_2 (t+5)^{\nfr{5}{2}}  }{(T+4)^2} } \notag \\
&\leq& \cO\rb{ \fr{\Lambda_0}{(T+4)^2} + \fr{4 c_1}{(T+4)^{\nfr{1}{2}} } + 4 c_2 (T+4)^{\nfr{3}{2}} } . 
\end{eqnarray}

Selecting $\g_0 = \min\cb{ \fr{1}{ 8 \sqrt{3} ( L_g + \sigma_g + D_h \g^{2 H} ) } , \fr{1}{ 2 \sigma_g \sqrt{2\mu}   }  } $, we have $c_1 \leq \fr{ 8\sqrt{3} (L_g + \sigma_h + D_h \g^{2 H} ) }{\mu} + \fr{2\sqrt{2} \sigma_g }{\sqrt{\mu}} $. 


Since $\log x \le x-1$ for $x > 0$ we have
\begin{eqnarray*}
    \g^{2 H } = \g^{4 \log (T + 4) (1 - \g)^{-1} } \le \g^{-2 \frac{\log (T + 4)}{ \log \g} } = \fr{1}{(T+4)^2}
\end{eqnarray*}
and thus $c_2 \le 24 \lambda_0 D_g^2 \fr{1}{(T+4)^2}$.


Combining all the results together we obtain
$$
\delta_T \le \max \Big\{\Lambda_T,\frac{\sqrt{2} \epsilon^\prime}{\sqrt{\mu}} \Big \} \le \cO \rb{\fr{\Lambda_0}{(T+4)^2} + \fr{ \sigma_g \sqrt{\mu}+ L_g + \sigma_h  }{(T+4)^{\nfr{1}{2}} }  + \frac{\lambda_0 D_g^2}{(T+4)^{\nfr{5}{2}}}} + \frac{\sqrt{2} \varepsilon^\prime}{\sqrt{\mu}}.
$$

\newpage
\section{Proof of Theorem~\ref{thm:N-HARPG} (\algname{N-HARPG})}
\label{sec:nstormhess_app}
We start by restating Theorem~\ref{thm:N-HARPG} in more details. 

\begin{theorem}\label{thm:NSTORM_Hess_trunc}
    Let Assumptions~\ref{hyp:parameterization-reg}, \ref{hyp:fisher-non-degenerate} and~\ref{hyp:tranf-compatib-fun-approx} hold.
    Set~$\gamma_t = \fr{3}{\sqrt{2\mu} (t+2) } = \fr{6 M_g}{\mu_F(t+2)}$ and~$\eta_t = \fr{2}{t+2} $, $H = \fr{3}{2} \rb{1-\g}^{-1}{\log(T + 1)}$. Then for every integer~$T \geq 1$, the output~$\theta_T$ of \algname{N-HARPG} (see Algorithm~\ref{alg:(N)-HARPG}) satisfies
    \begin{eqnarray}
        J^* - \Exp{ J(\theta_T) } &\leq& \cO\rb{ \fr{ J^* - J(\theta_0) } {(T+1)^2} + \fr{ D_h }{ \mu (T+1)^{\nfr{3}{2}} }+ \fr{ D_g }{ T+1 } +   \fr{ \sigma_g \sqrt{ \mu} + L_g }{\mu( T+1) } +  \fr{  \sigma_g \sqrt{\mu} + L_g + \sigma_h }{\mu (T+1)^{\nfr{1}{2}}}  }  + \fr{ \varepsilon^{\prime} }{\sqrt{2\mu} } \notag,
    \end{eqnarray}
    where~$\sigma_g, \sigma_h, L_g$ are defined in Proposition~\ref{prop:stoch-grad-and-J} and Lemma~\ref{lem:2nd-order-smoothness} respectively, $D_g$ and $D_h$ are defined in Lemma~\ref{le:trunc_grad_hess}, $\varepsilon^{\prime}$ and $\mu$ are defined in Lemma~\ref{lem:relaxed-w-grad-dom}.
\end{theorem}

Similarly to the proof of Theorem~\ref{thm:N-PG-IGT} for \algname{N-PG-IGT}, our starting point is the ascent-like result provided by Lemma~\ref{lem:descent-with-grad-dom} which holds independently of the update direction~$(d_t)$. 
In view of the bound~(\ref{eq:NSGD_descent2}), we need to control the error term
$$
V_t = \mathbb E[\|d_t - \nabla J_H(\theta_t)\|^2]\,.
$$ 
Recall that we have already provided a bound for this quantity in Lemma~\ref{le:g_upd_rec_Hess_Aid_trunc}. 
We are now ready to provide the proof of Theorem~\ref{thm:NSTORM_Hess_trunc}. 

\noindent\textbf{End of the proof of Theorem~\ref{thm:NSTORM_Hess_trunc} (and hence Theorem~\ref{thm:N-HARPG}).}
    Using Lemma~\ref{lem:descent-with-grad-dom} (note that this lemma does not use the specific update rule of the sequence~$(d_t)$ and can be used here) and applying Jensen's inequality for $x \mapsto x^2$, we have
    \begin{eqnarray}
    \label{eq:delta_t-ineq}
    \delta_{t+1} &\leq& \rb{1 - \fr{\sqrt{2 \mu} \g_t}{3} } \delta_t + \fr{8 \g_t }{3} V_t^{\nfr{1}{2}} + \fr{L_g \g_t^2}{2} + \fr{\varepsilon^{\prime}\g_t}{3}  + \fr{4 \g_t D_g}{3 } \g^{{H}} .
    \end{eqnarray}

    Then, using Lemma~\ref{le:g_upd_rec_Hess_Aid_trunc} and observing that~$\|\theta_{t+1} - \theta_t\| = \gamma_t$ as a consequence of the normalized update rule of Algorithm~\ref{alg:(N)-HARPG} (\algname{N-HARPG}), we have 
    \begin{eqnarray}
    V_{t+1} - V_t \leq - \eta_{t+1} V_t + 2\sigma_g^2 \eta_{t+1}^2  + 12 (2 L^2 + \sigma_H^2 + D_h^2 \g^{2 H})   \g_t^2 + 24 D_g^2 \g^{2 H}. 
    \end{eqnarray}
    Applying Lemma~\ref{le:aux_rec0} with $r_t = V_t$, $\al_t = \eta_{t+1}$, $a=1$, $\beta_t = c \cdot\eta_t^2 + \bar{c}$,  $c = 2 \sigma_g^2 + 12 (2 L_g^2 + \sigma_h^2 + D_h^2 \g^{2 H}) \g_0^2$, $\bar{c} = 24 D_g^2 \g^{2 H}$, $\tau = 2, t_ 0 = 0$, we get
    \begin{equation}
    \label{eq:V_t-estimate}
    V_t \leq \fr{V_0}{(t+1)^2} + \fr{ \sum_{\tau = 0}^{t-1} (c \eta_{\tau}^2 + \bar{c}) (\tau+2)^2}{(t+1)^2}
    \leq \sigma_g^2 \eta_t^2 + 2 c \eta_t + 2 \bar{c} \eta_t^{-1}\,.
    \end{equation}
    
    Combining inequalities~\eqref{eq:delta_t-ineq} and~\eqref{eq:V_t-estimate}, we obtain
    \begin{eqnarray}
    \delta_{t+1} &\leq& \rb{1 -  \eta_t  } \delta_t + \fr{8 \g_0 \eta_t }{3} V_t^{\nfr{1}{2}} + \fr{L_g \g_0^2 \eta_t^2}{2} + \fr{\g_0  \varepsilon^{\prime} \eta_t }{3} + \fr{4 \g_t D_g}{3 } \g^{{H}}  \notag \\
    &\leq& \rb{1 -  \eta_t  } \delta_t + \fr{8 \g_0 \eta_t }{3} \rb{ \sigma_g \eta_t  +  \sqrt{2 c} \eta_t^{\nfr{1}{2}} + \sqrt{2 \bar{c}} \eta_t^{-\nfr{1}{2}} }   + \fr{L_g \g_0^2 \eta_t^2}{2} + \fr{\g_0 \varepsilon^{\prime} \eta_t}{3}  + \fr{4 \g_t D_g}{3 } \g^{{H}} \notag \\
    &\leq& \rb{1 -  \eta_t  } \delta_t + \rb{ \fr{8 \g_0 \sigma_g }{3}  +  \fr{L_g \g_0^2 }{2} } \eta_t^2 + \fr{8 \sqrt{2c} \g_0  }{3} \eta_t^{\nfr{3}{2}} + \fr{\g_0 \varepsilon^{\prime}\eta_t }{3} + \fr{16 \sqrt{2 \bar{c}} \g_0  }{3} \eta_t^{\nfr{1}{2}}\notag .
    \end{eqnarray}    
    
    Again, applying Lemma~\ref{le:aux_rec0} with $r_t = \delta_t$, $\al_{t} = \eta_t $, $a = 1$, $\tau = 2$, $t_0 = 0$, $\beta_t = c_1  \eta_{t}^2 +  c_2 \eta_{t}^{\nfr{3}{2}} + c_3 \eta_t + c_4 \eta_t^{\nfr{1}{2}}$,  $c_1 = \rb{ \fr{8 \g_0 \sigma_g }{3}  +  \fr{L_g \g_0^2 }{2} }$, $c_2 = \fr{8 \sqrt{2c} \g_0  }{3} $, $c_3 = \fr{\g_0  \varepsilon^{\prime} }{3}$, $c_4 = \fr{16 \sqrt{2 \bar{c}} \g_0  }{3}$ we get for every integer~$T \geq 1$
    \begin{eqnarray}
        \delta_T &\leq& \fr{ \delta_{0} }{(T+1)^2} + \fr{  \sum_{t=0}^{T-1}\beta_t (t + 2 )^2 }{(T+1)^2} \notag \\
         &=& \fr{ \delta_{0} }{(T+1)^2} + \fr{  c_1 \sum_{t=0}^{T-1}\eta_t^2 (t + 2 )^2 }{(T+1)^2} + \fr{  c_2 \sum_{t=0}^{T-1}\eta_t^{\nfr{3}{2}} (t + 2 )^2 }{(T+1)^2} + \fr{  c_3 \sum_{t=0}^{T-1}\eta_t (t + 2 )^2 }{(T+1)^2}  + \fr{  c_4 \sum_{t=0}^{T-1}\eta_t^{\nfr{1}{2}} (t + 2 )^2 }{(T+1)^2}  \notag \\
         &\leq& \fr{ \delta_{0} }{(T+1)^2} + \fr{  4 c_1}{T+1} + \fr{  4 c_2 }{(T+1)^{\nfr{1}{2}}} + c_3 + 4 c_4 (T+1)^{\nfr{1}{2}}   \notag \\
         &\leq& \fr{ \delta_{0} }{(T+1)^2} + 4 \rb{ \fr{8 \sigma_g }{\sqrt{2 \mu}}  +  \fr{ 9 L_g  }{4 \mu} } \fr{ 1}{T+1} + 64 \rb{\fr{\sigma_g \sqrt{\mu} + 6 L_g + 3\sigma_h  + 3 D_h \g^{ H} }{\mu} } \fr{  1 }{(T+1)^{\nfr{1}{2}}} + \fr{\g_0  \varepsilon^{\prime} }{3} + 4 c_4 (T+1)^{\nfr{1}{2}} \notag \\
         &\leq& \fr{ \delta_{0} }{(T+1)^2} + 32 \cdot  \fr{ \sigma_g \sqrt{ \mu} + L_g }{\mu( T+1) } + 192 \cdot  \fr{  \sigma_g \sqrt{\mu} + L_g + \sigma_h + D_h \g^{ H}}{\mu (T+1)^{\nfr{1}{2}}} + \fr{ \varepsilon^{\prime} }{\sqrt{2\mu} } + 150 \cdot  D_g \g^{H}  (T+1)^{\nfr{1}{2}} . \notag 
    \end{eqnarray}
    The choice $H = \fr{3}{2} \rb{1-\g}^{-1}{\log(T + 1)}$ concludes the proof. 

\newpage
\section{Convergence Analysis of Normalized-Momentum Policy Gradient Method (\algname{N-MPG})}
\label{sec:NSGDM_app}
In this section, we present a global convergence result for a simple Normalized Momentum Policy Gradient algorithm which we name~\algname{N-MPG} (see Algorithm~\ref{alg:N-MPG}). Unlike~\algname{N-PG-IGT} which enjoys a~$\tilde{\mathcal{O}}(\varepsilon^{-2.5})$ global sample complexity (to reach a globally optimal policy up to the bias due to policy parameterization), \algname{N-MPG} does not use a look-ahead step and our analysis only guarantees a~$\tilde{\mathcal{O}}(\varepsilon^{-3})$ global sample complexity for this algorithm.  


\begin{algorithm}[H]
        \caption{\algname{N-MPG} (Normalized-Momentum Policy Gradient) }\label{alg:N-MPG}
        \begin{algorithmic}[1]
            \STATE \textbf{Input}: $\theta_0$, $T$, $\cb{\eta_t}_{t\geq 0}$, $\cb{\gamma_t}_{t\geq 0}\,$
            \STATE $\tau_0 \sim p(\cdot|\pi_{\theta_0})$
            \STATE $d_0 = g(\tau_0, \theta_0)$
            \STATE $\theta_1 = \theta_0 + \gamma_0 \frac{d_0}{\|d_0\|}$
            \FOR{$t=1, \ldots, T -  1$}
                \STATE $\tau_t \sim p(\cdot|\pi_{\theta_t})$ 
                \STATE \label{n-mpg:step-d-t} $d_t = (1 - \eta_t) d_{t-1} + \eta_t g(\tau_t, \theta_t)$
                \STATE \label{n-mpg:step-update} $\theta_{t+1} = \theta_t + \gamma_t \frac{d_t}{\|d_t\|}$ 
		    \ENDFOR
		\RETURN $\theta_T$
        \end{algorithmic}
\end{algorithm}

\begin{theorem}\label{thm:N-MPG}
   Let Assumptions~\ref{hyp:parameterization-reg}, \ref{hyp:fisher-non-degenerate} and \ref{hyp:tranf-compatib-fun-approx} hold. 
    Set~$\gamma_t = \fr{6  }{\sqrt{2\mu}( t+2 ) } $, $\eta_t = \rb{ \fr{2}{t+2} }^{\nfr{2}{3}}$, $H = \fr{5}{3} \rb{1-\g}^{-1} {\log(T + 1)}$ . Then for every integer~$T \geq 1$, the output~$\theta_T$ of \algname{N-MPG} (Algorithm~\ref{alg:N-MPG}) satisfies 
    \begin{eqnarray}
        J^* -  \Exp{J(\theta_T)} \leq \cO\rb{ \fr{ J^* -  J(\theta_0) }{(T+1)^2} + \fr{ D_g }{ \sqrt{\mu}(T+1) } +  \fr{ \sigma_g \sqrt{\mu} + L_g  }{\mu (T+1)^{\nfr{1}{3}}} } + \fr{ \varepsilon^{\prime}}{\sqrt{2 \mu}} .
    \end{eqnarray}
    where~$\sigma_g, L_g$ are defined in Proposition~\ref{prop:stoch-grad-and-J}, $D_g$ is defined in Lemma~\ref{le:trunc_grad_hess}, $\varepsilon^{\prime}$ and $\mu$ are defined in Lemma~\ref{lem:relaxed-w-grad-dom}. 
\end{theorem}

The following corollary provides the sample complexity of the algorithm, i.e., the number of observed state-action pairs to achieve a global optimality precision~$\varepsilon > 0$ up to the error floor~$\varepsilon^{\prime}/\sqrt{2 \mu}$ related to the function approximation power of the policy parameterization. 
\begin{corollary}
In the setting of Theorem~\ref{thm:N-MPG}, for every~$\varepsilon >0,$ the output~$\theta_T$ of \algname{N-MPG} satisfies~$J^* - \Exp{J(\theta_T)} \leq \varepsilon + \varepsilon^{\prime}/\sqrt{2 \mu}$ for~$T = \cO(\varepsilon^{-3}),$ i.e., the total sample complexity is of the order~$\tilde{\cO}(\varepsilon^{-3})\,.$
\end{corollary}

\begin{proofof}{Theorem~\ref{thm:N-MPG}}
    Combining Lemma~\ref{lem:descent-with-grad-dom} with Lemma~\ref{le:NSGDM_first_err_bound} we obtain
\begin{eqnarray}
    \delta_{t+1}  &\leq&  \rb{ 1 - \fr{\sqrt{2 \mu} \g_t}{3} } \delta_t + \fr{8 \g_t }{3} \Exp{ \norm{d_t - \nabla J(\theta_t)} } + \fr{L_g \g_t^2}{2} + \fr{\g_t \varepsilon^{\prime}}{3} + \fr{4 \g_t D_g}{3 } \g^{{H}} \notag \\
    &\leq&  \rb{ 1 - \fr{\sqrt{2 \mu} \g_t}{3} } \delta_t + \fr{16 }{3}  \sqrt{ C\rb{\fr{4}{3}, \fr{2}{3}} } \sigma_g  \g_t   \eta_t^{\nfr{1}{2}} + \fr{8 }{3} C\rb{1,  \fr{2}{3}} L_g \g_t^2 \eta_t^{-1} + \fr{L_g \g_t^2}{2} + \fr{\g_t \varepsilon^{\prime}}{3}  + \fr{4 \g_t D_g}{3 } \g^{{H}} \notag \\
    && \qquad + \fr{16}{3} D_g \g^H C\rb{0, \fr{2}{3}} \g_t \eta_t^{-1} \notag \\
    &\leq&  \rb{ 1 - \fr{\sqrt{2 \mu} \g_t}{3} } \delta_t + \fr{16 }{3}  \sqrt{ C\rb{\fr{4}{3}, \fr{2}{3}} } \sigma_g  \g_t   \eta_t^{\nfr{1}{2}} + 8 C\rb{1,  \fr{2}{3}} L_g \g_t^2 \eta_t^{-1} + \fr{\g_t \varepsilon^{\prime}}{3}  + \fr{4 \g_t D_g}{3 } \g^{{H}} \notag \\
    && \qquad + \fr{16}{3} D_g \g^H C\rb{0, \fr{2}{3}} \g_t \eta_t^{-1} \notag .
\end{eqnarray}
Using Lemma~\ref{le:aux_rec0} with $\al_t = \g_t$, $a = \fr{3}{\sqrt{2\mu}}$, $t_0 = 0$, $\tau = 2$, $\beta_t = \fr{16 }{3}  \sqrt{ C\rb{\fr{4}{3}, \fr{2}{3}} } \sigma_g  \g_t   \eta_t^{\nfr{1}{2}} + 8  C\rb{1, \fr{2}{3}} L_g \g_t^2 \eta_t^{-1} + \fr{\g_t \varepsilon^{\prime}}{3}  + \fr{4 \g_t D_g}{3 } \g^{{H}} + \fr{16}{3} D_g \g^H C\rb{0, \fr{2}{3}} \g_t \eta_t^{-1} $ leads to
\begin{eqnarray}
    \delta_T 
    &\leq& \fr{ \delta_0 }{(T+1)^2} +  \fr{  \sum_{t=0}^{T-1} \beta_t (t + 2 )^2 }{(T+1)^2} \notag \\
    &\leq& \fr{ \delta_0 }{(T+1)^2} +  \fr{16 }{3}  \sqrt{ C\rb{\fr{4}{3}, \fr{2}{3}} } \fr{ \sigma_g \g_0}{(T+1)^{\nfr{1}{3}}} +   16  C\rb{1, \fr{2}{3}} \fr{ L_g \g_0^2}{(T+1)^{\nfr{1}{3}}}   + \fr{\g_0 \varepsilon^{\prime}}{3}  + \fr{4 \g_0 D_g}{3 } \g^{{H}} \notag \\
    && \qquad + \fr{16   }{3 }  C\rb{0, \fr{2}{3}}  \g_0  \g^H  D_g (T+2)^{\nfr{2}{3}} \notag .
\end{eqnarray}
Setting~$H = \fr{5}{3} \rb{1-\g}^{-1}{\log(T + 1)}$ concludes the proof.
\end{proofof}

\begin{lemma}\label{le:NSGDM_first_err_bound}
Let Assumption~\ref{hyp:parameterization-reg} be satisfied. Let~$(d_t)$ and~$(\theta_t)$ be the sequences generated by the~\algname{N-MPG} algorithm (see Algorithm~\ref{alg:N-MPG})
with~$\eta_t = \rb{ \fr{2}{t+2} }^{\nfr{2}{3}}$ and~$\gamma_t = \fr{6  }{\sqrt{2\mu}( t+2 ) }$. Then for any integer~$t \geq 0$ it holds 
\begin{eqnarray}
     \Exp{ \norm{d_t - \nabla J_H(\theta_t)} } &\leq&   2 \cdot \sqrt{ C\rb{\fr{4}{3}, \fr{2}{3}} }  \eta_T^{\nfr{1}{2}} \sigma_g   +  C\rb{1, \fr{2}{3}}  \g_T \eta_T^{-1} L_g + 2 D_g \g^H C\rb{0, \fr{2}{3}} \eta_T^{-1} ,
\end{eqnarray}
where $C(p,q)$ for $q \in [0, 1)$, $p > 0 $ is defined in Lemma~\ref{le:sum_prod_bound1}.
\end{lemma}

\begin{proof}
Define $\hat{e}_t \eqdef d_t - \nabla J_H(\theta_t)$ and~${e}_t \eqdef  g(\tau_t, \theta_t) -\nabla J_H(\theta_t)$. Then the update rule of the sequence~$(d_t)$ implies
\begin{eqnarray}
    \hat{e}_{t} &=& d_{t} - \nabla J_H(\theta_t)  \notag \\
    &=& (1 - \eta_t) d_{t-1} + \eta_t g(\tau_t, \theta_t) - \nabla J_H(\theta_t) \notag \\
    &=& (1 - \eta_t) \rb{\hat{e}_{t-1} + \nabla J_H(\theta_{t-1}) } + \eta_t g(\tau_t, \theta_t) - \nabla J_H(\theta_t)  \notag \\
    &=& (1 - \eta_t) \hat{e}_{t-1}  + \eta_t e_{t} + (1 - \eta_t) \rb{  \nabla J_H(\theta_{t-1}) - \nabla J_H(\theta_{t}) }  \notag .
\end{eqnarray}
Unrolling the recursion and defining $S_t \eqdef \nabla J_H(\theta_{t-1}) - \nabla J_H(\theta_{t})$,
we have 
\begin{eqnarray}
    \hat{e}_{T} &= & \prod_{t = 0}^{T-1} (1 - \eta_{t+1}) \hat{e}_0 + \sum_{t = 0}^{T-1} \eta_{T-t} e_{T-t} \prod_{\tau = T - t}^{T-1} (1 - \eta_{\tau+1}) + \sum_{t = 0}^{T-1} (1-\eta_{T-t}) S_{T-t} \prod_{\tau = T - t}^{T-1} (1 - \eta_{\tau+1})\,.   \notag \\
\end{eqnarray}
Define for every integer~$t \geq 1$ the $\sigma$-field $\mathcal{F}_{t} \eqdef \sigma(\cb{\theta_0, \tau_0, \dots, \tau_{t-1} })$ where~$\tau_s \sim p(\cdot|\pi_{{\theta}_s})$ for every~$0 \leq s \leq t-1$. Notice that for any integers~$t_2 > t_1 \geq 1$  we have 
$$
\Exp{ \langle e_{t_1}, e_{t_2} \rangle } = \Exp{ \Exp{ \langle e_{t_1}, e_{t_2} \rangle | \mathcal{F}_{t_2} } } = \Exp{ \langle e_{t_1}, \Exp{  e_{t_2}  | \mathcal{F}_{t_2} }  \rangle  } = 0.
$$
Then using triangle inequality, taking expectation and applying Jensen's inequality, we get
\begin{eqnarray}
    \Exp{\norm{\hat{e}_{T}}} &\leq & \Exp{\norm{\hat{e}_0}} \prod_{t = 0}^{T-1} (1 - \eta_{t+1})  +  \Exp{\norm{ \sum_{t = 0}^{T-1} \eta_{T-t} e_{T-t} \prod_{\tau = T - t}^{T-1} (1 - \eta_{\tau+1}) }} \notag \\
    &&\qquad +  \Exp{\norm{  \sum_{t = 0}^{T-1} (1-\eta_{T-t}) S_{T-t} \prod_{\tau = T - t}^{T-1} (1 - \eta_{\tau+1})  } } \notag \\
    &\leq & \sigma_g\prod_{t = 0}^{T-1} (1 - \eta_{t+1})  +  \rb{ \Exp{ \sqnorm{ \sum_{t = 0}^{T-1} \eta_{T-t} e_{T-t} \prod_{\tau = T - t}^{T-1} (1 - \eta_{\tau+1}) }}}^{\nfr{1}{2}} \notag \\
    &&\qquad +   \sum_{t = 0}^{T-1} (1-\eta_{T-t}) \Exp{\norm{S_{T-t}}} \prod_{\tau = T  - t}^{T-1} (1 - \eta_{\tau+1})   \notag \\    
    &\overset{(i)}{\leq} & \sigma_g \prod_{t = 0}^{T-1} (1 - \eta_{t+1})   +  \rb{ \sum_{t = 0}^{T-1} \eta_{T-t}^2 \Exp{\sqnorm{ e_{T-t}}} \prod_{\tau = T  - t}^{T-1} (1 - \eta_{\tau+1})^2 }^{\nfr{1}{2}} \notag \\
    &&\qquad +  \rb{ \sum_{t = 0}^{T-1} (1-\eta_{T-t}) \gamma_{T-t-1} \prod_{\tau = T  - t}^{T-1} (1 - \eta_{\tau+1})  } L_g + 2 D_g \g^H \rb{ \sum_{t = 0}^{T-1}   \prod_{\tau = T  - t}^{T-1} (1 - \eta_{\tau+1})  }  \notag \\    
    &\leq & \rb{ \prod_{t = 0}^{T-1} (1 - \eta_{t+1}) } \sigma_g +  \rb{ \sum_{t = 0}^{T-1} \eta_{T-t}^2  \prod_{\tau = T - t}^{T-1} (1 - \eta_{\tau+1})^2 }^{\nfr{1}{2}} \sigma_g  \notag \\
    && \qquad + \rb{ \sum_{t = 0}^{T-1} (1-\eta_{T-t}) \gamma_{T-t- 1}  \prod_{\tau = T  - t}^{T-1} (1 - \eta_{\tau+1}) } L_g + 2 D_g \g^H \rb{ \sum_{t = 0}^{T-1}   \prod_{\tau = T  - t}^{T-1} (1 - \eta_{\tau+1})  } \notag ,   
\end{eqnarray}
where in $(i)$ we utilize smoothness (Proposition~\ref{prop:stoch-grad-and-J}), the fact that $\norm{\theta_{t+1} - \theta_t} = \g_t$ and use independence $\Exp{ \langle e_{t_1}, e_{t_2} \rangle } = 0$ for any $t_1 \neq t_2$. In the last inequality, we use $\Exp{\sqnorm{ e_{T-t}}} \leq \sigma_g^2$.

Further, using Lemma~\ref{le:prod_bound} and \ref{le:sum_prod_bound1} we have
\begin{eqnarray}
    \Exp{\norm{\hat{e}_{T}}}  &\leq & \eta_T \sigma_g +  \sqrt{ C\rb{\fr{4}{3}, \fr{2}{3}} } \eta_T^{\nfr{1}{2}} \sigma_g   +  C\rb{1, \fr{2}{3}} \g_T \eta_T^{-1} L_g + 2 D_g \g^H C\rb{0, \fr{2}{3}} \eta_T^{-1} \notag\\
    &\leq &  2 \cdot \sqrt{ C\rb{\fr{4}{3}, \fr{2}{3}} }  \eta_T^{\nfr{1}{2}} \sigma_g   +  C\rb{1, \fr{2}{3}}  \g_T \eta_T^{-1} L_g  + 2 D_g \g^H C\rb{0, \fr{2}{3}} \eta_T^{-1} \notag ,
\end{eqnarray}
where $C(p,q)$ for $q \in [0, 1)$, $p > 0 $ is defined in Lemma~\ref{le:sum_prod_bound1}.

\end{proof}

\newpage
\section{Technical Lemma}\label{sec:tech_lemmas}

\subsection{Lemma for solving recursions}

The following auxiliary lemma is similar to \citep[Lemma~7]{Unified_SGD_Stich} and is useful to solve our recursions.

\begin{lemma}\label{le:aux_rec0}
Let $a$ be a positive real, $\tau$ a positive integer and let $\cb{r_t}_{t\geq 0}$ be a non-negative sequence satisfying for every integer~$t \geq 0$
$$
r_{t+1} - r_t \leq - a \alpha_t r_t + \beta_t,
$$
where $\cb{\alpha_t}_{t\geq 0}, \cb{\beta_t}_{t\geq 0}$ are non-negative sequences and $ a \alpha_t \leq 1$ for all $t$. Then for $\al_t = \fr{2 }{ a (t+\tau) }$ we have for every integers~$t_0, T \geq 1$
$$
r_T \leq \fr{ (t_0 + \tau-1)^2 r_{t_0} }{(T+\tau-1)^2} + \fr{  \sum_{t=0}^{T-1}\beta_t (t + \tau )^2 }{(T+\tau-1)^2} .
$$
\end{lemma}
\begin{proof}
    Notice that $1 - a \al_t = \fr{t+\tau -2}{t+\tau}$. Then for all $t\geq 0$
    $$
    r_{t+1} \leq \fr{t+\tau -2}{t+\tau} r_t +   \beta_t.
    $$ 
    Multiplying both sides by $(t+ \tau)^2$, we get
    \begin{eqnarray}
    (t+\tau)^2 r_{t+1} &\leq& (t+\tau -2)(t+\tau) r_t +    \beta_t (t+\tau)^2 \notag \\
    &\leq& (t+\tau -1)^2 r_t +   \beta_t (t+\tau)^2 \notag .
    \end{eqnarray}
    By summing this inequality from $t= t_0$ to~$T-1$, we obtain
    $$
    (T+\tau-1)^2 r_T \leq (t_0 + \tau-1)^2 r_{t_0} +  \sum_{t=t_0}^{T-1}\beta_t (t + \tau )^2,
    $$
    $$
     r_T \leq \fr{ (t_0 + \tau-1)^2 r_{t_0} }{(T+\tau-1)^2} + \fr{  \sum_{t=0}^{T-1}\beta_t (t + \tau )^2 }{(T+\tau-1)^2} .
    $$
\end{proof}

\begin{lemma}\label{le:aux_rec3}
Let~$\cb{r_t}_{t\geq 0}$ be a non-negative sequence satisfying for every integer~$t \geq 0$
$$
r_{t+1} - r_t \leq - \alpha_t r_t + c \beta_t,
$$
where $\al_t = \rb{ \fr{2 }{ t+2 } }^{\nfr{2}{3}}$ and~$\cb{\beta_t}_{t\geq 0}$ is a sequence of non-negative reals. Then for every integer~$T$
$$
r_T \leq \fr{ r_{0} }{(T+1)^{\nfr{1}{3}}} + \fr{ c \sum_{t=0}^{T-1}\beta_t (t + 2 )^{\nfr{1}{3}} }{(T+1)^{\nfr{1}{3}}} .
$$
\end{lemma}
\begin{proof}
    Notice that $1 - \al_t = \fr{(t+2)^{\nfr{2}{3}} - 2^{\nfr{2}{3}}}{(t+2)^{\nfr{2}{3}}}$. Then multiplying both sides of the recursion by $(t+ \tau)^{\nfr{1}{3}}$, we get for all $t\geq 0$
    \begin{eqnarray}
    (t+2)^{\nfr{1}{3}} r_{t+1} &\leq& \rb{(t+2)^{\nfr{2}{3}} - 2^{\nfr{2}{3}}} (t+2)^{-\nfr{1}{3}} r_t +   c \beta_t (t+2)^{\nfr{1}{3}} \notag \\
    &\leq& t^{\nfr{2}{3}} (t+2)^{-\nfr{1}{3}} r_t +  c \beta_t (t+2)^{\nfr{1}{3}} \notag \\
    &\leq&  (t+1)^{\nfr{1}{3}} r_t +  c \beta_t (t+2)^{\nfr{1}{3}} \notag ,
    \end{eqnarray}
    where the second inequality holds since $ \rb{(t+2)^{\nfr{2}{3}} - 2^{\nfr{2}{3}}} (t+2)^{-\nfr{1}{3}} = t^{\nfr{2}{3}} (t+2)^{-\nfr{1}{3}} \rb{ \rb{1 + \fr{2}{t}}^{\nfr{2}{3}} - \fr{2^{\nfr{2}{3}}}{t^{\nfr{2}{3}}}  } \leq t^{\nfr{2}{3}} (t+2)^{-\nfr{1}{3}} \rb{ 1 + \fr{4}{3 t} - \fr{2^{\nfr{2}{3}}}{t^{\nfr{2}{3}}}  } \leq t^{\nfr{2}{3}} (t+2)^{-\nfr{1}{3}}$. 
    By telescoping and dividing by~$(T+1)^{\nfr{1}{3}}$, we obtain for every integer~$T$
    $$
     r_T \leq \fr{ r_0 }{(T+1)^{\nfr{1}{3}}} + \fr{ c \sum_{t=0}^{T-1} \beta_t (t + 2 )^{\nfr{1}{3}} }{(T+1)^{\nfr{1}{3}}} .
    $$
\end{proof}


\subsection{Lemma for decreasing step-sizes estimates}

\begin{lemma}\label{le:prod_bound}
   Let $q \in [0,1]$ and let $\eta_t = \rb{ \fr{2}{t+2} }^{q}$ for every integer~$t$. Then for every integer~$t$ and any integer~$T \geq 1$ we have
    \begin{eqnarray}\label{eq:tech1}
    \eta_t (1 - \eta_{t+1}) \leq \eta_{t+1},
    \end{eqnarray}
    $$
    \prod_{t = 0}^{T-1} (1 - \eta_{t+1}) \leq \eta_T\,.
    $$
\end{lemma}
\begin{proof}
For every integer~$t$ we have
    \begin{eqnarray}
        1 - \eta_{t+1} &=& 1 - \rb{\fr{2}{t+3}}^{q} 
        \leq 1 - \fr{1}{t+3}  =
        \fr{t+2}{t+3}  \leq 
        \fr{\eta_{t+1}}{\eta_t} \notag .
    \end{eqnarray}
    Using the above result, we can write
    \begin{eqnarray*}
        \prod_{t = 0}^{T-1} (1 - \eta_{t+1}) &\leq&  \prod_{t = 0}^{T-1} \fr{\eta_{t+1}}{\eta_{t}} =  \fr{\eta_{T}}{\eta_{0}} = \eta_T.
    \end{eqnarray*}
\end{proof}

The following lemma is a generalization of \citep[Proposition~B.1]{Gadat_SHB_18} holding for every integer~$T \geq 1$ (instead of only for~$T$ sufficiently large). 
\begin{lemma}\label{le:sum_prod_bound1}
Let $q \in [0, 1)$, $p \geq 0$, $\gamma_0 > 0$ and let $\eta_t =  \rb{ \fr{2}{t+2} }^q$, $\g_t = \g_0 \rb{ \fr{2}{t+2} }^p$ for every integer~$t$. Then for any integers~$t$ and~$T \geq 1$, it holds 
    $$
     \sum_{t = 0}^{T-1} \g_{t} \prod_{\tau = t+1 }^{T-1} (1 - \eta_{\tau}) \leq C \g_T \eta_T^{-1},
    $$
     where $C = C(p,q) \eqdef 2^{p-q} (1-q)^{-1}  t_0 \exp\rb{  2^q (1-q)  t_0^{1-q} }  + 2^{ 2 p + 1 - q} (1-q)^{-2}$ and~$ t_0 \eqdef \max\cb{\rb{ \fr{p}{(1-q)2^q} }^{\fr{1}{1-q}}, 2 \rb{\fr{p-q}{(1-q)^2}}}^{\fr{1}{1-q}} $.
\end{lemma}
\begin{proof} We start the proof by writing 
    \begin{eqnarray}
        \sum_{t = 0}^{T-1} \g_{t}  \prod_{\tau = t+1 }^{T-1} (1 - \eta_{\tau}) &\leq& \sum_{t = 0}^{T-1} \g_{t}  \exp\rb{ -\sum_{\tau = t+1 }^{T-1}   \eta_{\tau}}  \notag \\
        &=& \exp\rb{ -\sum_{\tau = 0 }^{T-1}   \eta_{\tau}} \sum_{t = 0}^{T-1} \g_{t} \exp\rb{ \sum_{\tau = 0 }^{t}   \eta_{\tau}}  \notag \\
        &\leq& \exp\rb{ -\sum_{\tau = 0 }^{T-1}   \eta_{\tau}} \sum_{t = 0}^{T-1} \g_{t} \exp\rb{  \int_{0 }^{t}   \eta_{\tau-1} d \tau}  \notag \\
        &\leq& \g_0 2^p \exp\rb{ -\sum_{\tau = 0 }^{T-1}   \eta_{\tau}} \sum_{t = 0}^{T-1} (t+2)^{-p} \exp\rb{  2^q (1-q)  (t+2)^{1-q} }  \notag .
    \end{eqnarray}
    For $t \geq \max\cb{\rb{ \fr{p}{(1-q)  2^q} }^{\fr{1}{1-q}}, 2 \rb{\fr{p-q}{ (1-q)^2}}^{\fr{1}{1-q}}} \eqdef t_0 $ the function $\phi(t) = t^{-p} \exp\rb{  2^q (1-q)  t^{1-q} } $ is increasing, therefore, the corresponding sum is bounded by
    \begin{eqnarray}
    2^{-p} t_0 \exp\rb{  2^q (1-q)  t_0^{1-q} } & + & \int_{t_0}^{T+1} \tau^{-p} \exp\rb{  2^q (1-q)  \tau^{1-q} } d \tau \notag .
    \end{eqnarray}
    Notice that integrating by parts, we have
    \begin{eqnarray}
        I &\eqdef& \int_{t_0}^{T+1} \tau^{-p} \exp\rb{ 2^q (1-q)  \tau^{1-q} } d \tau \notag \\
        &=& \int_{t_0}^{T+1}  2^{-q} (1-q)^{-2} \tau^{q-p} d \rb{ \exp\rb{ 2^q (1-q)  \tau^{1-q} } } \notag \\
        &=&  2^{-q} (1-q)^{-2} \exp\rb{ 2^q (1-q)  (T+1)^{1-q} } (T+1)^{q-p} \notag \\
        && \qquad +  2^{-q} (1-q)^{-2} (p-q ) \int_{t_0}^{T+1} \tau^{-p+q-1} \exp\rb{ 2^q (1-q)  \tau^{1-q} } d \tau \notag \\
         &=&  2^{-q} (1-q)^{-2} \exp\rb{ 2^q (1-q)  (T+1)^{1-q} } (T+1)^{q-p} \notag \\
        && \qquad + \ 2^{-q} (1-q)^{-2} (p-q ) t_0^{-(1-q)} \int_{t_0}^{T+1} \tau^{-p} \exp\rb{ 2^q (1-q)  \tau^{1-q} } d \tau \notag \\
        &\leq&  2^{-q} (1-q)^{-2} \exp\rb{  2^q (1-q)  (T+1)^{1-q} } (T+1)^{q-p} + \fr{1}{2} I \notag ,
    \end{eqnarray}
    where the last inequality is satisfied by the choice of $t_0$ when $p \geq q$, and holds trivially for $p \leq q$. Thus,
    $$
    I \leq 2^{1-q} (1-q)^{-2} \exp\rb{  2^q (1-q)  (T+1)^{1-q} } (T+1)^{q-p}.
    $$
    Finally, combining everything together, we have
    \begin{eqnarray}
        \sum_{t = 0}^{T-1} \g_{t}  \prod_{\tau = t+1 }^{T-1} (1 - \eta_{\tau})  &\leq& \g_0 2^p \exp\rb{ -  2^q (1-q)  (T+1)^{1-q} } 2^{-p} t_0 \exp\rb{  2^q (1-q)  t_0^{1-q} } \notag \\
        && \qquad + \g_0  2^{p+1-q} (1-q)^{-2} (T+1)^{q-p}   \notag \\
        &\leq& \g_0 \rb{  2^{-q} (1-q)^{-1}  t_0 \exp\rb{  2^q (1-q)  t_0^{1-q} }  +  2^{p+1-q} (1-q)^{-2} } (T+1)^{q-p} \notag \\
        &\leq& C \g_T \eta_{T}^{-1} \notag ,
    \end{eqnarray}
    where $C = C(p,q) \eqdef 2^{p-q} (1-q)^{-1}  t_0 \exp\rb{  2^q (1-q)  t_0^{1-q} }  + 2^{ 2 p + 1 - q} (1-q)^{-2}$.
\end{proof}

\end{document}